\documentclass[twoside]{article}

\usepackage{blindtext}

\usepackage[preprint]{jmlr2e}

\usepackage{lastpage}

\usepackage[utf8]{inputenc} %
\usepackage[T1]{fontenc}    %
\usepackage{hyperref}       %
\usepackage{url}            %
\usepackage{booktabs}       %
\usepackage{amsfonts}       %
\usepackage{nicefrac}       %
\usepackage{microtype}      %
\usepackage{listings}
\usepackage[svgnames]{xcolor}
\usepackage{enumitem}

\lstdefinelanguage{RASPL}{
    morekeywords={def,return,if,else,for,while,in,import,from},
    sensitive=true,
    morecomment=[l]{\#},
    morestring=[b]',
    morestring=[b]",
}

\lstdefinestyle{rasplstyle}{
    language=RASPL,
    backgroundcolor=\color{black!5},
    basicstyle=\ttfamily\footnotesize,
    keywordstyle=\color{blue!70!black},
    commentstyle=\color{gray!60},
    stringstyle=\color{orange!80!black},
    showstringspaces=false,
    frame=single,
    breaklines=true,
    tabsize=4,
}

\usepackage{amsmath}
\usepackage{breqn}
\usepackage{amssymb}
\usepackage{mathtools}
\usepackage{xspace}
\usepackage{physics}
\usepackage{graphicx}
\usepackage{wrapfig}
\usepackage{subcaption}
\usepackage{float}
\usepackage[capitalize,noabbrev]{cleveref}
\usepackage{enumitem}

\usepackage{adjustbox} 

\usepackage[most]{tcolorbox}
\tcbuselibrary{theorems,breakable,skins}

\definecolor{theoFrame}{HTML}{2B6CB0}   %
\definecolor{theoBack}{HTML}{EBF8FF}    %
\definecolor{defFrame}{HTML}{2F855A}    %
\definecolor{defBack}{HTML}{F0FFF4}     %
\definecolor{remFrame}{HTML}{4A5568}    %
\definecolor{remBack}{HTML}{F7FAFC}     %

\definecolor{req}{HTML}{E63946}
\definecolor{preq}{HTML}{1B4965}

\tcbset{
  mybox/.style={
    enhanced,
    breakable,
    boxrule=0.6pt,
    rounded corners,   %
    arc=3pt,           %
    boxrule=1.1pt,            %
    coltitle=black,
    fonttitle=\bfseries,
    before skip=10pt, after skip=10pt,
  },
}

\tcbset{
  headbar/.style={
    mybox,
    attach boxed title to top left={xshift=0.8em,yshift*=-\tcboxedtitleheight/2},
    boxed title style={sharp corners, boxrule=0pt, colback=#1!25},
    colback=#1!7,
    colframe=#1!85,
  },
}

\tcolorboxenvironment{theorem}{headbar=theoFrame}
\tcolorboxenvironment{lemma}{headbar=theoFrame}
\tcolorboxenvironment{corollary}{headbar=theoFrame}

\tcolorboxenvironment{assumption}{headbar=defFrame, colback=defBack, colframe=defFrame}
\newtcolorbox{remarkbox}[2][]{%
  headbar=remFrame,  %
  title={#2},        %
  #1                 %
}

\definecolor{propFrame}{HTML}{B7791F} %
\definecolor{propBack}{HTML}{FFFBEB}  %

\tcolorboxenvironment{proposition}{%
  headbar=propFrame,
  colback=propBack,
  colframe=propFrame,
}

\tcbuselibrary{skins,breakable,raster} %
\lstdefinestyle{py}{
  language=Python,
  basicstyle=\ttfamily\footnotesize,
  keywordstyle=\color{RoyalBlue},
  commentstyle=\color{DimGray}\itshape,
  stringstyle=\color{SaddleBrown},
  showstringspaces=false,
  breaklines=true,
  columns=fullflexible
}
\newtcblisting{pycard}[1][]{
  listing engine=listings,
  listing only,
  listing options={style=py},
  colback=FloralWhite,
  colframe=PeachPuff,
  boxrule=0.6pt,
  arc=3mm,
  left=3mm,right=3mm,top=3mm,bottom=3mm,
  width=\ifcsname rastercolumnwidth\endcsname \rastercolumnwidth \else \linewidth \fi,
  equal height group=cards,   %
  enhanced,drop shadow,breakable,
  #1
}

\usepackage{amsmath}
\DeclareMathOperator*{\argmax}{arg\,max}
\DeclareMathOperator*{\argmin}{arg\,min}
\usepackage[textsize=tiny]{todonotes}

\definecolor{amber}{RGB}{206,18,86}
\definecolor{blueish}{RGB}{43,140,190}
\definecolor{purpleish}{RGB}{140,81,10}
\definecolor{ye}{HTML}{ff7f00}
\definecolor{pu}{HTML}{984ea3}
\definecolor{gre}{HTML}{4daf4a}
\definecolor{re}{HTML}{e41a1c}

\newcommand{\xxcomment}[4]{\textcolor{#1}{[$^{\textsc{#2}}_{\textsc{#3}}$ #4]}}

\newcommand{\agw}[1]{\xxcomment{amber}{A}{W}{#1}}
\newcommand{\yj}[1]{\xxcomment{cyan}{Y}{J}{#1}}

\renewcommand{\L}{\mathcal{L}}
\newcommand{\mup}{$\mu$P}

\newcommand{\E}{\mathbb{E}}

\newcommand{\D}{\mathcal{D}}

\newcommand{\rH}{\mathrm{H}}
\newcommand{\rS}{\mathrm{S}}
\renewcommand{\L}{\mathcal{L}}

\makeatletter
\newcommand{\asteriskfootnote}[1]{%
  \begingroup
  \def\@thefnmark{*}%
  \long\def\@makefntext##1{%
    \noindent\hb@xt@1.8em{\hss\@makefnmark}##1%
  }%
  \@footnotetext{#1}%
  \endgroup
}
\newcommand{\daggerfootnote}[1]{%
  \begingroup
  \def\@thefnmark{\dag}%
  \long\def\@makefntext##1{%
    \noindent\hb@xt@1.8em{\hss\@makefnmark}##1%
  }%
  \@footnotetext{#1}%
  \endgroup
}

\makeatother

\usepackage{bbm}

\newcommand{\soph}[1]{epiplexity}
\newcommand{\csoph}[1]{Epiplexity}

\definecolor{lightblue}{HTML}{18282e}
\definecolor{lighterblue}{HTML}{f2fafd}
\newtcolorbox{abox}{colback=lighterblue,colframe=lightblue}

\NewTotalTCBox{\myverb}{ O{lighterblue} v !O{} }
{ fontupper=\ttfamily,nobeforeafter,tcbox raise base,arc=0pt,outer arc=0pt,
top=0pt,bottom=0pt,left=0mm,right=0mm,
leftrule=0.3mm,rightrule=0.3mm,toprule=0.3mm,bottomrule=0.3mm,boxsep=0.8mm,
colback=#1!10!lighterblue,colframe=#1!0!black,#3}{#2}

\ShortHeadings{}{}
\firstpageno{1}

\begin{document}

\title{From Entropy to Epiplexity: Rethinking Information for Computationally Bounded Intelligence}

\author{\centering Marc Finzi$^{*1}$ \quad Shikai Qiu$^{*2}$ \quad Yiding Jiang$^{*1}$ \quad Pavel Izmailov$^2$ \quad J. Zico Kolter$^1$ \\ Andrew Gordon Wilson$^2$ \\[0.4em]
\makebox[\linewidth][c]{
\normalfont\small $^1$Carnegie Mellon University \qquad $^2$New York University}}

\editor{}

\maketitle

\begin{abstract}
 Can we learn more from data than existed in the generating process itself? Can new and useful information be constructed from merely applying deterministic transformations to existing data? Can the learnable content in data be evaluated without considering a downstream task? On these questions, Shannon information and Kolmogorov complexity come up nearly empty-handed, in part because they assume observers with unlimited computational capacity and do not target the useful information content. 
In this work, we identify and exemplify three seeming paradoxes in information theory: (1) information cannot be increased by deterministic transformations; (2) information is independent of the order of data; (3) likelihood modeling is merely distribution matching. To shed light on the tension between these results and modern practice, and to quantify the value of data, we introduce \emph{epiplexity}$^\dagger$, a formalization of information capturing what computationally bounded observers can learn from data. Epiplexity captures the structural content in data while excluding time-bounded entropy, the random unpredictable content exemplified by pseudorandom number generators and chaotic dynamical systems. 
With these concepts, we demonstrate how information can be created with computation, how it depends on the ordering of the data, and how likelihood modeling can produce more complex programs than present in the data generating process itself.
We also present practical procedures to estimate epiplexity which we show capture differences across data sources, track with downstream performance, and highlight dataset interventions that improve out-of-distribution generalization. In contrast to principles of model selection, epiplexity provides a theoretical foundation for \emph{data selection}, guiding how to select, generate, or transform data for learning systems.
\end{abstract}

\section{Introduction}

\asteriskfootnote{Equal contribution.}
\daggerfootnote{Code available at \url{https://github.com/shikaiqiu/epiplexity}.}

As AI research progresses towards more general-purpose intelligent systems, cracks are beginning to show in mechanisms for grounding mathematical intuitions. 
Much of learning theory is built around controlling generalization error with respect to a given distribution, treating the training distribution as fixed and focusing optimization effort on the choice of model. Yet modern systems are expected to transfer across tasks, domains, and objectives that were not specified at training time, often after large-scale pretraining on diverse and heterogeneous data. In this regime, success or failure frequently hinges less on architectural choices than on what data the model was exposed to in the first place.
Pursuing broad generalization to diverse out-of-distribution tasks forces a shift in perspective: instead of treating data as given and optimizing for in-distribution performance, we need to choose and curate data to facilitate generalization to unseen tasks. %
This shift makes the value of data itself a central question---how much usable, transferable information can a model acquire from training? In other words, instead of model selection, how do we perform \emph{data selection}? On this question, existing theory offers little guidance and often naively contradicts empirical observations.

Consider synthetic data, crucial for further developing model capabilities \citep{abdin2024phi, maini2024rephrasing} when existing natural data are exhausted. Existing concepts in 
information theory like the data processing inequality appear 
to suggest that synthetic data adds no additional value.
Questions about what information is transferred to a given model seem naturally within the purview of information theory, yet, quantifying this information with existing tools proves to be elusive. 
Even basic questions, such as the source of the information in the weights of an AlphaZero game-playing model~\citep{silver2018general}, are surprisingly tricky to answer. AlphaZero takes in zero human data, learning merely from the deterministic rules of the game and the AlphaZero RL algorithm, both of which are simple to describe. Yet the resulting models achieve superhuman performance and are large in size.
To assert that AlphaZero has learned little to no information in this process is clearly missing the mark, and yet both Shannon and algorithmic information theory appear to say so.

\begin{figure*}[t!]
\centering
\begin{minipage}[b]{0.99\linewidth}
    \centering
    \includegraphics[width=\linewidth]{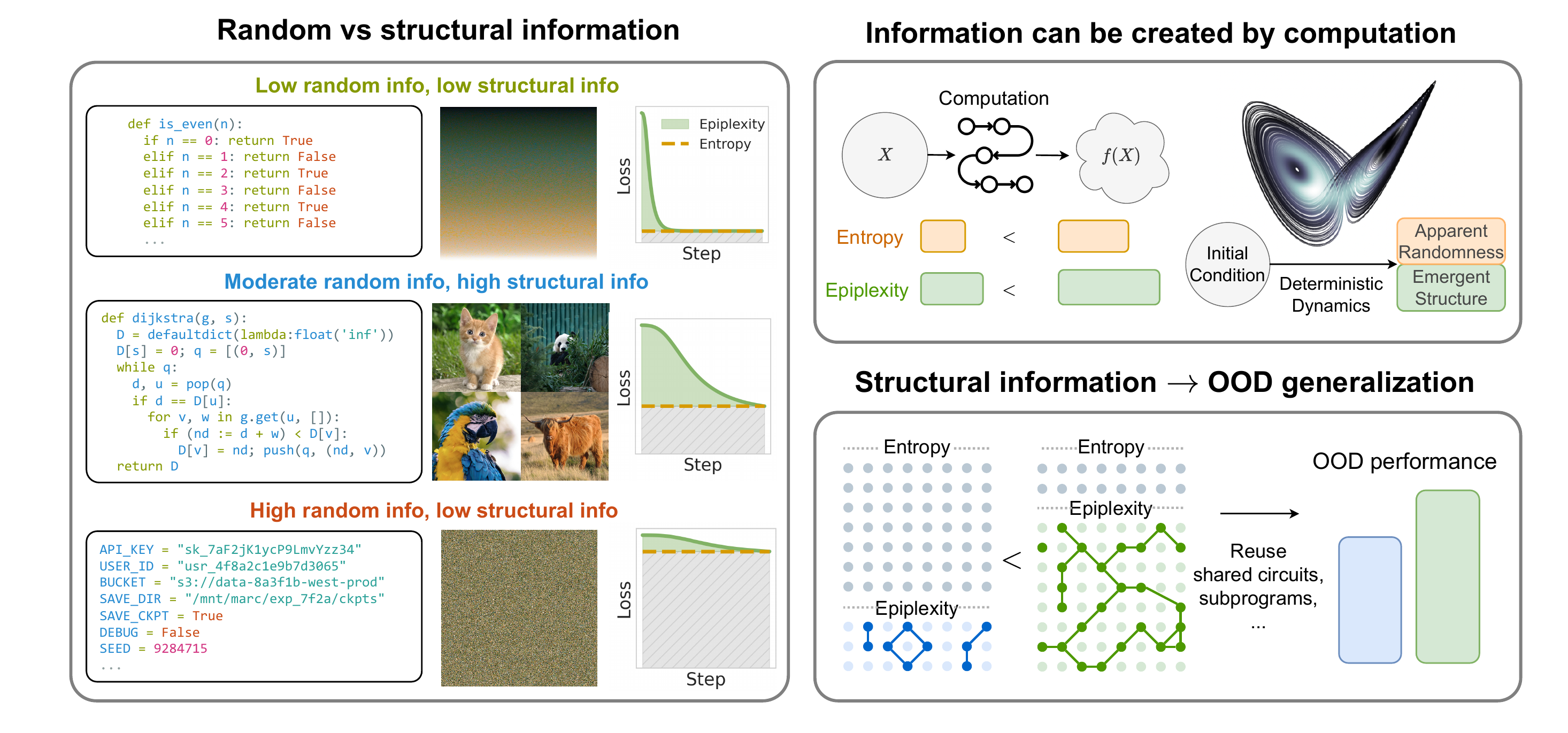}
\end{minipage}%
\vspace{1mm}

\caption{
\small
\textbf{Random vs structural information for computationally bounded observers.} (\textbf{Left}) Illustration of random vs structural information of different data for computationally bounded observers, which we formalize with time-bounded entropy and epiplexity (\Cref{sec:epiplexity}) and can be estimated from loss curves of neural networks trained on that data (\Cref{sec:measuring}).
(\textbf{Top Right}) Unlike other forms of information, time-bounded entropy and epiplexity can be increased through computational processes, such as simulating dynamical systems (cellular automation, Lorenz equations) and interventions like changing the data ordering, which can produce apparent randomness but also learnable, emergent structures like gliders and the Lorenz attractor invariant measure (\Cref{sec:paradox}). 
(\textbf{Bottom Right}) Whereas time-bounded entropy captures the in-distribution randomness and unpredictability, epiplexity measures the amount of structural information the model extracts from the data to its weights, which can be useful for OOD tasks such as by reusing learned circuits shared between the in-distribution and OOD tasks.
}
\label{fig:fig1}
\end{figure*}

In this paper, we argue that the amount of structural information a \emph{computationally bounded} observer can extract from a dataset is a fundamental concept that underlies many observed empirical phenomena. 
As we will show, existing notions from Shannon and algorithmic information theory are inadequate when forced to quantify this type of information.
These frameworks often lend intuitive or mathematical support to beliefs that, in fact, obscure important aspects of empirical phenomena. To highlight the limitations of classical frameworks and motivate the role of computational constraints in quantifying information, we identify and demonstrate three \emph{apparent paradoxes}: statements which can be justified mathematically by Shannon and algorithmic information theory, and yet are in tension with intuitions and empirical phenomena. %
\begin{enumerate}
[
    label=Paradox~\arabic*:,
    leftmargin=3em,
    labelsep=0.5em,
    align=left
]

    \item \textbf{Information cannot be increased by deterministic processes.}
    For both Shannon entropy and Kolmogorov complexity, deterministic transformations cannot meaningfully increase the information content of an object. And yet, we use pseudorandom number generators to produce randomness, synthetic data improves model capabilities, mathematicians can derive new knowledge by reasoning from axioms without external information, dynamical systems produce emergent phenomena, and self-play loops like AlphaZero learn sophisticated strategies from games \citep{silver2018general}. %
    \item \textbf{Information is independent of factorization order.} A property of both Shannon entropy and Kolmogorov complexity is that total information content is invariant to factorization: the information from observing first $X$ and then $Y$ is the same as observing $Y$ followed by $X$. On the other hand, LLMs learn better on English text ordered left-to-right than reverse ordered text, picking out an ``\emph{arrow of time}'' \citep{papadopoulos2024arrows, bengio2019meta}, and we have cryptography built on the existence of functions that are computationally hard to predict in one direction and easy in another. 

    \item \textbf{Likelihood modeling is merely distribution matching.}
    Maximizing the likelihood is often equated with matching the training data generating process: the true data-generating process is a perfect model of itself, and no model can achieve a higher expected likelihood. 
    As a consequence, it is often assumed that a model trained on a dataset cannot extract more structure or learn useful features that were not used in generating the data.
    However, we show that a computationally-limited observer can in fact uncover much more structure than is in the data generating process. 
    For example, in Conway's game of life the data are generated via simple programmatic rules that operate on two-dimensional arrays of bits.
    Applying these simple rules sequentially, we see emergent structures, such as different species of objects that move and interact in a predictable way.
    While an unbounded observer can simply simulate the evolution of the environment exactly, a computationally bounded observer would make use of the emergent structures and learn the different types of objects and their behaviors. 
 
\end{enumerate}

The tension between these theoretical statements and empirical phenomena can be resolved by imposing computational constraints on the observer and separating the random content from the structural content. Drawing on ideas from cryptography, algorithmic information theory, and these unexplained empirical phenomena, we define a new information measure, \textbf{epiplexity} (epistemic complexity), which formally defines the amount of structural information that a computationally bounded observer can extract from the data (\autoref{sec:epiplexity}, Definition~\ref{def:epiplexity}).
Briefly, epiplexity is the information in the model that minimizes the description length of data under computational constraints. A simple heuristic measurement is the area under the loss curve above the final loss, while a more rigorous approach uses the cumulative KL divergence between a teacher and student model (\autoref{sec:measuring}, \autoref{fig:measuring}).

Our definitions capture the intuition that an object contains both random, inherently unpredictable information (entropy), and predictable structured information that enables observers to generalize by identifying patterns (epiplexity). In \autoref{fig:fig1} (left) we illustrate this divide. In the top row, we have highly redundant and repetitive code and simple color gradients, which have little information content, be it structural or random. In the middle row, we have the inner workings of an algorithm and pictures of animals, showing complex, long-range interdependencies between the elements from which a model can learn complex features and subcircuits that are helpful even for different tasks. In contrast, on the bottom, we have random data with little structure: configuration files with randomly generated API keys, file paths, hashes, arbitrary boolean flags have negligible learnable content and no long-range dependencies or complex circuits that result from learning on this task. Similarly, uniformly shuffled pixels from the animal pictures have high entropy but are fundamentally unpredictable, and no complex features or circuits arise from training on these data.

An essential property of our formulation is that information is \emph{observer dependent}: the same object may appear random or structured depending on the computational resources of the observer. For instance, the output of a strong pseudorandom generator appears indistinguishable from true randomness to any polynomial-time observer lacking the secret key (seed), regardless of the algorithm or function class. In other situations, such as chaotic dynamical systems, both apparently random behavior is produced along with structure: the state of the system cannot be predicted precisely over long time-scales, but such observers may still learn meaningful predictive distributions, as shown by the invariant measure in \Cref{fig:fig1} (top right).

Models trained to represent these distributions are computer programs, and substructures within these programs, like circuits for performing specific tasks, or induction heads \citep{olsson2022context}, can be reused even for seemingly unrelated data. This view motivates selecting high epiplexity data that induces more structural information in the model, since these structures can then be reused for unseen out-of-distribution (OOD) tasks, as illustrated in \Cref{fig:fig1} (bottom right). We emphasize, however, that epiplexity is a measure of information, \emph{not} a guarantee of OOD generalization to specific tasks. Epiplexity quantifies the amount of structural information a model extracts, while being agnostic to whether these structures are relevant to a \emph{specific} downstream task. 

To build intuition, we explore a range of phenomena and provide experimental evidence for behaviours that are poorly accounted for by existing information-theoretic tools, yet naturally accommodated by epiplexity. We show that information \emph{can} be created purely through computation, giving insights into synthetic data (\autoref{sec:deterministic-info-creation}). We examine how certain factorizations of the same data can increase structural information and downstream OOD performance—even as they result in worse training loss (\autoref{sec:factorization}). We show why likelihood modeling is more than distribution matching, identifying induction and emergence as two settings where the observer can learn more information than was present in the data generating process (\autoref{sec:likelihood}). By measuring epiplexity, we can better understand why pre-training on text data transfers more broadly than image data, and why certain data selection strategies for LLMs are empirically successful (\autoref{sec:ood}). Together, our results provide clarity on the motivating questions: the information content of data can be compared independently of a specific task, new information can be created by computation, and models can learn more information than their generating processes contain.

In short, we identify a disparity between existing concepts in information theory and modern practice, embodied by three apparent paradoxes, and introduce epiplexity as a measurement of structural information acquired by a computationally bounded observer to help resolve them. We formally define epiplexity in \autoref{sec:epiplexity} (Definition~\ref{def:epiplexity}) and present measurement procedures in \autoref{sec:measuring}. In \autoref{sec:paradox}, we show how epiplexity and time-bounded entropy shed light on these paradoxes, including induction and emergent phenomena. Finally, in \autoref{sec:ood}, we demonstrate that epiplexity correlates with OOD generalization, helping explain why certain data enable broader generalization than others.

\section{Background}
In order to define the interesting, structural, and predictive component of information, we must separate it out from random information---that which is fundamentally unpredictable given the computational constraints of the observer. Along the way, we will review algorithmic randomness as developed in algorithmic information theory as well as notions of pseudo-randomness used in cryptography, and how these concepts crucially depend on the observer.

\subsection{What Does it Mean for An Object to Be Random?}
\textbf{Random Variables and Shannon Information.}
Many common intuitions about randomness start from random variables and Shannon information. A random variable defines a map from a given measurable probability space to different outcomes, with probabilities corresponding to the measure of the space that lead to a certain outcome. %
Shannon information assigns to each outcome $x$ a self-information (or surprisal) $\log 1/P(x)$ based on the probability $P$, and an entropy for the random variable $\mathrm{H}(X)=\mathbb{E}[\log 1/P(X)]$, which provides a lower bound on the average code length needed to \emph{communicate} samples to another party \citep{shannon1948mathematical}. In Shannon's theory, information comes only from distributions and random variables---objects that are not random must contain no information. As a result, non-random information is seemingly contradictory, and thus we must draw from a broader mathematical perspective to describe such concepts.

In the mid 1900s, mathematicians were interested in formalizing precisely what it means for a given sample to be a random draw from a given distribution, to ground the theory of probability and random variables \citep{shafer2006sources}. A central consideration involves a uniformly sampled binary sequence $u_{1:\infty}$ from which other distributions of interest can be constructed. This sequence can also be interpreted as the binary expression of a number $[0,1)$. Intuitively, one might think that all sequences should be regarded as equally random, as they are all equally likely according to the probability distribution: $1111111\dots$ has the same probability mass as $10011101\dots$ and also the same self-information. However, looking at statistics on these sequences reveals something missing from this perspective; from the law of large numbers, for example, it must be that $\lim_{N\to \infty}\frac{1}{N}\sum_{i=1}^Nu_i=0.5$, which is clearly not satisfied by the first sequence of $1$s.

\textbf{Martin-L\"of Randomness: No algorithm exists to predict the sequence.}
Initial attempts were made to formalize randomness as sequences which pass all statistical tests for randomness, such as the law of large numbers for selected substrings. However, under such definitions all sequences fail to be random since tests like $u_{1:\infty}\ne y_{1:\infty}$ for any particular sequence $y$ must also be included \citep{downey2019algorithmic}. The solution to these issues was found by defining random sequences not as those that pass all tests of randomness, but those that pass all \emph{computable} tests of randomness, in a formalization known as Martin-L\"of randomness \citep{martin1966definition}. As it turned out, this definition is equivalent to a number of seemingly distinct definitions, such as the inability for any gambler to exploit properties of the sequence to make a profit, or that all prefixes of the random sequence should be nearly incompressible \citep{terwijn2016mathematical}. 
For this last definition, we must invoke Kolmogorov complexity, a notion of compressibility and a key concept in this paper.

\begin{definition}[Prefix Kolmogorov complexity {\citep{Kolmogorov01011968,chaitin1975theory}}]
Fix a \\ universal prefix-free Turing machine $\mathcal{U}$. The (prefix) Kolmogorov complexity of a finite binary string $x$ is $K(x)\;=\;\min\{\,|p|:\; \mathcal{U}(p)=x\,\}$.
That is, $K(x)$ is the length of the shortest self-delimiting program (a program which also encodes its length) that outputs $x$ and halts. The conditional complexity $K(x|y)$ is the length of the shortest program that outputs $x$ and halts when provided $y$ as input.
\end{definition} Due to the universality of Turing machines, the Kolmogorov complexity for two Turing machines (or programming languages) $\mathcal U_1$ and $\mathcal U_2$ differ by at most a constant, $|K_{\mathcal{U}_1}(x)-K_{\mathcal{U}_2}(x)|\le C$, where the constant $C$ depends only on $\mathcal U_1,\mathcal U_2$, but not on $x$ \citep{li2008introduction}. %

\begin{definition}[Martin--L\"of random sequence \citep{martin1966definition}] An infinite sequence \\
$x_{1:\infty} \in \{0,1\}^{\mathbb{N}}$ is Martin--L\"of random iff there exists a constant $c$ such that for all $n$, $K(x_{1:n}) \;\ge\; n - c$. Using this criterion, all computable randomness tests are condensed into a single incomputable randomness test concerning Kolmogorov complexity. 
\end{definition}

One can extend Martin-L\"of randomness to finite sequences. We say that a sequence $x \in \{0,1\}^n$ is $c$-random if $K(x)>n-c$. Equivalently, \emph{randomness discrepancy} is defined as $\delta(x) = n-K(x)$, which measures how far away $x$ is from having maximum Kolmogorov complexity. A sequence $x$ is $c$-random if $\delta(x)<c$. High Kolmogorov complexity, low randomness discrepancy, sequences are overwhelmingly likely when sampled from uniform randomly sampled random variables. From Kraft's inequality \citep{Kraft1949Device,McMillan1956TwoInequalities}, there are at most $2^{n-c}$ (prefix-free) programs of length $L\le n-c$, therefore in the $2^n$ possibilities in uniformly sampling $X\sim U_n$, the probability that $K(X)$ is size $L$ or smaller is $P(K(X)\le n-c) = P(\delta(X) \ge c) < 2^{-c}$. The randomness discrepancy of a sequence can thus be viewed as a test statistic for rejecting the null hypothesis that the object $X$ was indeed sampled uniformly at random \citep{grunwald2008algorithmic}. For a sequence to have low randomness discrepancy, it must exhibit no discernible pattern, and thus there is an objective sense in which $1001011100$ is more random than $0101010101$.

Given the Martin-L\"of definition of infinite random sequences, every random sequence is incomputable; in other words, there is no program that can implement the function $\mathbb{N}\to \{0,1\}$ which produces the bits of the sequence. One should contrast such random numbers from those like $\pi/4$ or $e/3$, which though transcendental, are computable, as there exist programs that can compute the bits of their binary expressions. While the computable numbers in $[0,1)$ form a countable set, algorithmically random numbers in $[0,1)$ are uncountably large in number. With the incomputability of random sequences in mind we can appreciate the Von Neumann quote
\begin{quote}
\emph{“Anyone who considers arithmetical methods of producing random digits is, of course, in a state of sin.”} \citep{vonNeumann1951RandomDigits}
\end{quote}
which anticipates the Martin--L\"of formalization that came later. But this viewpoint also misses something essential, as evidenced by the success of pseudorandom number generation, derandomization, and cryptography.

\textbf{Cryptographic Randomness: No polynomial time algorithm exists to predict the sequence.}
An important practical and theoretical development of random numbers has come from the cryptography community, by once again limiting the computational model of the observer.

Rather than passing all computable tests as with Martin-L\"of randomness, cryptographically secure pseudorandom number generators (CSPRNG or PRG) are defined as functions which produce sequences that pass all \emph{polynomial time} tests of randomness.
Such functions are conjectured to be constructible by computer programs and are central to cryptographic research.

\begin{definition}[Non-uniform PRG~\citep{blum1984generate, goldreich2001foundations1}]
A function $G$ stretching $k$ input bits into $n$ output bits is a pseudorandom generator (PRG) if its outputs cannot be distinguished from a random sequence by any polynomial time algorithm more than a negligible fraction of the time. More precisely, $G$ is a (non-uniform) PRG iff for every non-uniform probabilistic polynomial time algorithm $D_k:\{0,1\}^n\to \{0,1\}$ (making use of advice strings \(\{a_k\}_{k\in\mathbb{N}}\) of length \(\mathrm{poly}(k)\))
has at most negligible advantage $\epsilon(k)$ distinguishing outputs of $G$ from uniformly random sequences $u\sim U_n$:
\begin{equation}
    \left|\Pr_{s\sim U_k} [D_n(G(s))=1] - \Pr_{u\sim U_n}[D_n(u)=1]\right| \ =\epsilon(k) < \ \mathrm{negl}(k)\,. \footnote{Here $\mathrm{negl}(k)$ means that the function decays faster than the reciprocal of any polynomial, i.e., $\mathrm{negl}(k) < \tfrac{1}{k^c}$ 
for all integers $c>0$ and sufficiently large $k$.}
\end{equation}
 
\end{definition}
The definition of indistinguishability via polynomial time tests is equivalent to a definition on the failure to predict the next element of a sequence given the previous elements: no polynomial time predictor can predict the next bit of the sequence with probability negligibly better than random guessing \citep{Yao1982Trapdoor}.

Following from the indistinguishability definition, randomness of this kind can be substituted for Martin-L\"of randomness in the vast majority of practical circumstances.\footnote{Specifically, when the difference between outcomes can be measured in polynomial time.} For a concrete example, if a use-case of randomness that runs in polynomial time like quicksort, and takes more iterations to run with PRG sequences than with truly random sequences, and this difference could be determined within polynomial time such as by measuring the quicksort runtime, then this construction could be used as a polynomial time distinguisher, 
which by the definition of PRG does not exist. If PRGs exist, then quicksort must run nearly as fast using pseudorandom number generation as it does with truly random sequences. %

The existence of PRGs hinges on the existence of \emph{one way functions} (OWF), from which PRGs and other cryptographic primitives are constructed, forming the basis of modern cryptography \citep{goldreich1989hard}. For example, the backbone algorithm for parallel random number generation in Jax \citep{jax2018github}, works to create random numbers $u_1,u_2,\dots u_N$ by simply encrypting the numbers $1,2,\dots, N$: $u_k = E(k,s)$ where the encryption key $s$ is the random seed and $E$ is the threefish block cypher \citep{salmon2011parallel}. Block ciphers, like other primitives, are constructed using one way functions.

\begin{definition}[Non-uniform one-way function, OWF~\citep{Yao1982Trapdoor, goldreich2001foundations1}]\label{def:owf} \
Let \(f:\{0,1\}^n \to \{0,1\}^m\) (with $m>n$) be computable in time \(\mathrm{poly}(n)\) where \(n=|x|\).
We say \(f\) is \emph{one-way against non-uniform PPT adversaries} if for every non-uniform probabilistic polynomial time
algorithm \(A_n\) (i.e., a polynomial-time algorithm $A$ with advice strings \(\{a_n\}_{n\in\mathbb{N}}\) of length \(\mathrm{poly}(n)\)),
\[
\Pr_{x \sim U_n}\!\left[\, A_n(f(x)) \in f^{-1}(f(x)) \,\right] \;<\; \mathrm{negl}(n),
\]
where the probability is over the uniform choice of \(x\) (and any internal randomness in \(A\)).%
\end{definition}

While cryptographers are most interested in the polynomial versus nonpolynomial compute separations for security, cryptographic primitives with respect to less extreme compute separations have been constructed and are believed to exist, for example for quadratic time \citep{merkle1978secure}, quasipolynomial time \citep{liu2024direct}, and even constraints on circuit depth \citep{applebaum2016cryptographic}. While the results we prove in this paper are based on the polynomial vs nonpolynomial separation in cryptographic primitives, it seems likely that a much wider array of compute separations are relevant for information in the machine learning context even if not as important for cryptography. For example, the separations between quadratic or cubic time and higher order polynomials may be relevant to transformer self attention, or gaps between fixed circuit depth and variable depth as made possible with chain of thought or other mechanisms.

\subsection{Random vs Structural Information}
\label{sec:sophistication}

With these notions of randomness in hand, we can use what is random to define what is not random. In algorithmic information theory, there is a lesser known concept that captures exactly this idea, known as \emph{sophistication} \citep{koppel1987structure}, which has no direct analog in Shannon information theory. While several variants of the definition exist, the most straightforward is perhaps the following:
\begin{definition}[Naive Sophistication \citep{mota2013sophistication}]\label{def:nsoph}
Sophistication, like Kolmogorov complexity, is defined on individual bitstrings, and it uses the compressibility criterion from Martin-L\"of randomness to carve out the random content of the bitstring. Sophistication is defined as the smallest Kolmogorov complexity of a set $S$ such that $x$ is a random element from that set (at randomness discrepancy of $c$).
\begin{equation}
    \mathrm{nsoph}_c(x) = \min_S: \{K(S): K(x\mid S)>\log |S|-c\}
\end{equation}
\end{definition}
Informally, sophistication describes the structural component of an object; however, it is surprisingly difficult to give concrete examples of high sophistication objects. %
The difficulty of finding high sophistication objects
is a consequence of Chaitin's incompleteness theorem \citep{chaitin1974information}. This theorem states that in a given formal system there is a constant $L$ for which there are no proofs that any specific string $x$ has $K(x)>L$, even though nearly all strings have nearly maximal complexity. Since $\mathrm{nsoph}_c(x)>L$ implies $K(x)>L-O(1)$, there can be no proofs that the sophistication of a particular string exceeds a certain constant either. It is known that high sophistication strings exist by a diagonalization argument \citep{antunes2006sophistication}, but we cannot pinpoint any specific strings which have high sophistication. 
On typical Turing machines, $L$ is often not more than a few thousand \citep{chaitin1998limits}, far from the terabytes of information that frontier AI models have encoded. %

We look towards complex systems and behaviors as likely examples of high sophistication objects; however in many of these cases the objects could conceivably be produced by simpler descriptions given tremendous amounts of computation. The mixing of two fluids for example can produce extremely complex transient behavior due to the complexities of fluid dynamics; however, with access to unlimited computation and some appropriately chosen random initial data one should be able to reproduce the exact dynamics~\citep{aaronson2014quantifying}.
Owing to the unbounded compute available for the programs in sophistication, many complex objects lose their complexity. Additionally, for strings that \emph{do} have high sophistication, the steps of computation required for the optimal program grow faster than any computable function with the sophistication content~\citep{ay2010effective}.
For a computationally bounded observer, an encrypted message or a \emph{cryptographically secure pseudo-random number generator} (CSPRNG) output \emph{is} random, and measurements that do not recognize this randomness do not reflect the circumstances of this observer. 
These limitations of sophistication leads to a disconnect with real systems with observers that have limited computation, and it is our contention that this disconnect is an essential one, central to phenomena such as emergence, induction, chaos, and cryptography. %

\subsection{The Minimum Description Length Principle}

Finally, we review the minimum description length principle (MDL), used as a theoretical criterion for model selection, which we will use in defining epiplexity. The principle states that among models for the data, the best explanation minimizes the total description length of the data, including both the description of the data using the model and the description of the model itself \citep{rissanen2004minimum}. The most common instantiation of this idea is via the statistical two-part code MDL.

\begin{definition}[Two-part MDL~\citep{rissanen2004minimum,grunwald2007minimum}]
Let $x \in \{0,1\}^{n \times d}$ be the data and $\mathcal{H}$ be a set of candidate models. The two-part MDL is:
$$L(x) = \min_{H\in \mathcal{H}} L(H) -\log P(x \mid H),$$
where $L(H)$ specifies the number of bits required to encode the model $H$, and $-\log P(x \mid H)$
is the number of bits required to encode the data given the model.
\end{definition}
This formulation provides an intuitive implementation of Occam's Razor: complex models (large $L(H)$) are penalized unless they provide a reduction in the data's description length (large $P(x\mid H)$).
If there are repeating patterns in the data, they can be stored in the model $H$ rather than being duplicated in the code for the data. 
We review the modern developments of MDL in Appendix~\ref{app:mdl}.
While MDL is a criterion for model selection given a fixed dataset, epiplexity, which we introduce next, can be viewed as its dual: a criterion for data selection given a fixed computation budget.

\section{Epiplexity: Structural Information Extractable by a Computationally Bounded Observer}
\label{sec:epiplexity}

Keeping in mind the distinction between structural and random information in the unbounded compute setting, and the computational nature of pseudorandomness in cryptography, we now introduce epiplexity. \emph{Epiplexity} captures the structural information present to a computationally bounded observer. As the computational constraints of this observer change, so too does the division between random and structured content. After introducing epiplexity here, we present ways of measuring epiplexity in \autoref{sec:measuring}. In Sections \ref{sec:paradox} and \ref{sec:ood} we show how epiplexity can shed light on seeming paradoxes in information theory around the value of data, and OOD generalization. 

First we will define what it means for a probability distribution to have an efficient implementation, requiring that it be implemented on a prefix-free universal Turing machine (UTM) and halt in a fixed number of steps.

\begin{definition}[Time-bounded probabilistic model]
\label{def:PT}
Let $T:\mathbb{N}\to\mathbb{N}$ be a non-decreasing time-constructible function and let $\mathcal{U}$ be a fixed prefix-free universal Turing machine.
A (prefix-free) program $\mathrm{P}$ is a \emph{$T$-time probabilistic model} over $\{0,1\}^n$ if it supports both sampling and probability evaluation in time $T(n)$:

\textbf{Evaluation.} On input $(0,x)$ with $x\in\{0,1\}^n$, $\mathcal{U}(\mathrm{P},(0,x))$ halts within $T(n)$ steps
and outputs an element in $[0,1]$ (with a finite binary expansion), denoted
\[
\mathrm{Prob}_{\mathrm{P}}(x)\;:=\;\mathcal{U}(\mathrm{P},(0,x)).
\]

\textbf{Sampling.} On input $(1,u)$ where $u\in\{0,1\}^\infty$ is an infinite random tape,
$\mathcal{U}(\mathrm{P},(1,u))$ halts within $T(n)$ steps and outputs an element of $\{0,1\}^n$, denoted
\[
\mathrm{Sample}_{\mathrm{P}}(u)\;:=\;\mathcal{U}(\mathrm{P},(1,u)).
\]

These outputs must define a normalized distribution matching the sampler:
\[
\sum_{x\in\{0,1\}^n} \mathrm{Prob}_{\mathrm{P}}(x)=1
\quad\text{and}\quad
\Pr_{u\sim U_\infty}[\mathrm{Sample}_{\mathrm{P}}(u)=x]=\mathrm{Prob}_{\mathrm{P}}(x)\ \ \forall x\in\{0,1\}^n.
\]

Let $\mathcal{P}_T$ be the set of all such programs. To simplify the notation, we will use italicized $P$ to denote the probability mass function $\mathrm{Prob}_{\mathrm{P}}$ in contrast with the non-italicized $\mathrm{P}$, which denotes the program.

\end{definition}

Here, $n$ denotes the dimension of the underlying sample space (e.g., the length of the binary string.) %
This definition allows us to constrain the amount of computation the function class can use. %
Such a model class enforces that the functions of interest are both efficiently sampleable and evaluable, which include most sequence models.
While in this work we focus primarily on computational constraints which we consider most fundamental, other constraints such as memory or within a given function class $\mathcal{F}$ can be accommodated by replacing $\mathcal{P}_T$ with $\mathcal{P}_{\mathcal{F}}$, and may be important for understanding particular phenomena.\footnote{One such possibility is to constrain the function class to all models reachable by a given optimization procedure with a given neural network architecture.} With these preliminaries in place, we can now separate the random and structural components of information.

We define epiplexity and time-bounded entropy in terms of the program which achieves the best expected compression of the random variable $X$, minimizing the two-part code length (model and data given model bits) under the given runtime constraint. 

\begin{remarkbox}{}
\begin{definition}[Epiplexity and Time-Bounded Entropy]
\label{def:epiplexity}
Consider a random variable $X$ on $\{0,1\}^n$. Let
\begin{equation}
\mathrm{P^\star} = \argmin_{\mathrm{P}\in \mathcal P_T}\qty{|\mathrm{P}| + \E[\log 1/P(X)]}
\end{equation}
be the program that minimizes the time bounded MDL with ties broken by the smallest program, and expectations taken over $X$. $|\mathrm{P}|$ denotes the length of the program $\mathrm{P}$ in bits, and logarithms are in base $2$. We define the $T$-bounded \emph{epiplexity} $\mathrm{S}_T$ and \emph{entropy} $\mathrm{H}_T$ of the random variable $X$ as  
\begin{equation}
    \mathrm{S}_T(X):=|\mathrm{P}^\star|, \quad \text{and} \quad  \mathrm{H}_T(X): = \E[\log 1/P^\star(X)].
\end{equation}
\end{definition}
\normalsize
\end{remarkbox}

\vspace{5mm}

The time-bounded entropy $\mathrm{H}_T$ captures the amount of information in the random variable that is random and unpredictable, whereas the epiplexity $\mathrm{S}_T$ captures the amount of structure and regularity visible within the object at the given level of compute $T$. Uniform random variables have trivial epiplexity because a model (or equivalently a program) as simple as the uniform distribution achieves a small two-part code length, despite having large time-bounded entropy. Explicitly, for a uniform random variable $U_n$ on $\{0,1\}^n$, and even a constant time bound $T(n)\ge c_1$, $\rS_T(U_n)+\rH_T(U_n) \le n+c_2$ where $c_2$ is the length of a program for the uniform distribution running in time $c_1$, and since $\rH_T(U_n) \ge \rH(U_n)=n$, it must be that $\rS_T(U_n) \le c_2$.
Random variables with simple patterns, like $0101010101...$ with probability $1/2$ and $1010101010...$ with probability $1/2$, also have low epiplexity because the time bounded MDL minimal model is simple. In this case with linear time $T(n) = \Theta(n)$, both $\rS_T(X)= O(1)$ and $\rH_T(X) = O(1)$. Henceforth, we will abbreviate $\mathrm{MDL}_T(X):=\mathrm{S}_T(X)+\mathrm{H}_T(X)$, which is the total time-bounded information content. We will now enumerate a few basic consequences of these definitions.

\vspace{5mm}

\begin{remarkbox}{Basic Properties}\label{eq:basic}
\[
\begin{aligned}
\text{(1)}\quad 
  & \rS_T(X)\ge 0, \quad \rH_T(X) \ge 0,\\[0.4ex]
\text{(2)}\quad 
  & \rH(X)\le \rS_T(X)+\rH_T(X) \le n+c_1,\\[0.4ex]
\text{(3)}\quad 
  & \mathrm{MDL}_{T'}(X) \le \mathrm{MDL}_{T}(X)
    \quad \text{whenever } \ T'(n)\ge T(n),\\[0.4ex]
\text{(4)}\quad 
  & \mathrm{MDL}_{T'}(f^{-1}(X))
    \le \mathrm{MDL}_T(X) + |\mathrm{f}|+c_2,
  \text{with } T'(n)=T(n)+\mathsf{Time}(\mathrm f).
\end{aligned}
\]
\normalsize
\end{remarkbox}
\vspace{5mm}
Statement 4 (defined for programs $\mathrm{f}$ that run in a fixed time implementing a bijection) is an analog of the information non-increase property $K(f(x)) \le K(x)+K(f)+c$. However, note that while the Kolmogorov complexity for $K(f)$ and $K(f^{-1})$ are the same to within an additive constant, in our setting of a fixed computational budget having a short program for $f^{-1}$ does not imply one for $f$, and vice versa. This gap between a function and its inverse has important consequences for the three paradoxes as we will see in \autoref{sec:paradox}. %

\paragraph{Pseudorandom number sequences have high random content and little structure.} 
Unlike Shannon entropy, Kolmogorov complexity, or even resource bounded forms of Kolmogorov complexity \citep{allender2011pervasive}, we show that CSPRNGs have nearly maximal time-bounded entropy for polynomial time observers. Additionally, while CSPRNGs produce random content, they do not produce structured content as the epiplexity is negligibly larger than constant.
Formally, let $U_k$ be the uniform distribution on $k$ bits.%

\begin{theorem}
\label{thm:csprng_entropy}
For any $G \in \mathrm{PRG}$ that stretches the input to $n=\mathrm{poly}(k)$ bits and allowing for an advantage of at most $\varepsilon(k)$, the polynomial time bounded entropy is nearly maximal:
$$
    n - 2 - \,n\varepsilon(k) < \mathrm{H}_{\mathrm{Poly}}(G(U_k)) \le n+c
$$
for a fixed constant $c$, and epiplexity is nearly constant
$$
\rS_{\mathrm{Poly}}(G(U_k)) \leq c + n\varepsilon(k).
$$
\\
Proof: see Appendix~\ref{app:csprng}.
\end{theorem}
In contrast, the Shannon entropy is $\rH(G(U_k))=k$, polynomial time bounded Kolmogorov complexity will be at most $k+c$ (assuming $n$ is fixed or specified ahead of time) as there is a short and efficiently runnable program $G$ which produces the output, and similarly with other notions such as Levin complexity~\citep{LiVitanyi2008} or time bounded Kolmogorov complexity~\citep{allender2011pervasive}. Taken together, these results show that epiplexity appropriately characterizes pseudorandom numbers as carrying a large amount of time-bounded randomness but essentially no learnable structure, exactly as intuition suggests.

\paragraph{Existence of Random Variables with High Epiplexity.} 
One may wonder whether any high epiplexity random variables exist at all. Indeed, assuming the existence of one-way functions, we can show via a counting argument that there exists a sequence of 
random variables whose epiplexity grows at least logarithmically with the dimension. %
\begin{theorem}\label{thm:high-ep}
    Assuming the existence of one-way functions secure against non-uniform probabilistic polynomial-time adversaries, there exists a sequence of random variables $\{X_n\}_{n=1}^\infty$ over $\{0,1\}^n$ such that $$\mathrm{S}_{\mathrm{Poly}}(X_n) =\Omega(\log n).$$ Proof: see Appendix~\ref{app:high-ep}.
\end{theorem}
This result implies that epiplexity can be unbounded; however, logarithmically growing information content only admits a very modest amount of structural information, still far from the power law scaling we see with some natural data. We also note that the argument is nonconstructive and hence does not compromise cryptographic security.
\textbf{Conditional Entropy and Epiplexity.}
To describe situations like image classification, 
where we are only interested in a function which predicts the label from the image, and not the information in generating the images, we define \emph{conditional} time-bounded entropy and epiplexity.

\begin{definition}[Conditional epiplexity and time-bounded entropy]
    For a pair of random variables $X$ and $Y$, define $\mathcal{P}_{T(n)}^X$ as the set of probabilistic models $P$ such that for each fixed $x$, the conditional model $\mathrm{P}_{Y\mid x}$ is in $\mathcal{P}_{T(n)}$. The optimal conditional model with  access to $X$ is:
    \begin{align}
        \mathrm{P}^\star_{Y \mid X} = \argmin_{\mathrm{P} \in \mathcal{P}_T^X} \left\{|\mathrm{P}| + \mathbb{E}_{(X,Y)}\left[-\log P(Y \mid X)\right]\right\}.
    \end{align}
    The conditional \emph{epiplexity} and \emph{time-bounded entropy} are defined as:
    \begin{align}
        \rS_T(Y \mid X) := \left|\mathrm{P}^{\star}_{Y\mid X}\right|, \quad \mathrm{H}_T(Y
        \mid X) := \mathbb{E}_{(X, Y)}\left[-\log P^\star_{Y\mid X}(y\mid x)\right].
    \end{align}
    These quantities are defined with respect to the time bounded MDL over programs which take as input $X,Y$ and output the probabilities over $Y$ (conditioned on $X$), and with expectations taken over both $X$ and $Y$. We note that in general this definition is not equivalent to the difference of the joint and individual entropies, $\rH_T(Y,X)-\rH_T(X)\ne \rH_T(Y|X)$.   
    Unlike Shannon entropy, we can also condition on deterministic strings, which will change the values on account of not needing such a large program $\mathrm{P}$. For example, we may be interested in the conditional epiplexity $\rS_T(X|m)$ or entropy $\rH_T(X|m)$ given a model $m$.
    For a deterministic string $d \in \{0,1\}^*$ we define the conditional epiplexity via
    \begin{align}
        \mathrm{P}^\star_{Y \mid d} = \min_{\mathrm{P} \in \mathcal{P}_T^{\{0,1\}^*}} \left\{|\mathrm{P}| + \mathbb{E}_{Y}\left[-\log P(Y \mid d)\right]\right\},
    \end{align}
    where the minimization is over time bounded functions $P(\cdot \ | \ \cdot)$ that take in the string $d$ as the second argument (which we refer to as $\mathcal{P}_T^{\{0,1\}^*}$).
\end{definition}

For the machine learning setting, we take the random variable $X$ to refer to the \emph{entire dataset} of interest, i.e. typically a collection $X=[X_1,X_2,\dots]$ of many iid samples from a given distribution, rather than a lone sample from, and $\E[\log 1/P(X)]$ scales with the dataset size.  %
Epiplexity typically grows with the size of the dataset (see detailed arguments for why this is the case in \Cref{app:epi-scaling}) as larger datasets allow identifying and extracting more intricate structure and patterns, mirroring the practice of ML training. Moreover, as we will see later, the epiplexity of a typical dataset is orders of magnitudes smaller than the random information content. While not a focus of this paper, conditioning on deterministic strings opens up the possibility to understand what additional data is most useful for a specific machine learning model, such as on top of a pretrained LLM.

\section{Measuring Epiplexity and Time-Bounded Entropy}
\label{sec:measuring}

We have now introduced epiplexity and time-bounded entropy as measures of structural and random information of the data. In this section, we present practical procedures to estimate upper bounds and empirical proxies for these quantities. Intuitively, we want to find a probabilistic model $P(\cdot)$ of the data $X$ that achieves low expected loss $\E[\log 1/P(X)]$, is described by a short program $\mathrm{P},$ and evaluating $P(X)$ takes time at most $T(|X|),$ which we will abbreviate as $T.$ Using this model, we thereby decompose the information of the data into its structural and random components, namely, (1) epiplexity $\rS_T(X)$: the length of the program $|\mathrm{P}|,$ accounting for the bits required to model the data distribution, and (2) time-bounded entropy $\rH_T(X)$: the expected length for entropy coding the data using this model, which accounts for the bits required to specify the particular realization of $X$ within that distribution. We estimate conditional epiplexity analogously, providing random variable conditioning as input into the model. %

Since directly searching over the space of programs is intractable, we restrict attention to probabilistic models parameterized by neural networks, as they achieve strong empirical compression across data modalities \citep{mackay2003information, goldblum2023no, deletang2023language, balle2018variational} and capture the most relevant ML phenomenology. %
While a naive approach is to let $\mathrm{P}$ be a program that directly stores the architecture and weights of a neural network and evaluates it on the given data, this approach can significantly overestimate the information content in the weights, particularly for large models trained on relatively little data. Instead, we will use a more efficient approach that encodes the training process that produces the weights. We will discuss two approaches for encoding neural network training processes, based on \emph{prequential coding} \citep{dawid1984present} and \emph{requential coding} \citep{finzi2026requential}, respectively. The former is more straightforward to understand and evaluate, but relies on a heuristic argument to separate structure bits from noise bits, while the latter is rigorous at the cost of being more difficult to evaluate. 
Fortunately, both approaches often yield comparable rankings of epiplexity across datasets (\Cref{sec:preq-proxy}).

Moving forward, we will measure time by the number of floating-point operations (FLOPs) and dataset size by number of tokens, so that training a model with $N$ parameters on $D$ tokens takes time approximately $6ND$ \citep{kaplan2020scaling}, while evaluating it on $X$ takes time $2N\D$ with $\D=|X|$ the number of tokens in $X.$ To distinguish $X$ from the training dataset, which we are free to choose, we will refer to $X$ as the test dataset, as it is the data we need to perform inference on.

\begin{figure*}[t]
\centering
\begin{minipage}[t]{0.32\linewidth}
    \centering
    \includegraphics[width=\linewidth]{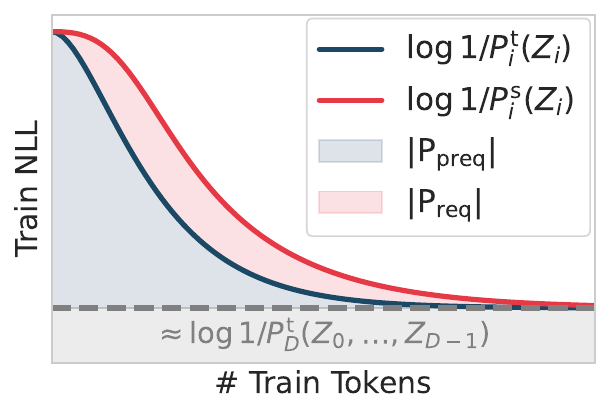}
    \subcaption{Estimate information in model}
    \label{fig:estimate_epi}
\end{minipage}%
\hfill
\begin{minipage}[t]{0.31\linewidth}
    \centering
    \includegraphics[width=\linewidth]{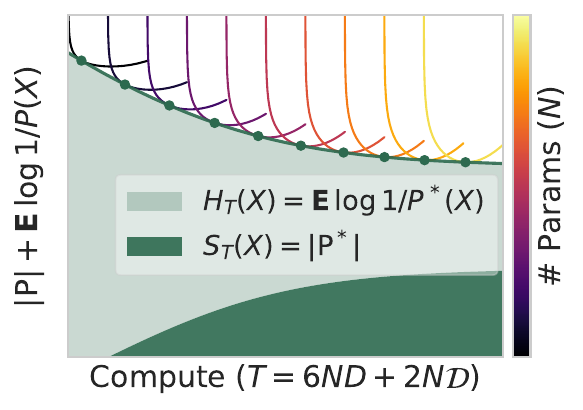}
    \subcaption{Compute-optimal 2-part code}
    \label{fig:min-two-part}
\end{minipage}%
\hfill
\begin{minipage}[t]{0.32\linewidth}
    \centering
    \includegraphics[width=\linewidth]{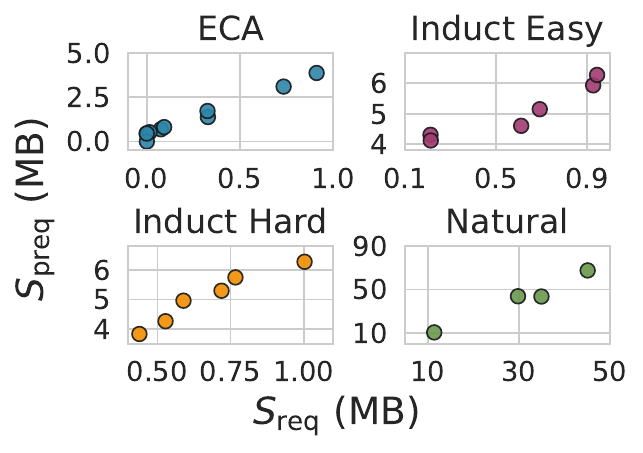}
    \subcaption{Requential vs Prequential}
    \label{fig:req_vs_preq}
\end{minipage}%
\caption{\small \textbf{How to estimate epiplexity.} (\textbf{a}) We consider two approaches for efficiently coding trained neural networks. Prequential estimation estimates information content as the area under the loss curve of a model above its final loss, with the training set matching the test data distribution. Requential coding, which provides an explicit code for $P^{\mathrm{s}}$ with expected length as the cumulative KL between a student model $P^{\mathrm{s}}$ and the teacher $P^{\mathrm{t}}$ that generates its \emph{synthetic} training data, visualized approximately by their loss gap. We typically choose $P^{\mathrm{t}}$ to be a model trained on the \emph{real} training set, as in prequential coding.
(\textbf{b}) Using either approach, we optimize hyperparameters (model size $N$, training tokens $D$, etc.) to find the shortest two-part code for each compute budget, which decomposes into the estimated epiplexity and time-bounded entropy.
(\textbf{c}) Comparing prequential and requential coding on four groups of datsets used in this work. The prequential estimate is typically larger, but the two correlate well, particularly within each group.
}
\label{fig:measuring}
\end{figure*}

\subsection{Approximating Model Description Length with Prequential Coding}\label{sec:preq}
Prequential coding provides a classic approach for compressing the training process of a neural network. We assume a batch size of one for simplicity, but generalizing to batch sizes larger than one is straightforward. Starting with a randomly initialized network $P_0$ (where the subscript indicates timestep), we proceed iteratively: at each step $i$, we entropy encode the current training token $Z_i$ using $\log 1/P_i(Z_i)$ bits, then train the model on this token to produce $P_{i+1}$. Typically $Z_i$'s are drawn i.i.d. from the same distribution as $X.$ On the side of the decoder, a synchronized model is maintained; the model decodes $Z_i$ using $P_i$ and then trains on it to produce the identical $P_{i+1}$. Omitting small constant overheads for specifying the random initialization, architecture, and training algorithm, a total of $L({Z_{:M}, P_M}) = \sum_{i=0}^{M-1} \log 1/P_i(Z_i)$ bits yields an explicit code for both the training data $Z_{:M} = \{Z_0, \ldots, Z_{M-1}\}$ and the final model weights $P_M$, which can be decoded in time $6ND$ for a model with $N$ parameters trained on $D$ tokens (typically $D> M$ as each example contains multiple tokens). Despite having an explicit code for $Z,P_M$, we cannot easily separate this into a code for $P_M$ alone for estimating epiplexity.

To isolate the description length of $P_M$ alone, we adopt the heuristic in \citet{zhang2020measuring} and \citet{finzi2025compute}: we first estimate the description length of the training data given $P_M$ as its entropy code length under the final model, $L({Z_{:M}|P_M}) = \sum_{i=0}^{M-1} \log 1/P_M(Z_i)$. Then, appealing to the symmetry of information, which states $K(P_M) = K(Z_{:M}, P_M) - K(Z_{:M}|P_M)$ up to constant terms, we estimate the description length of $P_M$ as the difference $L({Z_{:M}, P_M}) - L({Z_{:M} | P_M})$:
\begin{equation}\label{eq:preq}
    |\mathrm{P}_{\mathrm{preq}}| \,\approx \sum_{i=0}^{M-1} \qty(\log 1/P_i(Z_i) - \log 1/P_M(Z_i)).
\end{equation}
If $Z_i$ is sampled i.i.d., as is typically the case, then the code length for the model \emph{can be visualized as the area under the loss curve above the final loss} in \Cref{fig:estimate_epi}. Intuitively, the model absorbs a significant amount of information from the data if training yields a sustained and substantial reduction in loss. For random data, $\log 1/P_i(Z_i)$ never decreases, while for simple data, $\log 1/P_i(Z_i)$ drops rapidly and stabilizes, both leading to small $|\mathrm{P}_{\mathrm{preq}}|.$ We note that the prequential loss values are effectively taken on estimates of the \emph{test loss}, because they evaluate the log probabilities on a batch before it is trained on, a central detail to the coding scheme. In cases where train and test diverge, such as when there is overfitting, this difference could become important important.

Encoding the test dataset $X$ (not to be confused with the training data) using this model, we obtain a two-part code of expected length $|\mathrm{P}_{\mathrm{preq}}| + \E[\log 1/P_M(X)]$ that runs in time $6ND + 2N\D.$ We optimize the training hyperparameters (e.g., learning rate) and the trade-off between $N$ and $D$ subject to the time bound $6ND + 2N\D \leq T$ to find the optimal $P^\star$ that minimizes the two-part code within this family, and estimate epiplexity and time-bounded entropy as $\rS_T(X) = |\mathrm{P}_{\mathrm{preq}}^\star|$ and $\rH_T(X)= \E[\log1/P^\star(X)].$ The better these hyperparameters are optimized, the more accurate our estimates become. We use the Maximal Update Parameterization ($\mu$P) \citep{yang2022tensor} to ensure the optimal learning rate and initialization are consistent across model sizes, simplifying tuning. We estimate the expectation $\E[\log 1/P_M(X)]$ by its empirical value on held-out validation data, i.e., the validation loss scaled by the size of $X$. We detail the full procedure in \Cref{app:measure}, such as how we choose the hyperparameters and estimate the Pareto frontier of MDL vs compute.

While conceptually simple, practically useful, and easy to evaluate, this prequential approach to approximating epiplexity is not rigorous for two reasons. First, both $L({Z_{:M}, P_M})$ and $L({Z_{:M} | P_M})$ can only upper-bound the respective Kolmogorov complexities, and thus their difference does not yield an upper bound for $K(P_M).$\footnote{We have $L({Z_{:M}, P_M}) + O(1) \geq K(Z_{:M}, P_M),$ but not that $L({Z_{:M} | P_M}) + O(1) \leq K(Z_{:M} | P_M).$} Second, even setting this issue aside, the argument only establishes the existence of a program that encodes $P_M$ with length $|\mathrm{P}_{\mathrm{preq}}|,$ but does not guarantee that its runtime falls within $6ND,$ since the symmetry of information does not extend to time-bounded Kolmogorov complexity. Nevertheless, prequential coding can serve as a useful starting point for crudely estimating epiplexity, particularly convenient when one already has access to the loss curve from an existing training run.

\subsection{Explicitly Coding the Model with Requential Coding}
To address the shortcomings of the previous approach based on prequential coding, we adopt requential coding \citep{finzi2026requential} for constructing an explicit code of the model with a known runtime. Rather than trying to code a particular training dataset, with requential coding one can use the insensitivity to the exact data points sampled to code for \emph{a} sampled dataset that leads to a performant model but without paying for the entropy of the data. Specifically, it encodes a training run where at step $i$ a student model $P^{\mathrm{s}}_i$ is trained on a synthetic token sampled randomly from a teacher model $P^{\mathrm{t}}_i$, where the sequence $P^{\mathrm{t}}_0, \ldots, P^{\mathrm{t}}_{M-1}$ are arbitrary teacher model checkpoints. We typically choose $P^{\mathrm{t}}_i$ to be the checkpoints from training on the original \emph{real} training set, as in prequential coding. Using relative entropy coding \citep{theis2022algorithms}, the synthetic tokens $\widetilde{Z}_i \sim P^{\mathrm{t}}_i$ can be coded given only the student $P^{\mathrm{s}}_i$ (synchronized between encoder and decoder) using $\mathrm{KL}(P^{\mathrm{t}}_i\|P^{\mathrm{s}}_i) + \log\bigl(1+\mathrm{KL}(P^{\mathrm{t}}_i\|P^{\mathrm{s}}_i)\bigr) + 4$ bits in expectation.
Summing over all steps gives the requential code length for $P^{\mathrm{s}}_M$:
\begin{align}
|\mathrm{P}_{\mathrm{req}}|\,
= \sum_{i=0}^{M-1} \mathrm{KL}(P^{\mathrm{t}}_i\|P^{\mathrm{s}}_i) + \log\bigl(1+\mathrm{KL}(P^{\mathrm{t}}_i\|P^{\mathrm{s}}_i)\bigr) + 4 + O(1)
\approx \sum_{i=0}^{M-1} \mathrm{KL}(P^{\mathrm{t}}_i\|P^{\mathrm{s}}_i),
\end{align}
where the logarithmic and constant overheads are typically negligible due to large sequence length and batch size, and as before we omit the small constant cost of specifying the random initialization, architecture, and training algorithm. In addition to providing an explicit code, a key advantage of requential coding is its flexibility in choosing the teacher sequence: by selecting teachers $P^{\mathrm{t}}_i$ that remain close to the student $P^{\mathrm{s}}_i$ while still pointing toward the target distribution, we keep the per-step coding cost $\mathrm{KL}(P^{\mathrm{t}}_i\|P^{\mathrm{s}}_i)$ small while effectively guiding the student's learning.

\Cref{fig:estimate_epi} connects requential coding to the student's and teacher's loss curves: suppose we take as teachers the checkpoints $P^{\mathrm{t}}_0,\ldots,P^{\mathrm{t}}_{M-1}$ from a model trained on real data $Z_0,\ldots,Z_{M-2} \sim P_X$. For visualization, we can then estimate $\mathrm{KL}(P^{\mathrm{t}}_i\|P^{\mathrm{s}}_i)$ by the loss gap $\log 1/P^{\mathrm{s}}_i(Z_i) - \log 1/P^{\mathrm{t}}_i(Z_i)$, which is accurate when $P^{\mathrm{t}}_i \approx P_X$. 
We can thus visualize the code length for the student as approximately the area between the teacher's and student's loss curves on real data, as shown in \Cref{fig:estimate_epi}. 

The two-part code has expected length $|\mathrm{P}_{\mathrm{req}}| + \E[\log 1/P^{\mathrm{s}}_M(X)],$ consisting of first decoding $P^{\mathrm{s}}_M$ by replaying the training process, which takes time $6ND$ for a total of $D$ requential training tokens, and then evaluating $P^{\mathrm{s}}_M$ on the test dataset $X,$ taking an additional time $2N\D,$ for a total runtime of $6ND + 2N\D$. We optimize the training hyperparameters, teacher choices, and the trade-off between $N$ and $D$ subject to the specified time bound $T$ to find the optimal model $P^\star$ minimizing the two-part code, and estimate $\rS_T(X) = |\mathrm{P}_{\mathrm{req}}^\star|$ and $\rH_T(X) = \E[\log 1/P^\star(X)].$ See details in \Cref{app:detail-procedure}.

\subsection{Comparison Between the Two Approaches and Practical Recommendations}\label{sec:preq-proxy}
\Cref{fig:req_vs_preq} compares the estimated epiplexity obtained by the two approaches across four groups of datasets used in this work: ECA (\Cref{sec:deterministic-info-creation}), easy and hard induction (\Cref{sec:induction}), and natural datasets (\Cref{sec:epiplexity-natural}). While the prequential estimate is typically several times larger than the requential estimate, the two estimates correlate well, particularly within each group where the datasets yield similar learning dynamics. We detail the datasets and time bounds used in \Cref{app:preq-vs-req}. This general agreement is expected since the prequential estimate can be viewed as an approximation of requential coding with a static teacher (\Cref{app:preq-as-approx}). In general, however, the discrepancy between the two estimates will depend on particular datasets and training configurations, and a good correlation between the two is not guaranteed. 

While requential coding is the more rigorous approach, it is typically $2\times$ to $10\times$ slower than prequential coding, which requires only standard training. The overhead depends on batch size, sequence length, and inference implementation (smaller overhead for large batches and short sequences), as requential coding requires repeatedly sampling from the teacher, though it is possible that the overhead can be reduced with more efficient algorithms. Therefore, we recommend using prequential coding for crudely estimating epiplexity and ranking the epiplexity of different datasets, particularly when one has access to the loss curve from an existing expensive training run (e.g., see an application in \Cref{sec:epiplexity-natural}), and requential coding for obtaining the most accurate estimates otherwise.

\subsection{How Epiplexity and Time-Bounded Entropy Scale with Compute and Data}
Under natural assumptions about neural network training—namely, that larger models are more sample-efficient and that there are diminishing returns to scaling model size or data alone—we expect epiplexity and time-bounded entropy to exhibit certain generic scaling behavior as a function of the compute budget $T$ and dataset size $\D$. In \Cref{app:epi-scaling}, we show that, under these assumptions, the compute-optimal model size $N^\star(T)$ and training data size $D^\star(T)$ are generally increasing in the compute budget $T$, which implies that epiplexity $\rS_T(X)$ typically grows with $T$ while time-bounded entropy $\rH_T(X)$ decreases. In the infinite-compute limit, epiplexity $\rS_\infty(X)$ typically grows with the test set size $\D = |X|$, while the per-token time-bounded entropy $\rH_\infty(X)/\D$ decreases. These results align with our intuition that larger compute budgets and more data allow the model to extract more structural information from the dataset and reduce the apparent randomness remaining in each sample. However, they should be understood only as typical trends, with a counterexample shown in \Cref{sec:emergent} relating to the phenomenon of emergence.

\section{Three Apparent Paradoxes of Information}
\label{sec:paradox}

To illustrate the lacunae in existing information theory perspectives, we highlight three \emph{apparent paradoxes} of information: (1) information cannot be created by deterministic transformations; (2) total information content of an object is the same regardless of the factorization;  and (3) likelihood modeling can only learn to match the data-generating process. %
Each statement captures some existing sentiment within the machine learning community, can be justified mathematically by Shannon and algorithmic information theory, and yet seems to be in conflict with intuitions and experimental observations. In this section, we will show with both theoretical results and empirical evidence that time bounding and epiplexity help resolve these apparent paradoxes.

\subsection{Paradox 1: Information Cannot be Created by Deterministic Transformations}
\label{sec:deterministic-info-creation}
Both Shannon and algorithmic information theory state in some form that the total information cannot be increased by applying deterministic transformations on existing data. The data processing inequality (DPI) states that if some information source $W$ produces natural data $X$ that are collected, then no deterministic \emph{or stochastic} transformations used to produce $Y$ from $X$ can increase the mutual information with the variable of interest $W$: $I(Y;W) \le I(X;W)$. Similarly, information non-increase states that a deterministic transformation $f$ can only preserve or decrease the Shannon information, a property that holds pointwise $-\log P_Y(f(x)) \le -\log P_X(x)$ and in expectation: $\rH(f(X)) \le \rH(X)$ (we note $X$ here is a discrete random variable). In algorithmic information theory, there is a corresponding property: $K(f(x)) \le K(x)+K(f)+c$ for a fixed constant $c$. These inequalities appear to rule out creating new information with deterministic computational processes.

How can we reconcile this fact with algorithms like AlphaZero \citep{silver2018general} that can be run in a closed environment from a small deterministic program on the game of chess, extracting insights about the game, different openings, the relative values of pieces in different positions, tactics and high level strategy, and requiring megabytes of information stored in the weights? Similarly we have dynamical systems with simple descriptions of the underlying laws that produce rich and unexpected structures, from which we can learn new things about them and mathematics. 

We also have evidence that synthetic data is helpful for model capabilities \citep{liu2024deepseek,gerstgrasser2024model,maini2024rephrasing,openai2025gpt5systemcard}. Moreover, if we believe that the processes that create natural data could in principle have been simulated to sufficient precision on a large computer, then all data could have been equivalently replaced with synthetic data. For practical synthetic data produced from transformations of samples from a given model and prompt, this sampling is performed with pseudorandom number generators, making the entire transformation deterministic. If we consider $f$ as the transformations we use to produce synthetic data and $x$ was the limited real data we started with, these inequalities appear to state very concretely that our synthetic data adds no additional information beyond the model and training data.

\begin{figure*}[t!]
\centering

\centering
\includegraphics[width=0.49\linewidth]{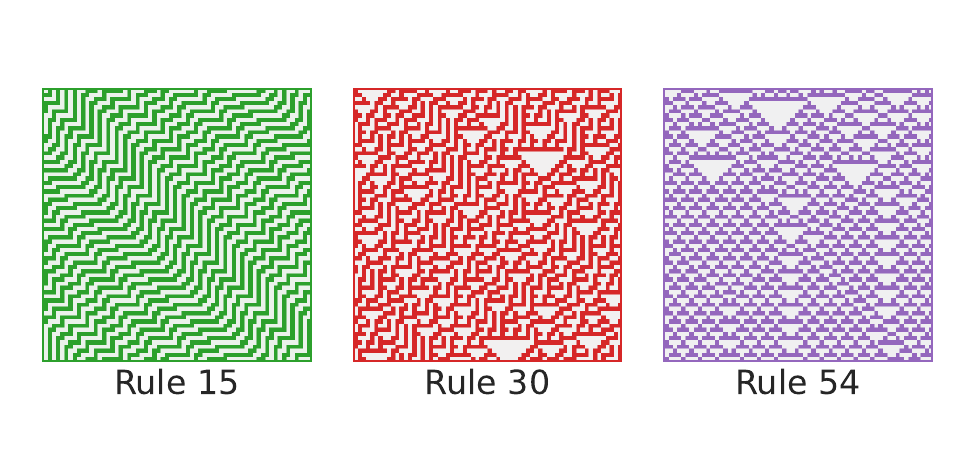}
\hfill
\centering
\includegraphics[width=0.49\linewidth]{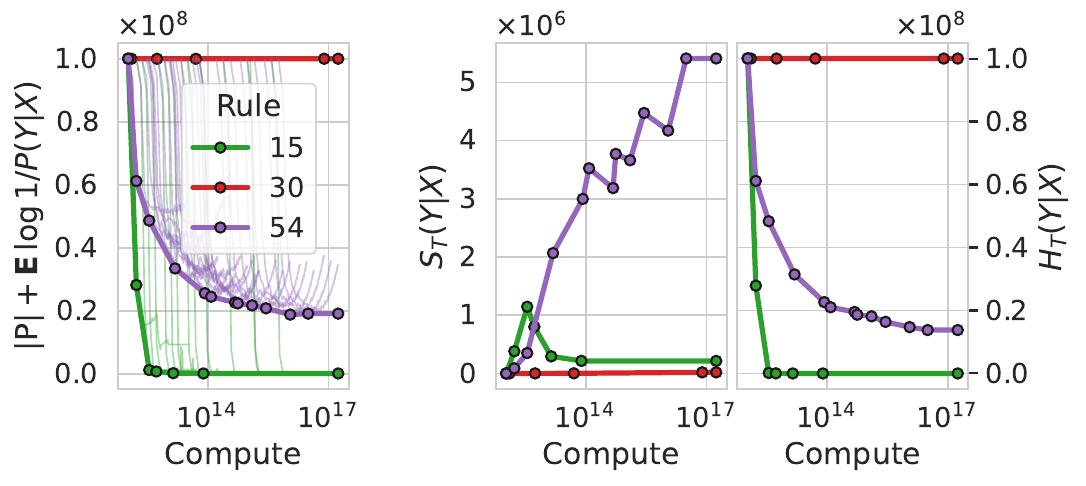}
\caption{
\small
\textbf{
Information created with cellular automata.}
(\textbf{Left}) Example rollouts from random initial conditions of the class II rule 15, class III rule 30, and class IV rule 54. Time flows from up to down.
(\textbf{Right}) Measuring epiplexity on data produced by these transformations, we see that rule 15 produces little information (low $\rH_T$, low $\rS_T)$, rule 30 produces lots of unpredictable random information (high $\rH_T$, low $\rS_T$), and rule 54 produces both random and structural information (medium $\rH_T$, high $\rS_T$). These observations are reflected in the training loss curve of LLMs, which saturates quickly for rule 15, makes no progress for rule 30, and makes continued progress with compute for rule 54.
}
\label{fig:eca}
\end{figure*}

Whatever information it is that we mean when we say that AlphaZero has produced new and unexpected insights in chess, or new theoretical results in mathematics, or with synthetic data, it is not Shannon or algorithmic information. We argue that these unintuitive properties of information theory are a consequence of assuming unlimited computation for the observer. With limited computation, a description of the AlphaZero algorithm and the result of running AlphaZero for thousands of TPU hours are distinct. To build intuition, we start with the humble CSPRNG which also creates time-bounded information through computation (albeit random information).

\begin{theorem}\label{thm:creation}
    Let $G: \{0,1\}^k \to \{0,1\}^n$ be a $\mathrm{PRG}$ which admits advantage $\varepsilon(k)$ and $U_k$ be the uniform distribution. $\rH_{\mathrm{Poly}}(G(U_k))-\rH_{\mathrm{Poly}}(U_k)> n-k-n\varepsilon(k)-c$ for a fixed constant $c$. \\
    Proof: see Appendix~\ref{sec:epi_increase}.
\end{theorem}

Notably, we have a deterministic function which dramatically increases the time-bounded information content of the input. It is worth contrasting this result with \autoref{eq:basic}, where the time-bounded information content increase from a deterministic function \emph{can} be bounded if the inverse function has a short program which can run efficiently. The statement highlights an important asymmetry between the function $G$ and its inverse with fixed computation that does not hold with unlimited computation (e.g. $K(G^{-1})=K(G)+O(1)$). Simultaneously, it provides some useful guidance for synthetic data: if we want to produce interesting information, we should make sure the functions we use do not have simple and efficiently computable inverses.

As an illustrative example, consider the iterated dynamics of elementary cellular automata \citep{wolfram2003new,zhang2024intelligence}. An elementary cellular automaton (ECA) is a one‑dimensional array of binary cells that evolves in discrete time steps according to a \emph{fixed} rule mapping each cell’s current state and the states of its two immediate neighbors to its next state. Despite their simple formulation -- only 256 possible rules---these systems can produce a rich variety of behaviors, from stable and periodic patterns to chaotic and computationally universal dynamics. We setup the problem of predicting $Y_i=F(X_i)$ from random initial data $X_i$ for $F$ being an ECA iterated $48$ times on a grid of size 64, and assemble these pairs into a dataset $X=[X_1,\dots,X_K]$ and $Y=[Y_1,\dots,Y_K]$ for a total dataset of $\D=100$M tokens. %
We measure the conditional information content $Y|X$ (epiplexity and entropy) for ECA rules 15, 30, and 54 by training LLMs on this dataset. We provide a visualization of these dynamics in \autoref{fig:eca} (left). For the class II rule 15 in the Wolfram hierarchy \citep{wolfram2003new}, the produced behavior is periodic and has a simple inverse. Consequently, in \autoref{fig:eca} (right), we see that training dynamics that rapidly converge to optimal predictions and with little epiplexity or time-bounded entropy. With the class III rule 30, the computation produces outputs that are inherently intractable to predict with limited computation, and as a result we see that there is maximal time-bounded entropy that is produced but no epiplexity. For the class IV rule 54, we see that the dynamics are complex but also partly understandable: the loss decreases slowly and much epiplexity is produced. These results highlight the sensitivity of epiplexity to the generating process. With the same compute spent and with a very similar program we can have drastically different outcomes, producing simple objects, producing only random content, and producing a mix of random and structured content.

\subsection{Paradox 2: Information Content is Independent of Factorization}
\label{sec:factorization}

An important property of Shannon's information is the symmetry of information, which states that the amount of information content does not change with factorization.
The information we acquire when predicting $x$ and then $y$ is exactly equal to when predicting $y$ and then $x$:~ Shannon entropy satisfies $\rH(Y \mid X) + \rH(X) = \rH(X,Y) = \rH(X\mid Y) + \rH(Y)$. An analogous property also holds for Kolmogorov complexity, known as the symmetry of information identity: $K(y\mid x) + K(x) = K(x \mid y) + K(y) + O(1)$.

On the other hand, multiple works have observed that natural text is better compressed (with final model achieving higher likelihoods) when modeled in the left-to-right order (for English) than when modeled in reverse order~\citep{papadopoulos2024arrows, bengio2019meta}, picking out an \textit{arrow of time} in LLMs where one direction of modeling is preferred over the other. It seems likely that for many documents, other orderings may lead to more information extracted by LLMs. Similarly, as we will show later, small rearrangements of the data can lead to substantially different losses and downstream performance. Cryptographic primitives like one way functions and block cyphers also provide examples where the order of conditioning can make all the difference to how entropic the data appears, for example considering autoregressive modeling of two prime numbers followed by their product vs the reverse ordering. These experimental results and cryptographic ideas indicate what can be learned is dependent on the ordering of the data, which in turn suggests that different amounts of ``information'' are extracted from these different orderings. 

Our time-bounded definitions capture this discrepancy. Under the existence of one way permutations, we can prove that a gap in prediction exists over different factorizations for time bounded entropy.
\begin{theorem}\label{thm:asymmtry}
    Let $f$ be a one-way permutation and let $X = U_n$ be uniform and $Y=f(X)$. %
    
    $ \rH_{\mathrm{Poly}}(X \mid Y) + \rH_{\mathrm{Poly}}(Y) > \rH_{\mathrm{Poly}}(Y \mid X) + \rH_{\mathrm{Poly}}(X)+\omega(\log n). $ \\
    Proof: see Appendix~\ref{sec:inf_fac}.
\end{theorem}
As a corollary, we show no polynomial time probability model which can fit a one way function's forward direction can satisfy Bayes theorem (see \autoref{cor:inf_fac2}). Adding to these theoretical results, we look empirically at the gap in time-bounded entropy for one way functions, and the gap in both entropy and epiplexity over two orderings of predicting chess data.
In \autoref{fig:soi}(a), we choose $f$ to be given by the $8$ steps of evolution of the ECA rule 30 with state size $n$ and periodic boundary conditions~\citep{wolfram2003new}. Though distinct from the one way functions used in cryptography, rule 30 is believed to be one way~\citep{wolfram2003new} and unlike typical one way functions, the forward pass of rule 30 can be modeled by an autoregressive transformer, which we demonstrate by constructing an explicit RASP-L \citep{zhou2023algorithms, weiss2021thinking} program in Appendix~\ref{app:rule30_rasp}. As shown in \autoref{fig:soi}(a), the model achieves the Shannon entropy (gray) in the forward direction, but has a consistent gap in the reverse direction.

\begin{figure*}[t!]
\centering
\begin{minipage}[b]{0.34\linewidth}
    \centering
    \includegraphics[width=\linewidth]{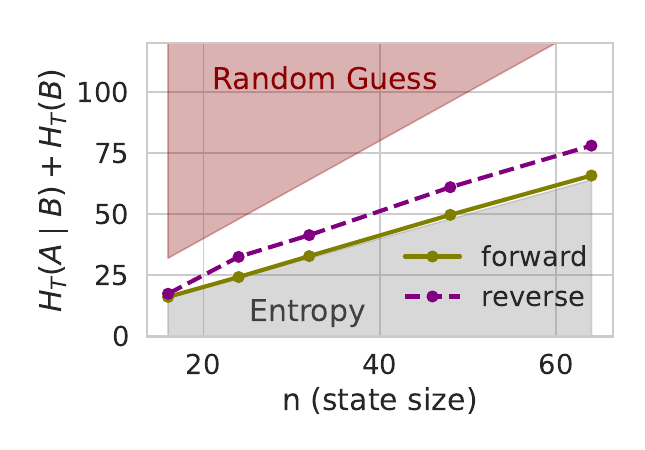}
    \subcaption{One way functions}
\end{minipage}%
\hfill
\begin{minipage}[b]{0.3\linewidth}
    \centering
    \includegraphics[width=\linewidth]{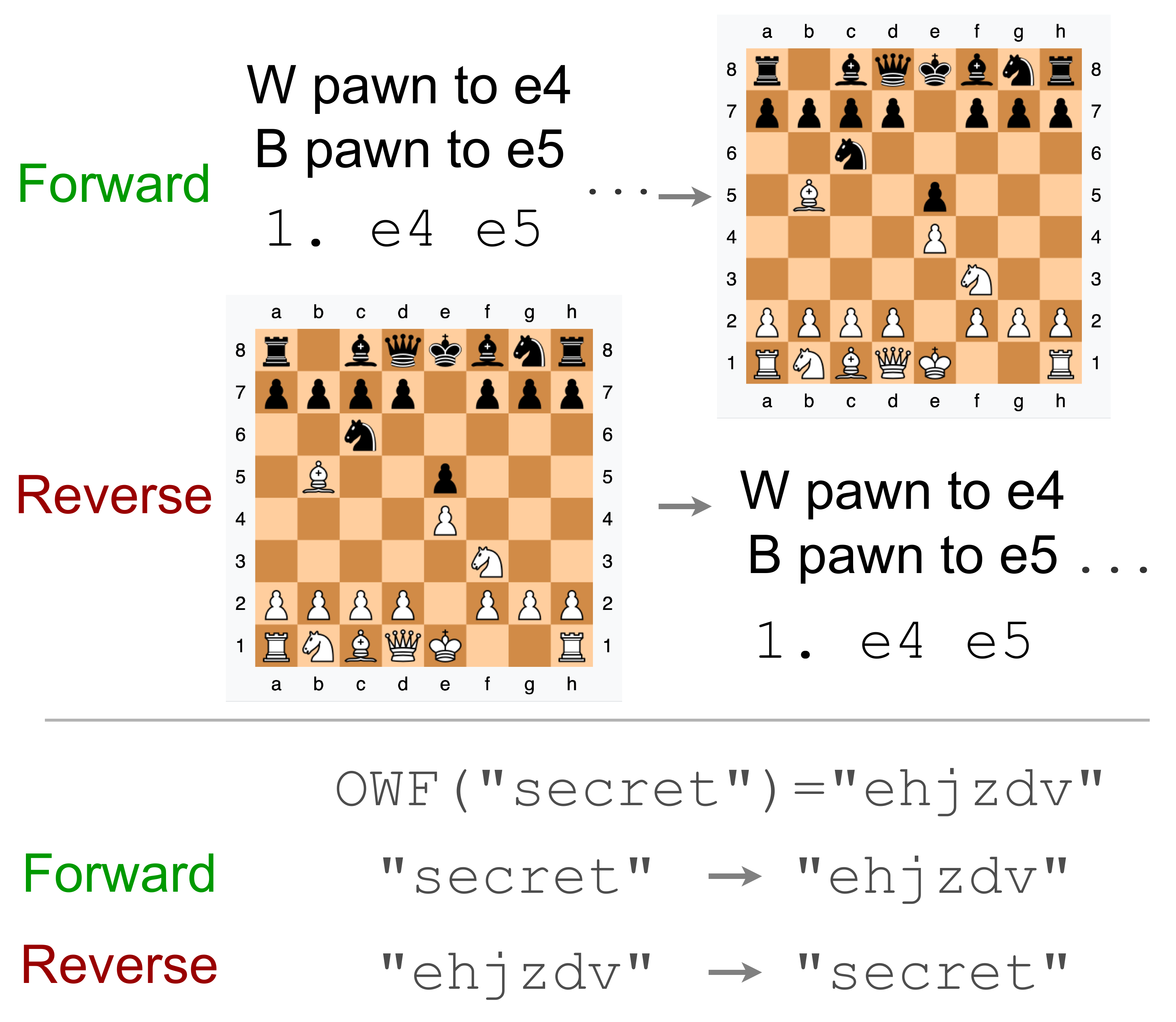}
    \subcaption{Factorization order}
    \label{fig:board-order}
    
\end{minipage}%
\hfill
\begin{minipage}[b]{0.34\linewidth}
    \centering
    \includegraphics[width=\linewidth]{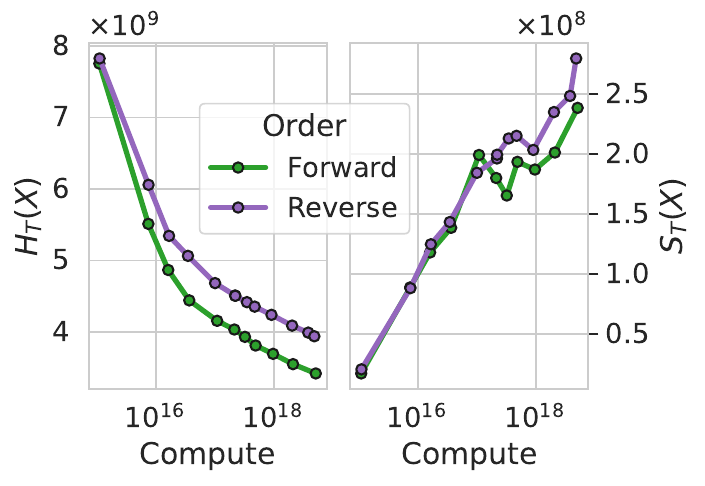}
    \subcaption{Chess orderings}
    \label{fig:chess-order-info}
    
\end{minipage}%

\caption{
\small \textbf{Factorization matters.}
(\textbf{a}) We compare the losses from modeling a conjectured one way function in forward and reverse as the state size $n$ is increased. The model reaches Shannon entropy in the forward direction, but with a persistent gap in the reverse direction. %
(\textbf{b}) The two orderings produce different outcomes. Analogous to the OWF, predicting the moves followed by the final board state is the direction that can be predicted with a straightfoward computation. Predicting the board first and then the moves requires more complex behaviors.
(\textbf{c}) As compute increases, the same chess data presented in the reverse order leads to higher time-bounded entropy and epiplexity, showing it becomes more difficult to predict but allows more structure to be learned.
}
\label{fig:soi}
\end{figure*}

Beyond just how the random information can vary with orderings, the structural information can also differ as we will show next.
We demonstrate this fact by training autoregressive transformer models on the Lichess dataset, a large collection of chess games where the moves are recorded in algebraic chess notation. We consider two variants of this dataset: (1) formatting each game as the move sequence followed by final board state in FEN notation, and (2) formatting each game as the final board state followed by the move sequence, as illustrated in \Cref{fig:board-order}. We provide full experiment details in \Cref{app:chess}. While there is no clear polynomial vs non-polynomial time separation in this setup, the first ordering is analogous to the forward direction as the final board state can be straightforwardly mapped from the moves with a simple function, while the latter ordering is analogous to the reverse direction, where recovering the moves from the final board state requires the inverse function that infers the intermediate moves from the final state. We hypothesize the reverse direction is a more complex task and will lead the model to acquire more structural information, such as a deeper understanding of the board state. \Cref{fig:chess-order-info} confirms this hypothesis, showing that the reverse order has both time-bounded higher entropy and epiplexity. This gap vanishes at small compute budgets where the model likely learns only surface statistics common to both orderings before the additional complexity of the reverse task forces it to develop richer board-state representations.

\subsection{Paradox 3: Likelihood Modeling is Merely Distribution Matching}\label{sec:likelihood}
There is a prevailing view that from a particular training distribution, we can at best hope to match the data generating process. If there is a property or function that is not present in the data-generating process, then we should not expect to learn it in our models. As an extension, if the generating process is simple, then so are models that attempt to match it. This viewpoint can be supported by considering the likelihood maximization process abstractly,
$
\argmin_P\mathbb{E}_{X\sim Q}[-\log P(X)] = Q;
$
the test NLL is minimized when the two distributions match. The extent to which the distributions differ is regarded as a failure either from too limited a function class or insufficient data for generalization. From these arguments we could reasonably believe that AI models cannot surpass human intelligence when pretraining on human data. Here we provide two classes of phenomena that seem to contradict this viewpoint: induction, and emergence. In both cases, restricting the compute available to AI models leads them to extract more structural information than what is required for implementing the generating process itself.

\subsubsection{Induction}\label{sec:induction}
\begin{figure*}[t!]
\centering
\begin{minipage}[b]{0.3\linewidth}
    \centering
    \includegraphics[width=\linewidth]{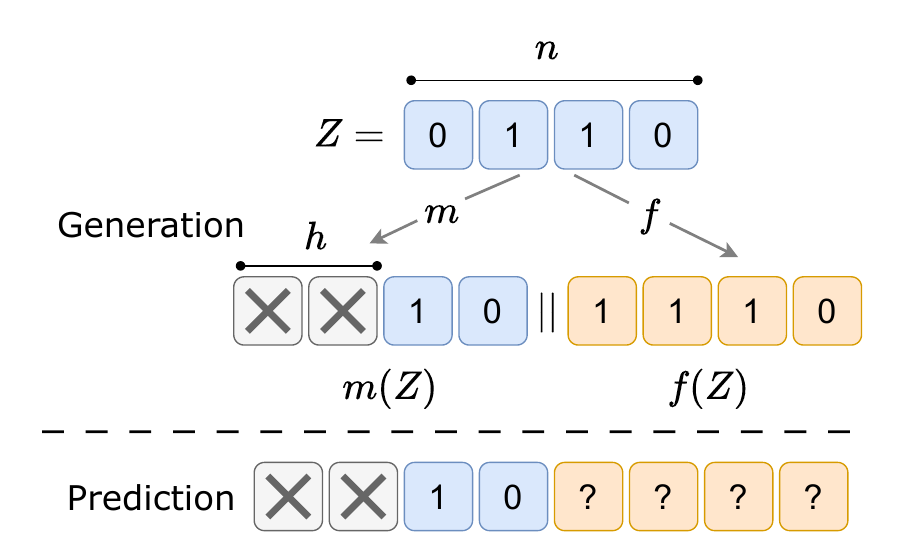}
    \subcaption{Data generating process}
\end{minipage}%
\hspace{2mm}
\begin{minipage}[b]{0.3\linewidth}
    \centering
    \includegraphics[width=\linewidth]{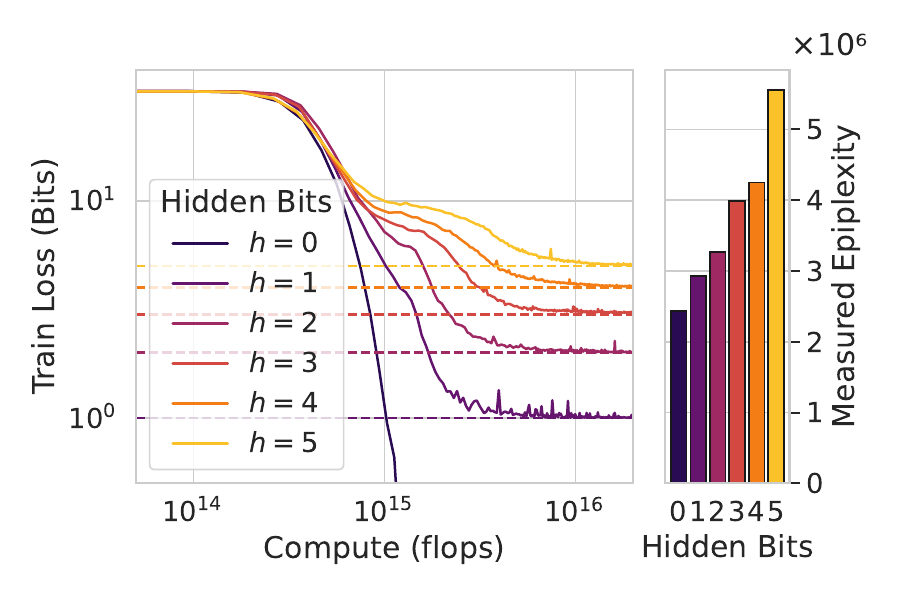}
    \subcaption{Induction (hard)}
    \label{fig:induction_hard}
    
\end{minipage}%
\hspace{2mm}
\begin{minipage}[b]{0.3\linewidth}
    \centering
    \includegraphics[width=\linewidth]{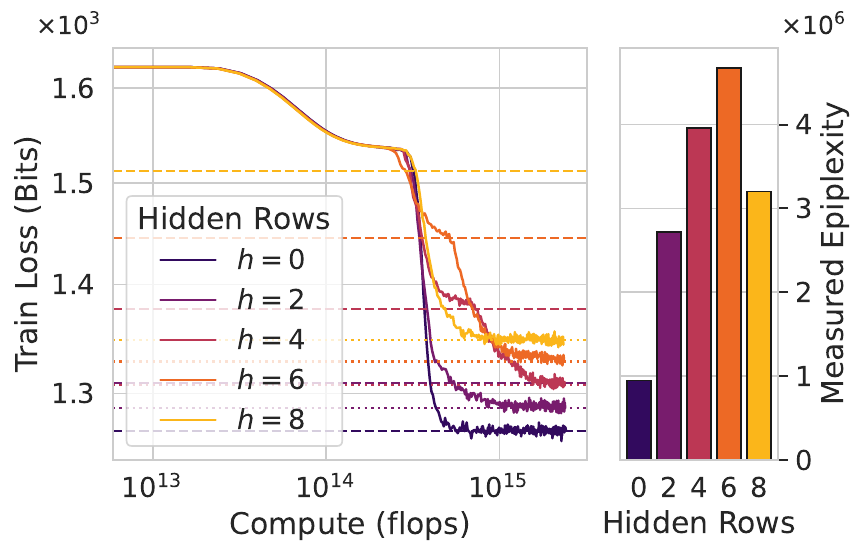}
    \subcaption{Induction (easy)}
    \label{fig:induction_easy}
    
\end{minipage}%

\caption{
\small
\textbf{Studying induction through epiplexity.} (a) Our setup for creating induction problems.
(b) Predicting Rule 30 ECA with hidden inputs. The LLM must induct on the $h$ bits missing from the input, paying a cost exponential in $h$. For $h$ small enough but $>0$, epiplexity is increased.
(c) Predicting Markov chain samples with hidden transition probabilities. Models that need to both use the provided probabilities and induct on the missing ones acquire the most epiplexity.
}
\label{fig:induction}
\end{figure*}

The generative modeling community is often challenged with simultaneously wanting a tractable sampling process and tractable likelihood evaluation, with autoregressors, diffusion models, VAEs, GANs, and normalizing flows each providing different approaches. For natural generative processes, it is often the case that one direction may be much more straightforward than the other. Here we investigate generative processes which can be constructed by transforming latent variables such that computing likelihoods requires inducting on the values of those latents.

A window into the phenomenon can be appreciated through this quote from Ilya Sutskever:
\begin{quote}
    ``\emph{You're reading a murder mystery and at some point the text reveals the identity of the criminal. ... If the model can predict [the name] then it must have figured out [who perpetrated the murder from the evidence provided].}'' \citep{sutskever2019gpt2}
\end{quote}
The author of the book on the other hand, need not have made that same induction. Instead, they may have chosen the murderer first and then painted a compelling story of their actions. 
This example highlights a gap between the generating process and the requirements of a predictive model, a gap which we explore with the following more mathematical setup.

As we illustrate in \autoref{fig:induction}(a), consider a simple to model random variable $Z$ over $\{0,1\}^n$ which we transform with two functions $m$ and $f$, which are both short in length and efficient to compute, and produce the data $Y=(m(Z),f(Z))$. We choose $m: \{0,1\}^n \to \{0,1\}^{n-h}$ as a masking function which removes the bits at a total of $h$ fixed locations in the input, leaving the rest unchanged. The generating process is simple to implement and can be executed efficiently. Now consider a likelihood generative model learning to model $Y$, under any given factorization. With appropriate properties of the function $f$, in producing the likelihoods the model must learn to induct on the missing information in the state $Z$, and then apply the transformation given by the data generating process. We consider cases both where the function $f$ is hard to invert and those where $f$ is not especially hard to invert. In both cases, predictive circuits must be learned that were not present in the data generating process, but with hard $f$ these circuits only appear at exponentially high compute. 

\textbf{Induction Hard: Rule 30 ECA.} For the first setting we use uniform $Z=U_n$ and $f$ as $4$ steps of the rule 30 ECA on state size $n=32$, $m$ simply removes the first $h$ bits, and we also compute the loss only on $f(Z)$ (conditioned on $m(Z)$) as the bits in $m(Z)$ are uniform and only add noise. We use an LLM, and the loss curves and measured epiplexities are shown in \Cref{fig:induction_hard}. The loss converges to the number of hidden bits $-\log P(f(Z)\mid m(Z)) = h$, representing the $2^h$ possible inductions on the hidden state. However, the total compute required for this loss to converge grows exponentially with $h$, an overall behavior consistent with a strategy of passing all $2^h$ candidates through $f$ and then eliminating inconsistent candidates as values of $f(Z)_i$ are observed with the autoregressive factorization. This complex learned function stands in contrast with the mere $f(Z)$ and simple postprocessing removing bits with masking. This picture is mirrored by the measured epiplexity: as the model is forced to induct on the missing bits, the epiplexity grows.

\textbf{Induction Easy: Random Markov Chains.} In the second setting, we leverage the statistical induction heads setup \citep{edelman2024evolution} with a few modifications. %
$Z$ is given by a random Markov chain transition matrix with $V=8$ symbols, and %
$m$ removes $h$ columns of the matrix at fixed random locations. The function $f(Z)$ computes a sampled sequence from the Markov chain of length $n=512$. When $h > 0$, the optimal solution involves 1) using the provided rows $Z$ to perfectly predict next-token probabilities on $V-h$ of the symbols, and 2) inducting on the missing rows of $Z$ in-context based on the empirically observed transitions to improve remaining predictions. For $h=0,$ the first is sufficient, and for $h=8$ the second is sufficient. In \Cref{fig:induction_easy}, we find evidence that both strategies are employed whenever $0<h<8$ as the final loss achieved matches the theoretical loss of both (the lower of the two dotted lines). The higher horizontal line marks the loss achievable using 1) along with a simple unigram strategy \citep{edelman2024evolution}, showing that the transformer learns 1) first and later the induction strategy 2). While the data generating program only only involves strategy one followed by the postprocessing masking step, the model must learn both strategies to reach these values.
Measured epiplexity matches this picture, with values $0<h<8$ having higher epiplexity than $h=0$ or $h=8$.  We emphasize that the induction strategy was never present in the data-generating process, yet it is learned by a generative model trained on that same data distribution. In \Cref{app:induction-not-specific}, we argue the induction phenomena are not specific to autoregressive models, but occur more generally for models trained via Maximum Likelihood Estimation as they need to be able to evaluate the likelihood $P(x)$ for an arbitrary data point $x$ rather than merely sample random $x$ from $P.$ VAEs \citep{kingma2013auto} provide a clear example of explicitly performing induction in non-autoregressive models: the encoder is trained specifically to approximate the posterior $P_{Z|X}$, enabling tractable likelihood estimation, yet this encoder is entirely unnecessary if the goal is merely to sample from the model.

In both of the hard and easy induction examples, the size of the program needed to perform the induction strategy is greater than the size of the program needed generate the data. We can expect that with limited computational constraints, it will not be generically possible to invert the generation process using brute force, and thus, in cases where alternative inverse strategies exist (like the easy induction example with the statistical induction heads), those additional strategies increase the epiplexity. Given that there is likely no single generally applicable strategy for these computationally efficient inverses across problems, it is likely to be possible as a source of epiplexity.

To make these statements more precise, it seems likely that there are \emph{no} constants $c_1$ and $c_2$ for which the following property holds:
\begin{remarkbox}{}
\label{def:induction_def}
\textbf{Limited Epiplexity Increase Property:} Given any program $\mathrm{G}:\{0,1\}^k\to \{0,1\}^n$ running in time at most $T_1$ on random variable $Z$, the epiplexity of $G(Z)$ is increased by at most a constant more than the size of $G$:
$\rS_{T_2}(G(U_k)) \le |\mathrm{G}|+c_1$ for $T_2(n)>T_1(k)+c_2$.

\end{remarkbox}
In other words, there is no bound on how much larger the MDL optimal probability model will be than the generating program even when the model is allowed more compute than the generating program. We present this phenomenon in contrast to Shannon information or Kolmogorov complexity, where a function and its inverse can differ in complexity by at most a fixed constant: $K(F^{-1}) = K(F)+O(1)$. When the computational constraints are lifted, the brute force inverse is possible, and there is no essential gap between deduction and induction, or between sampling and likelihood computation.

\subsubsection{Emergent Phenomena}
\label{sec:emergent}

One of the most striking counterexamples to the ``distribution matching'' viewpoint is \emph{emergence}. Even when a system’s underlying dynamics admit a simple description, an observer with limited computation may need to learn a richer, and seemingly unrelated, set of concepts to predict or explain its behavior. As articulated by \citet{anderson1972more}, reductionism---that a complex object’s behavior follows from its parts---does not guarantee that knowing those parts lets us predict the whole. Across biology and physics, many‐body interactions give rise to behaviors (e.g.\ bird flocking, Conway’s Game of Life patterns, molecular chemistry, superconductivity) that are not apparent from the microscopic laws alone. Here we sketch how emergence critically relates to the computational constraints of the observer, demonstrating how observers predicting future states may be required to learn \emph{more} than their unbounded counterparts who can execute the full generating process. %

Consider Type‐Ib emergence in the \citet{carroll2024emergence} classification, in which higher‐level patterns arise from local rules yet resist prediction from those rules. A canonical example is Conway’s Game of Life (see \autoref{app:conway} for definition), where iterating a simple computational rule $\Phi$ %
on a $2$D grid leads to complex emergent behavior. For observers that lack the computational resources to directly compute the iterated evolution $\Phi^k$, an alternate description must be found. In the state evolution, one can identify localized “species” (static blocks, oscillators, gliders, guns) which propagate through space and time. %
By classifying these species, learning their velocities, and how they are altered under collisions with other species, as well as the ability to identify their presence in the initial state, computationally more limited observers can make predictions about the future state of the system. Doing so, however, requires a more complex program in the sense of description length, and the epiplexity will be higher. We can formalize this intuition into the following definition of emergence.

\begin{remarkbox}{}
\begin{definition}[Epiplexity Emergent]
\label{def:epiemergent}
Let $\{\Phi_n\}_{n\ge1}$ be a computable family $\Phi_n:\{0,1\}^n\to\{0,1\}^n$
and let $\{X_n\}_{n\ge1}$ be random variables over $\{0,1\}^n$.
We say $(\Phi,X)$ is \emph{epiplexity-emergent} if there exist
time bounds $T_1,T_2$ with $T_1(n)=o(T_2(n))$ and an iteration schedule $k(n)$ such that as $n\to\infty$,

\begin{align}
    \rS_{T_1}(\Phi(X) \mid X,n) - \rS_{T_2}(\Phi(X) \mid X,n)&=\Theta(1)\,,\\
    \rS_{T_1}(\Phi^k(X) \mid X,n,k) -\rS_{T_2}(\Phi^k(X) \mid X,n,k) &=\omega(1), \notag
\end{align}

where we have suppressed the dependence of $X_n$ and $\Phi_n$ on $n$ for clarity.
\end{definition}
\end{remarkbox}
In words, $\Phi,X$ displays emergent phenomena if two observers see equivalent structural complexity in the one step map, but asymptotically more structural complexity in the multistep map for the observer with fewer computational resources.

Considering $\Phi$ from the Game of Life as an example, $P(\Phi(X)\mid X,n)$ could be well estimated by both $T_1$ and $T_2$-bounded observers using the exact time evolution rule, using constant bits for both. $P(\Phi^k(X) \mid X,n,k)$ could be estimated by $T_2$ using the iterated rule, but not by $T_1$. Using knowledge of the different pattern species improves predictions of $\Phi^k(X) \mid X$, so they would need to be learned; however, the number of patterns that needs to be considered in the time-bounded optimal solution is unbounded, and grows with the size of the board $n$, and thus the gap in epiplexity for the two time bounds grows with $n$. We have not proven that the Game of Life satisfies this definition, which is likely difficult as small changes to the evolution rule can destroy the emergent behavior; however, we provide empirical evidence for this set being non-empty with the example below.

\begin{wrapfigure}[14]{r}{0.35\linewidth}
  \centering
  \vspace{-5mm}
  \includegraphics[width=\linewidth]{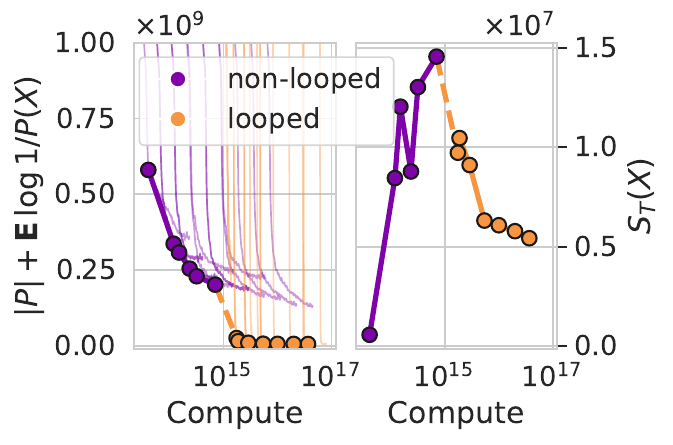}
  \vspace{-6mm}
  \caption{\small \textbf{Emergence in ECA.} Compute-constrained models extract high epiplexity from data generated by simple rules, trading increased program length for reduced computation.}
  \label{fig:eca_emergence}
\end{wrapfigure}
In \Cref{fig:eca_emergence}, we empirically demonstrate the emergence phenomenon by training a transformer to predict the iterated dynamics of ECA rule 54, a class IV rule that produces complex patterns. As in Conway's Game of Life, a model with sufficient computation can exactly simulate the dynamics by directly iterating the per-step rule---a brute-force solution with a short description length. However, a compute-limited model cannot afford this approach and must instead learn emergent patterns (e.g., gliders and their collision rules) that approximately shortcut the infeasible exact simulation. The brute-force solution can be naturally implemented by learning to autoregressively unroll intermediate ECA states rather than directly predicting the final state, resembling the use of chain-of-thought \citep{wei2022chain} or looped transformers \citep{dehghani2018universal,giannou2023looped,saunshi2025reasoning}. We provide experiment details in \Cref{app:eca-emergence}. While initially the non-looped model (directly predicting final state) gradually achieves better MDL and higher epiplexity as compute increases, we identify a compute threshold beyond which the looped model suddenly becomes favorable, causing an abrupt drop in MDL and epiplexity, likely by learning the simple, brute-force solution. Below this threshold, the looped model underperforms likely because it lacks the compute to fully unroll the dynamics. The non-looped model, unable to rely on brute-force simulation, must instead learn increasingly sophisticated emergent rules, recognizing more species and their interactions, thus causing epiplexity to initially rise with compute before eventually falling.

While this experiment cleanly demonstrates how compute-limited models can learn richer structure from data, it is a more uncommon situation where the brute-force solution is accessible, and where training with more compute reveals a much simpler underlying structure. With natural data and compute bounds that are not extraordinarily high, we expect that expending additional compute leads to increased rather than decreased observed structure. %

We explore other kinds of emergence, such as in chaotic dynamical systems or in the optimal strategies of game playing agents in \autoref{app:emergence}. Each of these examples presents clear evidence that in pursuit of the best probability distribution to explain the data, observers with limited compute will require models with greater description length than the minimal data generating process in order to achieve comparable predictive performance \citep{martinez2006phenomenology,redeker2010language}. Epiplexity provides a general tool for understanding and quantifying these phenomena of emergence, and how simple rules can create meaningful, complex structures that AI models can learn from, as recently demonstrated empirically by \citet{zhang2024intelligence}.

\section{Epiplexity, Pre-Training, and OOD Generalization}
\label{sec:ood}

Pre-training on internet-scale data has led to remarkable OOD generalization, yet a thorough understanding of this phenomenon remains elusive. What kinds of data provide the best signal for enabling broad generalization? Why does pre-training on text yield capabilities that transfer across domains while image data does not? As high-quality internet data becomes exhausted, what metric should guide the selection or synthesis of new pre-training data?  In this section, we show how epiplexity helps answer these foundational questions.

OOD generalization is fundamentally about how much reusable structure the model acquires, not how well it predicts in-distribution. Two models trained on different corpora can achieve the same in-distribution loss, yet differ dramatically in their ability to transfer to OOD tasks. This happens because loss captures only the residual unpredictability, corresponding to the time-bounded entropy, not how much reusable structure the model has internalized to achieve that loss. Epiplexity measures exactly this missing component: the amount of information in the learned program. Intuitively, loss indicates how random the data looks to the model, while epiplexity indicates how much structure the model must acquire to explain away the non-random part. If OOD generalization depends on reusing learned mechanisms rather than memorizing superficial statistics, then epiplexity is a natural lens through which to understand the relationship between pre-training data and OOD transfer.
As a motivating toy example, \citet{zhang2024intelligence} observed that downstream task performance benefits most from training on type IV ECA rules over the other ECA rules, aligned with Figure~\ref{fig:eca} where we showed that rule 54 (a type IV rule) induces much higher epiplexity compared to other rules.

\vspace{-1mm}
\subsection{Epiplexity Correlates with OOD Generalization in Chess}
\vspace{-1mm}
\begin{wrapfigure}[14]{r}{0.3\linewidth}
  \centering
  \vspace{-5mm}
  \includegraphics[width=\linewidth]{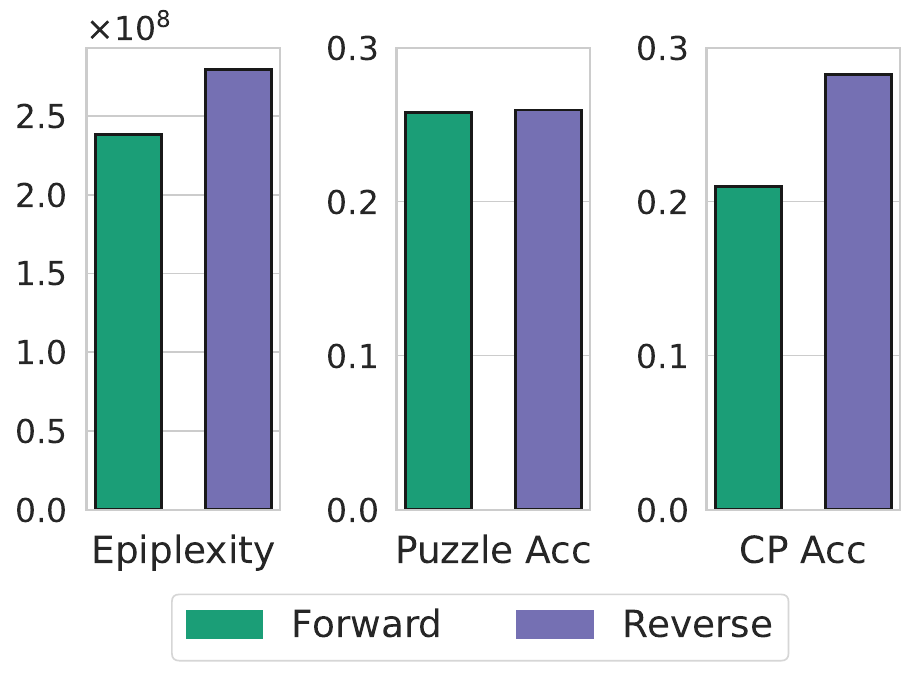}
  \vspace{-6mm}
  \caption{\small \textbf{Epiplexity and OOD performance in chess.} Models trained on the higher epiplexity reverse order performs better in OOD tasks.}
  \label{fig:chess-ood}
\end{wrapfigure}

We finetune models trained on either ordering from \Cref{sec:factorization} on two downstream tasks: (1) solving chess puzzles, where the model must predict the \emph{optimal} next move given a board state \citep{burns2023weak}, and (2) predicting centipawn evaluation, where the model evaluates positional advantage from FEN notation—a more substantial distribution shift from next-move prediction learned in pre-training. Experiment details are in \Cref{app:chess}.
As shown in \autoref{fig:chess-ood}, the reverse (board-then-moves) ordering yields higher epiplexity and better downstream performance: matching accuracy on chess puzzles but significantly higher accuracy on the centipawn task. This result supports our hypothesis: the reverse order forces the model to develop richer board-state representations needed to infer the intermediate moves, and these representations transfer to OOD tasks like centipawn evaluation that similarly require understanding the board state. This example reflects a more general principle: epiplexity measures the learnable structural information a model extracts from data to its weights, which is precisely the source of the information transferable to novel tasks, making epiplexity a plausible indicator for the potential of OOD generalization. However, we emphasize that higher epiplexity does not guarantee better generalization to any specific task: epiplexity measures the amount of structural information, irrespective of its content. A model trained on high epiplexity data can learn a lot of structures, but these structures may or may not be relevant to the particular downstream task of interest.

\begin{figure*}[t]
\centering
\begin{minipage}[t]{0.28\linewidth}
    \centering
    \includegraphics[width=\linewidth]{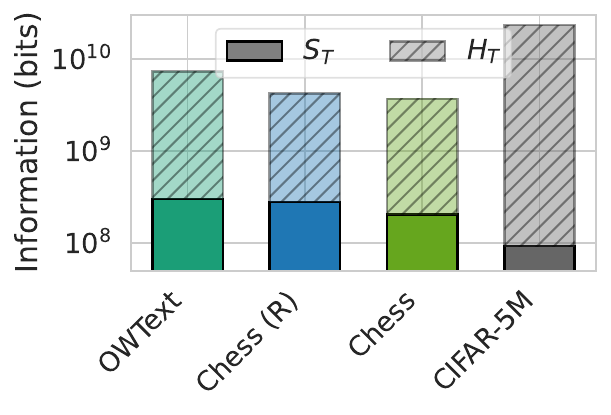}
    \subcaption{Epiplexity in natural data}
    \label{fig:natural-stacked}
\end{minipage}%
\hfill
\begin{minipage}[t]{0.29\linewidth}
    \centering
    \includegraphics[width=\linewidth]{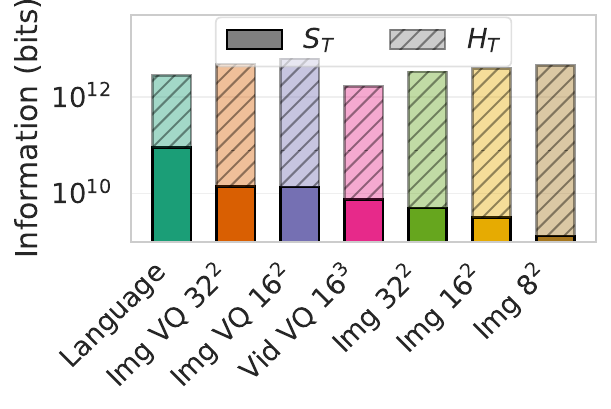}
    \subcaption{Estimation via scaling laws}
    \label{fig:scaling-law-stacked}
\end{minipage}%
\hfill
\begin{minipage}[t]{0.4\linewidth}
    \centering
    \includegraphics[width=\linewidth]{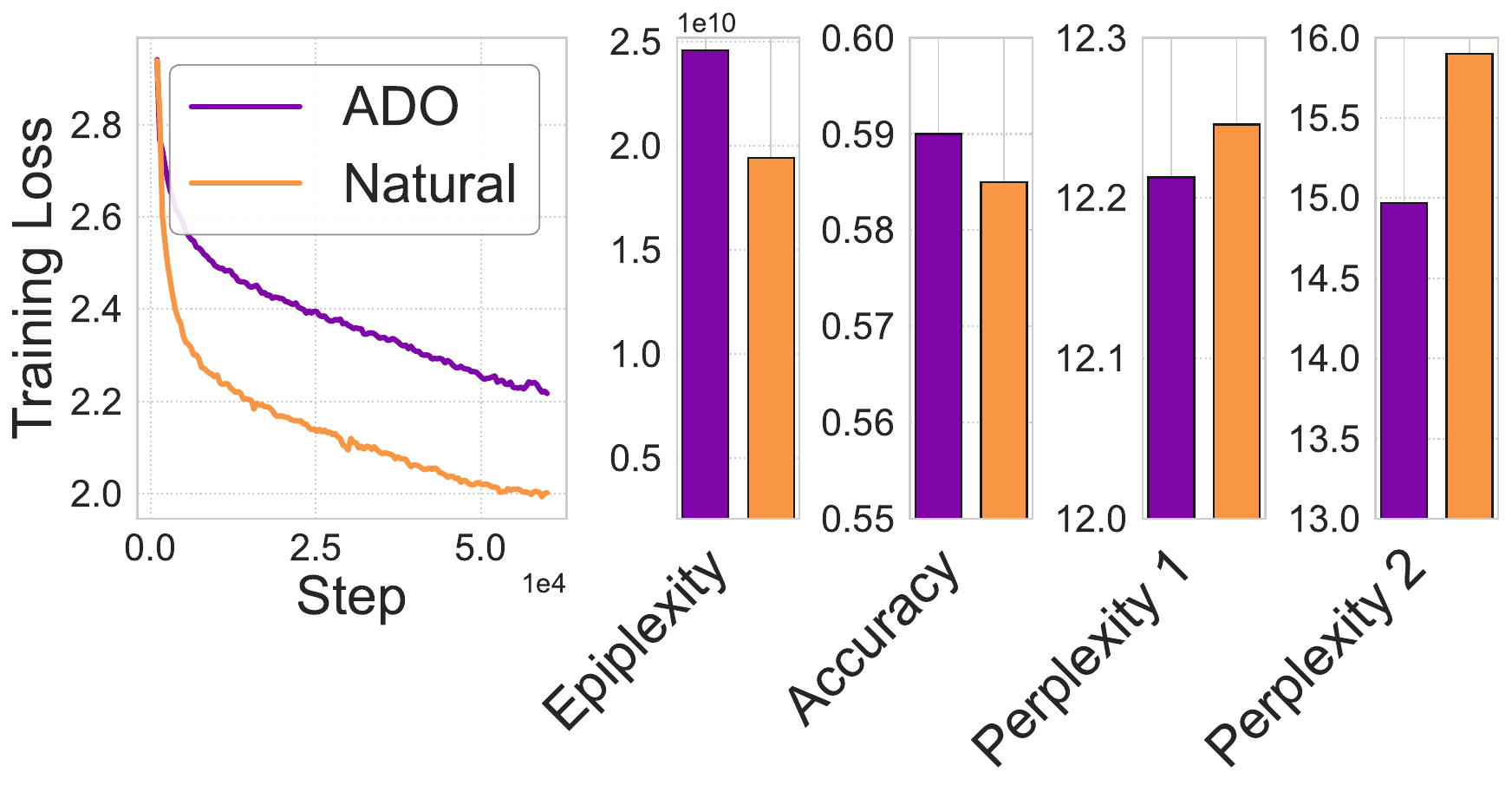}
    \subcaption{ADO: epiplexity and downstream metrics}
    \label{fig:ado}
\end{minipage}%

\caption{
\small
\textbf{Epiplexity reveals differences in the structural information across data modalities and can guide pre-training data selection.}
(\textbf{a}) Estimated epiplexity and time-bounded entropy using requential coding for 1B OpenWebText, Chess, and CIFAR-5M tokens at $6\times10^{18}$ FLOPs. 
(\textbf{b}) Estimated values based on scaling laws and prequential coding for 1T language, image, and video tokens at $10^{25}$ FLOPs.
(\textbf{c}) Selecting pre-training data using ADO \citep{jiang2025adaptive} leads to different loss curves than standard sampling (natural). Our measurement shows ADO selects data with higher epiplexity, in line with the improved downstream performance and OOD perplexity on different text corpora. 
}
\end{figure*}

\vspace{-1mm}
\subsection{Measuring Structural Information in Natural Data} \label{sec:epiplexity-natural}
\vspace{-1mm}
Among different modalities of natural data, language has proven uniquely fruitful for pre-training, not only for improving in-distribution performance such as language understanding \citep{radford2019language}, but also for out-of-distribution tasks such as robotics control \citep{ahn2022can}, formal theorem proving \citep{song2024lean}, and time-series forecasting \citep{gruver2023large}. While equally abundant total information is available in other modalities, such as images and videos, pre-training on those data sources typically does not confer a similarly broad increase in capabilities. We now show that epiplexity helps explain this asymmetry by revealing differences in their structural information content. In \Cref{fig:natural-stacked}, we show the estimated decomposition of the information in 5B tokens of data from OpenWebText, Lichess, and CIFAR-5M \citep{nakkiran2020deep} into epiplexity (structural) and time-bounded entropy (random) with a time-bound of $6\times10^{18}$ FLOPs, by training models of up to 160M parameters on at most 5B tokens using requential coding. In all cases, epiplexity accounts for only a tiny fraction of the total information, with the OpenWebText carrying the most epiplexity, followed by chess data. Despite having the most total information, CIFAR-5M data has the least epiplexity, as over $99\%$ of its information is random (e.g., unpredictability of the exact pixels).

\begin{wrapfigure}[15]{r}{0.4\linewidth}
  \centering
  \vspace{-5mm}
  \includegraphics[width=\linewidth]{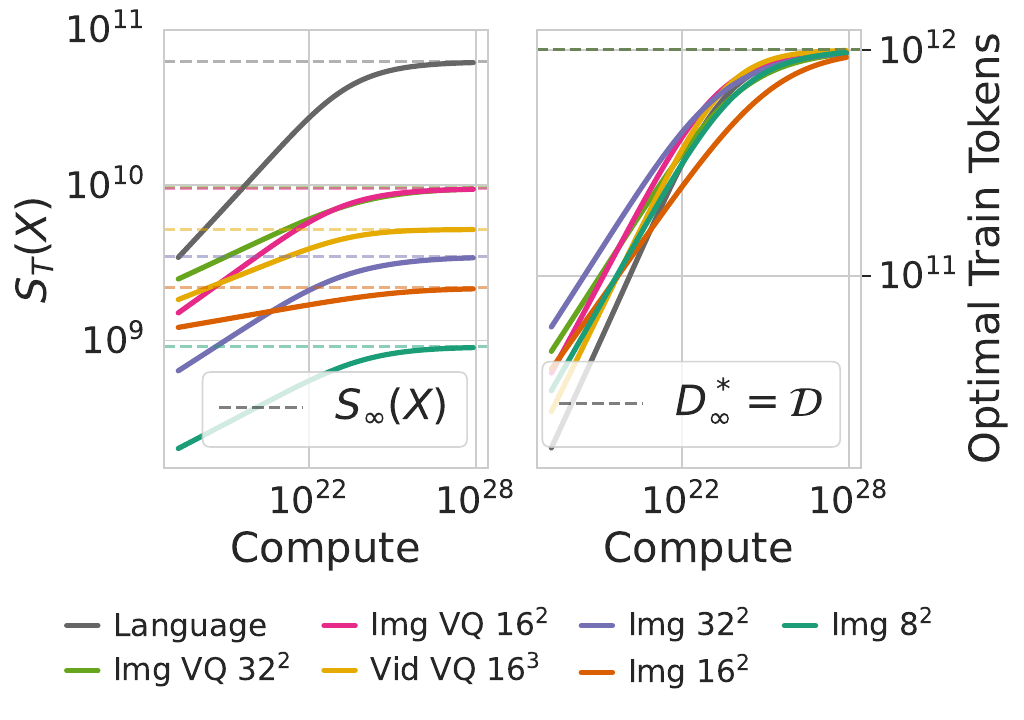}
  \vspace{-6mm}
  \caption{\small \textbf{Epiplexity and optimal training tokens for each fixed dataset converge to predictable limits as compute increases.}}
  \label{fig:asymptotic}
\end{wrapfigure}
\subsection{Estimating Epiplexity from Scaling Laws}
We can estimate the epiplexities of larger datasets at higher compute budgets using reported scaling laws, 
which describe the loss achieved by an $N$-parameter model trained on $D$ tokens as $\L(N,D) = E + \qty(N/N_0)^{-\alpha} + \qty(D/D_0)^{-\beta}$, for some dataset-specific constants $\alpha,\beta,N_0, D_0, E$ \citep{hoffmann2022training,kaplan2020scaling,henighan2020scaling}. By estimating the model's description length via the prequential coding approach (\Cref{sec:preq-proxy}), we obtain estimates for the epiplexity and time-bounded entropy for language, image, and video datasets, with varying resolutions and tokenizations of size $\D=10^{12}$ (1T) tokens under a compute budget of $10^{25}$ FLOPs (equivalent to the training compute of Llama3 70B), illustrated in \Cref{fig:scaling-law-stacked} (see details in \Cref{app:scaling-law-exps}). Consistent with our smaller-scale experiments, we find that language data has the highest epiplexity, while image data has the least. For image data, applying VQ tokenization leads to a significant increase in epiplexity, likely as a result of allowing the model to focus on higher-level semantic structures. Video data has less time-bounded entropy and epiplexity than image data with the same resolution, likely due to significant redundancy across the temporal dimension.

Using this approach, we can also gain some analytical insights about epiplexity for data admitting scaling laws of this form. As we derive in \Cref{app:scaling-law-model}, for a fixed dataset $X$ with $\D$ tokens, the optimal split of the compute budget between training and inference (evaluating the trained model on $X$) approaches a fixed ratio as compute increases, with the optimal asymptotic training tokens $D^\star_\infty = \D$ and asymptotic epiplexity $\rS_\infty(X) = \frac{\beta}{1-\beta} D_0^\beta \D^{1-\beta},$ both illustrated in \Cref{fig:asymptotic}. As expected, the maximum amount of extractable structural information is ultimately capped by the dataset size $\D$ when compute is not the bottleneck, and epiplexity can increase further if we also grow the dataset size. For large $\D,$ the scale of the asymptotic epiplexity is primarily determined by $\beta$ and $D_0,$ with smaller $\beta$ and larger $D_0$ leading to higher epiplexity, corresponding to slower improvement in loss and thus more (estimated) information absorbed per token. In line with our discussion on emergence in \Cref{sec:emergent}, it is possible that with significantly more compute much simpler programs can model these natural datasets, such as by directly simulating the basic laws of physics from which the natural world emerges, but the amount of required computation is likely so high that such programs remain inaccessible to any physically realizable observer and we must treat natural data as having high epiplexity for all practical purposes.

\subsection{Pre-Training Data Selection and Curriculum for Language Models} \label{sec:ado}
A crucial step in pretraining a language model is designing the composition of the pretraining data, but there lack clear guidelines for this step.
Existing data mixtures are designed through extensive trial-and-error and rely on heuristic guidelines such as ``diversity'' or ``high-quality''.
More importantly, the primary way of comparing different training data is via perplexity metrics of held-out datasets and downstream performance.
These procedures are highly susceptible to data contamination, overfitting to a narrow set of downstream evaluations, and Goodhart's law. After all, no suite of downstream evaluations is extensive enough to faithfully capture the range of tasks that a general-purpose language model will encounter in the real world.

As we argued above, \soph{} measures the structural information learned by the model, which could be affected by data selection strategies. 
\citet{jiang2025adaptive} demonstrated that models of the loss curves for different data subsets can be used to dynamically adjust the data distribution online to favor data subsets whose training losses are \emph{decreasing faster}\footnote{It is worth noting that choosing data subsets with faster-decreasing loss does not mean that the observed training loss would be smaller because such data subsets tend to have higher loss values since there is more learnable information in them. Consequently, training on them often leads to a larger area under the training loss curve.}.
Intuitively, this objective aligns with increasing the prequential estimate of epiplexity described in \Cref{sec:preq} by maximizing information absorbed per token.
We hypothesize that the proposed algorithm, Adaptive Data Optimization (ADO), inadvertently achieves higher \soph{}.
Experiments of \citet{jiang2025adaptive} are conducted on decoder-only transformers with 1.3B parameters trained on 125B tokens from the Pile dataset~\citep{gao2020pile}.
The models are evaluated on a suite of 7 zero-shot downstream tasks and two OOD validation datasets, SlimPajama~\citep{cerebras2023slimpajama} and FineWeb~\citep{penedo2024fineweb}.

In Figure~\ref{fig:ado}(c), we show the estimated \soph{} and the downstream performance as well as perplexity on two OOD datasets, adapted from \citet{jiang2025adaptive}.
As shown in \citet{jiang2025adaptive}, ADO achieves higher downstream performance than a standard data sampling strategy that uniformly samples from the entire dataset (denoted by \emph{Natural} in Figure~\ref{fig:ado}), despite not being optimized for any of these metrics.
Interestingly, we see that ADO indeed achieves higher epiplexity measured by prequential coding. 
While these downstream evaluations do not capture everything about a pretrained model, they do offer evidence that epiplexity is a potentially useful concept for understanding the intrinsic value of pretraining data without particular downstream evaluations.

\section{Additional Related Work}
Epiplexity builds on a number of related ideas in algorithmic information theory and complexity science that attempt to theoretically characterize \emph{meaningful information}. 
A group of closely related concepts are sophistication (\autoref{sec:sophistication}), effective complexity, and logical depth. Similar to sophistication, effective complexity aims to separate random from structural content \citep{gell1996information}. From a different starting point, \citet{bennett1988logical} introduced logical depth, measuring the number of time steps required by a nearly optimal program to produce a given string, and which was later shown to be equivalent to sophistication through the busy beaver function \citep{antunes2006sophistication,ay2010effective}. Several other formal measures have been developed to quantify structured or meaningful complexity. Algorithmic statistics offers a principled decomposition of data into regular versus random components by introducing the notion of an algorithmic sufficient statistic \citep{vereshchagin2004kolmogorov}, a concept closely tied to sophistication. Relatedly, statistical complexity in computational mechanics \citep{shalizi2001computational} measures the entropy of causal states in an optimally predictive model, capturing structure in time-series data. As we argued above, these existing notions of complexity do not account for the limited computation available to the observer, which is essential for understanding machine learning algorithms.
Being oblivious to computational limits means that they cannot characterize CSPRNGs or encrypted objects as being random. One might think that these failures are surface-level; for example, a plausible strategy would be to upgrade sophistication by replacing Kolmogorov complexity with time-bounded Kolmogorov complexity in (Definition~\ref{def:nsoph}). However, this approach does not work for several reasons, the most obvious being that CSPRNG outputs do have short and efficiently runnable generating programs and thus their time-bounded Kolmogorov complexities are small.
A more subtle reason is that doing so results in trivial sophistication for all strings, which we discuss in more detail in Appendix~\ref{app:time_bounded_sophistication}.

Our work is also closely related to several lines of work trying to characterize observer-dependent notions of information. In cryptography, \citet{barak2003computational} and \citet{hsiao2007conditional} discuss several possible definitions for \emph{computational pseudoentropy}, an observer-dependent analogue of entropy.
HILL-pseudoentropy \citep{haastad1999pseudorandom} is defined relative to a class of tests: a source is considered random if no test within the class can distinguish it from a high-entropy distribution with nontrivial advantage, and Yao-pseudoentropy is defined via compressing and decompressing an object for example.
Both definitions are closely related to time-bounded entropy, which measures the random content to a given computationally bounded observer; however, our formulation directly maps on to machine learning practice and allows for separating out the structural information content, a key contribution of our work. %
More recently, \citet{xu2020theory} propose $\mathcal{V}$-entropy, a generalization of Shannon entropy to the minimum expected negative log probability over a given family of probability models, such as those with given computational constraints. With $\mathcal{V}$-entropy, the symmetry of information can be violated, and so too can the data processing inequality, though neither is explicitly proven in the paper. Unlike time-bounded entropy, the computational constraint in $\mathcal{V}$-entropy only limits the inference time, and does not account for the time to find such a model. 
Hence, the minimizer can be far away from the regime that is practically evaluated (such as models that are \emph{trained} on infinite data or with infinite compute). 
While these undesirable behaviors can be overcome by imposing further data constraints, we believe our formulation of imposing a single bound on both training and inference time leads to fewer complications. 
More importantly, both pseudoentropy and $\mathcal{V}$-entropy, much like time-bounded entropy, capture only the random component of information since it still measures the unpredictability of the random variable under the best feasible model.
For understanding what useful information a model has learned, we are more interested in the non-random component of information as measured by epiplexity. Using existing measures of complexity, such as the Lempel-Ziv complexity and Wolfram classification, \citet{zhang2024intelligence} showed that models trained on complex data like Class IV ECA rules tend to perform better on downstream tasks.

Other parts, such as the area under the curve estimate of epiplexity, have seen some related exploration in prior work. The concept of excess entropy, independently introduced under various names \citep{crutchfield1983symbolic, shaw1984dripping, grassberger1986toward} and reviewed in \citet{feldman1998information}, is defined as the area between finite-block entropy density estimates and the asymptotic entropy rate of a stationary process, an analogous construction to our prequential estimate of epiplexity. However, excess entropy is defined for stationary processes observed by computationally unbounded agents, lacking the explicit dependence on the observer's compute budget that we view as essential for the machine learning setting. More recently, \citet{whitney2020evaluating} introduced surplus description length (SDL), which is the summed online loss of the training algorithm, with either the entropy of the data or a fixed baseline performance subtracted out. The authors use this measurement to evaluate pre-trained representations for solving a downstream task, arguing that smaller SDL is preferred as they lead to more efficient downstream learning. In contrast, we seek to create datasets and interventions to the data which \emph{increase} epiplexity. More analogous to the spirit of epiplexity is information transfer from \citet{zhang2020measuring}, which sums a variant of a loss difference, adapted to held out test data and for the classification setting. In this work, the authors present information transfer to measure how much is learned from the data. Epiplexity is complementary to these works, clarifying the role of computation in defining information, and explicitly separating random and structural information.

Several works have also explored how to quantify data complexity. \citet{dziugaite2025size} suggests that the complexity of a minimal near-optimal reference model can be viewed as a measure of data complexity under the PAC-Bayes framework and how such data complexity gives rise to empirical scaling laws.
This perspective is related to epiplexity in that both associate data complexity with the size of compact models that explain the data well. 
However, the two notions differ in important ways. 
In particular, the PAC-Bayes formulation is concerned with the existence of some small reference model achieving good in-distribution performance, whereas epiplexity characterizes the amount of structural information extractable by a computationally bounded observer, formalized through a two-part code that explicitly accounts for the cost of obtaining such a model. 
Further, our primary interest is not in characterizing in-distribution generalization, but in using epiplexity to measure the intrinsic value of data in settings that extend beyond supervised learning.
Relatedly, \citet{hutter2021learning} shows that power-law learning curves can emerge under specific assumptions on the data-generating distribution, illustrating how properties of the data itself can shape empirical scaling behavior. While this line of work focuses on explaining observed learning dynamics rather than defining a complexity measure, it similarly emphasizes the role of data structure in determining learning outcomes. 
These perspectives on data complexity can be viewed as instances of \emph{coarse graining}, where one seeks a compressed representation that preserves some notion of ``relevant'' structure. A canonical example is the information bottleneck framework, which formalizes coarse graining as a trade-off between compression and retained information about a relevant variable~\citep{tishby2000information}. 
Epiplexity is aligned with this perspective, but rather than defining relevance through a task variable or through distinguishability to tests, it measures the amount of structural information extractable by a computationally bounded learner, while explicitly accounting for the cost of obtaining the model.

More broadly, our work is related to several lines of work on how resource constraints fundamentally alter the notion of simplicity and learnability. In algorithmic information theory, \citet{schmidhuber2002speed} proposes the speed prior, which replaces Solomonoff’s universal prior with a \emph{computable} semimeasure that favors both shorter program length and smaller computation time, thereby incorporating computational resources directly into the definition of simplicity.
\citet{achille2025ai} argue that in the transductive setting, the role of information from past data is to reduce the time needed to solve new tasks rather than to reduce uncertainty, with the optimal speedup tightly characterized by the amount of shared algorithmic information between past data and future tasks.
In this setting, \emph{larger} information content is shown to be more conducive to better performance.
In learning theory, a related line of work shows that computational limitations can directly affect what can be learned from data.
For instance, in the problem of sparse PCA detection, \citet{berthet2013computational} show that although there exist procedures that succeed with an information-theoretically minimal number of samples, any algorithm that runs in polynomial time necessarily requires more data under widely used average-case hardness assumptions.
Memory and space constraints alone can also qualitatively change learnability.
\citet{steinhardt2016memory} show that restricting a learner’s memory can dramatically increase the amount of data required to learn, even when the target concept itself has a very concise description.
They identify parity functions as a canonical example where this tension is conjectured to be sharp.
\citet{raz2018fast} later resolves this conjecture by proving that any learner with sub-quadratic memory requires exponentially many samples to learn parity from random examples.

\section{Discussion}

Much of classical information theory is concerned with the representation and transmission of information, and abstracts away key aspects of the computational processes by which information is extracted and used. While complexity theory and cryptography treat computation as fundamental, machine learning theory typically does not.
Yet learning, whether biological or artificial, is an inherently computational process. What can be learned from data depends not only on statistical feasibility, but on the available resources. This perspective calls for more theoretical tools that place computation on an equal footing with information.

This work reframes information as a property of data relative to a computationally bounded observer, and demonstrates that information can be decomposed into time-bounded entropy and epiplexity, a formalization of structural information. It also sheds light on how perceived information can be changed through computation. This perspective resolves several tensions between information theory and empirical machine learning---including the usefulness of synthetic data, the dependence of learning on factorization and ordering, and the emergence of structure beyond the data-generating process itself. Technically, epiplexity connects ideas from algorithmic statistics, cryptography, and learning theory, showing that standard assumptions (i.e., existence of one-way functions) suffice to produce distributions with high structural complexity for efficient learners.

Our framework opens several exciting directions for future work. On the theoretical side, it invites a systematic and more fine-grained understanding of how structural information changes with computational budget, model class, and data transformations, potentially yielding new lower bounds and impossibility results for representation learning and transfer. Taking information and computation as the fundamental resources may offer new explanations for the relative universality observed in large-scale training, including why scaling law exponents depend only weakly on architectural and optimizer details.
There is also a possibility of a compute-aware analogue of classical notions such as sufficient statistics and information bottlenecks. 
More broadly, framing emergence, induction, and generalization through the lens of computationally bounded observers may offer a unifying language across learning theory, algorithmic information theory, cryptography, and complexity theory.

On the empirical side, epiplexity provides a way to reason about why some data sources, formatting, and transformations can lead to more transferable models than others, even when they do not improve training loss.
The framework suggests that pretraining data should be evaluated not only by held-out perplexity, but by how much reusable structural information it induces in a computationally bounded model.
This perspective helps explain empirical successes of curriculum design, data ordering, augmentation strategies, and even synthetic data that appear counterintuitive from a purely statistical viewpoint. 
Our empirical estimator offers a concrete starting point for comparing datasets and interventions in data centric research.
In the long run, we believe epiplexity could provide guidance on how to generate new synthetic data from existing data.

Finally, representation learning can be understood as the gradual accumulation of epiplexity: the construction of increasingly rich internal programs that approximate a data distribution within a fixed time budget. While epiplexity in isolation is not a measure of generalization, or a complete theory of learning, this perspective raises the possibility of new notions of hardness for learning and transfer that are orthogonal to classical PAC-style measures, capturing not sample complexity but the size of the structure that must be extracted. 
Such notions may help explain why certain tasks appear to require disproportionately large models or long training horizons despite admitting simple generative descriptions, and why improvements in generalization sometimes correlate more strongly with training dynamics or data structure than with likelihood alone.

\vspace{5mm}
\textbf{Acknowledgements.} We thank NSF CAREER IIS-2145492, NSF CDS\&E-MSS 2134216, and DARPA AIQ
HR00112590066 for support, and Scott Aaronson, Alan Amin, Brandon Amos, Martin Marek, Zhili Feng, Vaishnavh Nagarajan, Patrick Shafto, Charlie Chen, Alex Ozdemir, Andres Potapczynski, and Ethan Baron for helpful feedback. This work was supported by 
Google’s TPU Research Cloud (TRC) program: \url{https://sites.research.google/trc}. YJ thanks the support of the Google PhD Fellowship, and SQ thanks the support of the Two Sigma Fellowship.

\bibliography{ref}

\newpage
\appendix
\onecolumn

\newpage

\section*{Appendix Outline}

This appendix provides the technical details, proofs, and experimental specifications supporting the main text. 

\textbf{Appendix \ref{app:proofs}} presents rigorous proofs of all theoretical results, including properties of cryptographically secure pseudorandom number generators under time-bounded entropy and epiplexity (Theorem \ref{thm:csprng_entropy}), creation of information through deterministic transformations (Theorem \ref{thm:creation}), the existence of high-epiplexity random variables (Theorem \ref{thm:high-ep}), the factorization dependence of information content (Theorem \ref{thm:asymmtry}).

\textbf{Appendix \ref{app:measure}} details the practical methodology for estimating epiplexity, covering both prequential and requential coding implementations, hyperparameter optimization procedures for compute-optimal two-part codes, the connection between prequential and requential estimates under a static teacher assumption, and a solvable analytical model combining neural scaling laws with prequential coding. We also establish general properties showing that optimal model size and training tokens increase monotonically with compute budget, that optimal training tokens for prequential coding generally saturate at the test set size for large compute budgets, and that epiplexity and per-token entropy exhibit predictable monotonicity with respect to dataset size.

\textbf{Appendix \ref{app:experiment_details}} provides comprehensive experimental specifications for all empirical results, including architectural choices, hyperparameters, and dataset details for elementary cellular automata experiments, easy and hard variants of induction tasks, chess experiments (with both pre-training data formatting and downstream evaluation tasks), natural data experiments on OpenWebText and CIFAR-5M, comparisons between prequential and requential coding estimates, and scaling law estimation procedures.

\textbf{Appendix \ref{app:rule30_rasp}} presents executable RASP-L code demonstrating that elementary cellular automaton evolution rules can be implemented within the transformer computational model, providing constructive evidence that autoregressive transformers are capable of solving these tasks.

\textbf{Appendix \ref{app:conway}} contains definitions of elementary cellular automata and Conway's Game of Life, emergence examples referenced in the paper.

\textbf{Appendix \ref{app:emergence}} explores additional examples illustrating the relationship between emergence and epiplexity, including the Lorenz system as a case study in chaotic dynamics where entropy is created at a rate determined by Lyapunov exponents, and chess strategy as exemplified by the contrast between AlphaZero's multi-million parameter networks solution at moderate compute and the simple minimax algorithm available at very high compute.

\textbf{Appendix \ref{app:induction-not-specific}} argues that induction phenomena occur not merely in autoregressive models; instead, the key requirement is maximum likelihood estimation rather than autoregressive factorization specifically.

\textbf{Appendix~\ref{app:mdl}} provides a more comprehensive review of MDL, in particular on two-part code, one-part code and the notion of regret, related to epiplexity. %

\textbf{Compute Resources.}\quad 
A cluster of 6 2080Ti was used for many of the smaller scale experiments. A cluster of 6 Titan RTX and 32 TPUv4 provided by the Google TPU Research Cloud was used for the more computationally expensive natural data experiments. We refer the reader to \citet{jiang2025adaptive} for computational resources required in evaluating ADO.

\textbf{Licenses.}\quad The Chess data used in \cref{sec:factorization} is released under Creative Commons CC0 license (\href{https://database.lichess.org/}{\url{database.lichess.org/}}).
The OpenWebText dataset used in \cref{sec:epiplexity-natural} is released under Creative Commons CC0 license.

\section{Proofs}\label{app:proofs}

First, we prove two short lemmas about the basic properties of epiplexity and time-bounded entropy.

\begin{lemma}[Maximum expected description length]
\label{lemma:maxdl}
    For any random variable $X$ on $\{0,1\}^n$ there exists constants $c_1,c_2,c_3$ such that:
    \begin{equation}
        \rS_T(X) + \rH_T(X) \leq n + c_1
    \end{equation}
    for time bounds $T(n) \ge c_2n+c_3$.
\end{lemma}

\begin{proof}
Let $U_n$ be the uniform distribution $Q_{\mathrm{unif}}(x) = 2^{-n}$.
$Q_{\mathrm{unif}}$ can be computed in linear time (just by outputting $2^{-n}$ for each input) and with a program of constant size $c_1$ and in time $c_2n+c_3$ with constants depending on the Turing machine..
$$|Q^\star_X| + \mathbb{E}[-\log Q^\star_X(x)] \leq |Q_{\mathrm{unif}}| + \mathbb{E}[-\log Q_{\mathrm{unif}}(x)] \leq c + n.$$
\end{proof}

\begin{lemma}[Time-bounded entropy of uniform distribution]
    Let $X = U_n$ be the uniform distribution on $\{0,1\}^n$. The time-bounded entropy of $U_n$ for $T(n)\ge c_2n+c_3$ is:
    \begin{equation}
        n\le \rH_T(X) \le n+c_1.
    \end{equation}
\end{lemma}
\begin{proof}

For the lower bound, we have $$\mathbb{E}_X[-\log Q(X)] = \rH(X) + D_{\mathrm{KL}}(P_X \| Q) \ge \rH(X)=n$$ given that the KL is always positive.
For the upper bound, we have that $$\rH_T(X)\le \mathrm{MDL}_T(X) \le n+c$$.
\end{proof}

\subsection{PRGs/CSPRNGs have (nearly) maximal time-bounded Entropy and low epiplexity}\label{app:csprng}

\begin{theorem}
\label{thm:maxent}
Let $X = U_k$ and $n = \ell(k)$ for a non-uniform PRG $G$ that admits advantage $\varepsilon(n)$.
Then, for every polynomial time bound $T(n)$,
\begin{equation}
    \rH_T\bigl(G(U_k)\bigr)\ \ge\ n - 2 - n\,\varepsilon(k).
\end{equation}
\end{theorem}
\begin{proof}
Fix $\mathrm{P}\in \mathcal{P}_T$ and  let $L(x) = -\log P(x)$. For each precision level $t\in \{1, 2, \dots, n\}$, we define the following distinguisher:
$$D_t(x) = \mathbbm{1}\{L(x) \leq n-t\} = \mathbbm{1}\{P(x) \geq 2^{-(n-t)}\}.$$
For any solution $P$ for $\mathrm{MDL}_T$, we have that $\mathrm{MDL}_T(X) = |\mathrm{P}| + \E[-\log P(X)] \leq n + c$. Since both quantities are positive, it must be the case that $|\mathrm{P}| \leq n+c$, which means that $|\mathrm{P}| \in \mathrm{poly}(n)$.
Since $\mathrm{P}$ belongs in $\mathcal{P}_T$ and cannot be longer than $n$, each $D_t$ is a non-uniform PPT algorithm with polysized advice (i.e., $\mathrm{P}$) that PRGs are secure against.

\paragraph{Uniform threshold bound.} Let $U_n$ be uniform on $\{0,1\}^n$ and set $A_t := \{x: D_t(x) = 1\}$.
$$1 \geq \sum_{x} P(x) \geq \sum_{x\in A_t} P(x) \geq |A_t| 2^{-(n-t)} \Rightarrow |A_t| \leq 2^{n-t}.$$
Hence ,$$\Pr[D_t(U_n)=1] = \frac{|A_t|}{2^n} \leq \frac{ 2^{n-t}}{2^n} = 2^{-t}.$$

\paragraph{PRG transfers bound to $X:=G(U_k)$.} By the security of $G$, for each $t$,
$$
\Pr\bigl[D_t(X)=1\bigr]\ \le\ \Pr\bigl[D_t(U_n)=1\bigr]+\varepsilon(k)\ \le\ 2^{-t}+\varepsilon(k),
$$

\paragraph{From threshold probabilities to an entropy lower bound.}
For any non-negative random variable $Z$, we have the layercake representation:
\begin{align}
    \mathbb{E}[Z] &= \sum_{u=0}^\infty (1-P(Z \leq u))\\
     n-\mathbb{E}[Z] &= \sum_{u=0}^{n-1}1 -\sum_{u=0}^\infty (1-P(Z \leq u))\\
    &= \sum_{u=0}^{n-1}1 -\sum_{u=0}^{n-1} (1-P(Z \leq u))-\sum_{u=n}^\infty (1-P(Z \leq u))
    \\&= \sum_{u=0}^{n-1} P(Z \leq u)-\sum_{u=n}^\infty (1-P(Z \leq u))\\
    &\leq \sum_{u=0}^{n-1} P(Z \leq u).
\end{align}
Now we change the bounds to be in terms of $t$ with $t=n-u$. The lower bound becomes $t = n$. The upper bound becomes $t = 1$, which yields
$$n-\mathbb{E}[Z] \leq  \sum_{u=0}^{n-1} P(Z \leq u) = \sum_{t=1}^{n} P(Z \leq n-t).$$
Let $Z = L(X) = -\log P(X)$:
$$n-\mathbb{E}[Z] \leq \sum_{t=1}^{n} P(Z \leq n-t) = \sum_{t=1}^{n} P(D_t(X)=1)  \leq \sum_{t=1}^{n} 2^{-t} + \varepsilon(k)\leq 1 + n\varepsilon(k).$$
The last two steps come from the fact that $X$ is a CSPRNG. This means that:
$$n - \mathbb{E}[L(X)] \leq 1 + n\varepsilon(k) \Rightarrow \mathbb{E}[-\log P(X)]\geq n - n\varepsilon(k)-1.$$
Since this is true for any $P\in \mathcal{P}_T$, taking the minimum yields:
$$\rH_T(X) = \rH_T(G(U_n)) = \min_{P \in \mathcal{P}_T}\mathbb{E}[-\log P(X)]\geq n - n\varepsilon(k)-1.$$
\end{proof}

\subsection{Deterministic transformation can increase time bounded entropy and epiplexity}
\label{sec:epi_increase}

\begin{theorem}
    Let $G: \{0,1\}^k \to \{0,1\}^n$ be a $\mathrm{CSPRNG}$ which admits advantage $\varepsilon(k)$ and $U_k$ be the uniform distribution. $\rH_{\mathrm{Poly}}(G(U_k))>\rH_{\mathrm{Poly}}(U_k) + n-k-n\varepsilon(k)-c$ for a fixed constant $c$. Proof: see Appendix~\ref{app:csprng}.
\end{theorem}

\begin{proof}
By Lemma~\ref{lemma:maxdl} applied to the uniform distribution on $\{0,1\}^k$,
there is an absolute constant $c$ such that
\[
\rH_{\mathrm{poly}}(U_k)\le k+c.
\]
Rearranging gives $k \ge \rH_{\mathrm{poly}}(U_k)-O(1)$.
Combining this with the assumed CSPRNG lower bound (Lemma~\ref{thm:maxent}),
\[
\rH_{\mathrm{poly}}(G(U_k)) \ge n-2-n\varepsilon(k),
\]
we obtain,
\begin{align*}
    &\rH_{\mathrm{poly}}(G(U_k)) - \rH_{\mathrm{poly}}(U_k) \geq n-2-n\varepsilon(k) - (k+c)\\
    \Rightarrow &\rH_{\mathrm{Poly}}(G(U_k))>\rH_{\mathrm{Poly}}(U_k) + n-n\varepsilon(k)-k - O(1).
\end{align*}
\end{proof}
\subsection{CSPRNGs have low epiplexity}
\label{sec:csprng-epi}
\begin{theorem}
    Let $X = U_k$ and $n = \ell(k)$ for CSPRNG $G$ that admits advantange $\varepsilon(n)$.
    Then, for every polynomial time bound $T(n)$, the epiplexity of $Y = G(X)$ is,
    \begin{equation}
        \rS_T(Y) \leq c + n\varepsilon(k).
    \end{equation}
\end{theorem}
\begin{proof}
    We know from Theorem~\ref{thm:maxent} that $\rH_T(G(U_k)) \geq n-n\varepsilon(k) -2$, which means:
    \begin{align}
        \rS_T(Y) + \rH_T(Y) \geq \rS_T(Y) + n - n\varepsilon(k) - 2.
    \end{align}
    We also have from Lemma~\ref{lemma:maxdl} that $\rS_T(Y) + \rH_T(Y) \leq n + c$. Combining these two results yields:
    \begin{align}
        \rS_T(Y) + n - n\varepsilon(k)-1 \leq n +c 
        \Rightarrow \rS_T(Y) \leq c + n \varepsilon(k).
    \end{align}
\end{proof}

\subsection{Existence of High Epiplexity random variables}
\label{app:high-ep}

\begin{definition}[Pseudorandom functions (PRF)]
Let $\mathrm{PRF}$ be the class of keyed functions $F: \{0,1\}^k \times \{0,1\}^n \to \{0,1\}^m$ that are computable in polynomial time and satisfy the following property: For any probabilistic polynomial-time distinguisher $D$ with oracle access to the provided function,
    \begin{equation}
        |\Pr_{K \sim U_k}[D^{F_K(\cdot)}] - \Pr_{f \sim \mathcal{F}_n}[D^{f(\cdot)}]| < \frac{1}{n^c},
    \end{equation}
for all integers $c > 0$ and sufficiently large $n$. Here, $F_K(\cdot)$ denotes the function $F(K, \cdot)$ with the key $K$ fixed, and $\mathcal{F}_n$ is the set of all functions mapping $\{0,1\}^n$ to $\{0,1\}^m$.
\end{definition}

\paragraph{Cryptographic assumptions.} Assume one-way functions exist (secure against non-uniform PPT adversaries with inversion probability at most 
$\varepsilon(n)$). By standard constructions~\citep{haastad1999pseudorandom}, this implies the existence of PRFs secure against non-uniform PPT distinguishers with advantage $\mathrm{poly}(\varepsilon(n))$ (and in particular negligible if $\varepsilon(n)$ is negligible).

\begin{definition}[Heavy set]
\label{def:heavyset}
    For a distribution $Q$ on $\{0,1\}^n$, $m < n$, and a fixed threshold $t \geq 0$, the $(Q,t)$-heavy set is:
    \begin{align}
        A_{Q,t} := \{z: Q(z) \geq 2^{-2(m+t)}\}.
    \end{align}
\end{definition}
\begin{lemma}
\label{lemma:heavyset1}
    Let $P$ be a distribution on $\{0,1\}^n$ with entropy $\rH(P)=m$. If $\mathrm{KL}(P, Q) \leq t$, then $P(A_{Q,t}) \geq \frac{1}{2}$.
\end{lemma}
\begin{proof}
    First, observe the standard inequality:
    $$\E_{z\sim P}\left[\log\frac{1}{Q(z)} \right] = \rH(P) + \mathrm{KL}(P\|Q) \leq m + t.$$
    Applying Markov's inequality, we get:
    \begin{align}
        \Pr_{z\sim P}\left[\log\frac{1}{Q(z)} \geq 2(m+t)\right] \leq \frac{\E[-\log Q(z)]}{2(m+t)} \leq \frac{1}{2}.
    \end{align}
    Taking the complement gives:
    \begin{align}
        \Pr_{z\sim P}\left[\log\frac{1}{Q(z)} \leq 2(m+t)\right] = \Pr_{z\sim P}\left[Q(z) \geq 2^{-2(m+t)}\right] = P(A_{Q,t}) \geq  \frac{1}{2}.
    \end{align}
\end{proof}
\begin{lemma}
\label{lemma:heavyset2}
    Let $U_n$ be the uniform distribution over $\{0,1\}^n$, the weights of $A_{Q,t}$ under $U_n$ is $U_n(A_{Q,t}) \leq 2^{-(n-2(m+t))}$
\end{lemma}
\begin{proof}
For $z \sim U_n$, we have $\E_{z \sim U_n}[Q(z)] = \sum_{z} 2^{-n} Q(z) = 2^{-n}$. Applying Markov' inequaltiy:
\begin{align}
    \Pr_{z\sim U_n}\left[Q(z) \geq 2^{-2(m+t)}\right] \leq \frac{\E_{z \sim U_n}[Q(z)]}{2^{-2(m+t)}} \leq 2^{-n + 2(m+t)} = 2^{-(n-2(m+t))}.
\end{align}
\end{proof}

\begin{theorem}
    If there exists a PRF family $F_K: \{0,1\}^m \rightarrow \{0,1\}^k$ that is indexed by $K \in \{0,1\}^m$ and secure against a non-uniform PPT distinguisher $D_m$ allowing for an advantage of at most $\varepsilon(m)$, there exists $n_0$ such that for all $n = m+k \geq n_0$, there exists a sequence of random variables $\{X_k\}_{k=1}^n$ over $\{0,1\}^n$ such that $\mathrm{S}_{\mathrm{Poly}}(X_n) = \Omega(\log n)$.
\end{theorem}

\begin{proof}
    We will prove the existence of such $P$ via a counting argument. First, we define the family of distributions of interest. Concretely, we draw a sample $P_K$ as follows:
    \begin{enumerate}
        \item Sample $x \sim U_m$
        \item Output $z = (x, F_K(x)) \in \{0,1\}^n$
    \end{enumerate}
    Since $F_K$ is a deterministic function, $\rH(P_K) = m$.

    We also defined a \emph{keyed model} $\mathrm{Q}_K$ that models $P_K$ by directly storing the key $K$ and the program for generating PRF from $K$ inside its program:
    $$Q_K(x,y) = 2^{-m} \mathbbm{1}\{y = F_K(x)\}.$$
    This model matches the density of $P_K$ so $\mathrm{KL}(P_K \| Q_K) = 0$, and:
    $$L(Q_K, P_K) = |\mathrm{Q}_K| + \rH(P_K) \leq m + c_1 + m = 2m + c_1.$$
    $c_1$ is the constant overhead to implement the PRF evaluation and sampling wrapper under a fixed encoding (i.e., a UTM).

    \paragraph{Constructing distinguisher from $Q$.} Given a model $Q$ and its heavy set $A_{Q,t}$ (Definition~\ref{def:heavyset}), we can turn $Q$ into a \emph{single-query} distinguisher $D^O$:
    \begin{enumerate}
        \item Sample $x\sim U_m$ and query the oracle $y=O(x)$ and set $z = (x, y)$.
        \item Output $1$ if $z \in A_{Q,t}$ i.e., $Q(z) \geq 2^{-2(m+t)}$ else $0$.
    \end{enumerate}
    If $O$ is a truly random function $R$, then $(x, R(x))$ follows $U_n$ and by Lemma~\ref{lemma:heavyset2}:
    \begin{align}
        \Pr[D^R=1] = \Pr_{z\sim U_n}\left[z \in A_{Q,t}\right] \leq 2^{-(n-2(m+t))}
    \end{align}
    If $O$ is the PRF $F_K$ for a $K$ that satisfies $\mathrm{KL}(P_K \| Q) \leq t$, Lemma~\ref{lemma:heavyset1} gives:
    \begin{align}
        \Pr\left[D^{F_K}=1 \mid \mathrm{KL}(P_K \| Q) \leq t\right] \geq \frac{1}{2}.
    \end{align}
    Let $p_{Q, t} = \Pr_K[\mathrm{KL}(P_K \| Q) \leq t]$. We can average over all possible $K$ and obtain the following bound:
    \begin{align}
        \Pr\left[D^{F_K}=1\right] \geq \Pr_K[\mathrm{KL}(P_K \| Q) \leq t] \Pr\left[D^{F_K}=1 \mid \mathrm{KL}(P_K \| Q) \leq t\right] \geq \frac{1}{2} p_{Q, t}.
    \end{align}
    Therefore,  the distinguishing advantage of $D^O$ is:
    \begin{align}
        \mathsf{Adv}(D^O) = \Pr\left[D^{F_K}=1\right] - \Pr[D^R=1] \geq \frac{1}{2} p_{Q, t} - 2^{-(n-2(m+t))}.
    \end{align}
    Rearranging:
    \begin{align}
        p_{Q, t} \leq 2\mathsf{Adv}(D^O) + 2 \cdot 2^{-(n-2(m+t))}.
    \end{align}
    Since $F_K$ is a PRF and $D_O$ is a PPT distinguisher, the advantage is upperbounded by $\varepsilon(m)$:
    \begin{align}
        p_{Q, t} \leq 2\varepsilon(m) + 2 \cdot 2^{-(n-2(m+t))}.
    \end{align}
\paragraph{Union bound over short models.} Given a maximum program length $s$, there are at most $2^{s+1}$ candidate programs $\mathrm{Q}$ with $|\mathrm{Q}| \leq s$. Applying union bound on all such $Q$'s:
\begin{align}
    \Pr_{K}\left[\exists \mathrm{Q} : |\mathrm{Q}| \leq s \,\wedge \, \mathrm{KL}(P_K\|Q)\leq t\right]\leq 2^{s+1} p_{Q, t} \leq 2^{s+1}\left(2\varepsilon(m) + 2 \cdot 2^{-(n-2(m+t))}\right).
    \label{eq:union}
\end{align}

Now, it suffices to choose parameters such that the RHS of equation~\ref{eq:union} is smaller than $1$, which implies there exists a hard key $K^\star$ such that:
\begin{align}
    \mathrm{KL}(P_{K^\star}\|Q) > t, \,\,\forall \mathrm{Q} \,\,\text{satisfying} \,\,|\mathrm{Q}| \leq s.
\end{align}

\paragraph{MDL lower bound from $K^\star$.} 
For $K^\star$, every $|\mathrm{Q}|\leq s$ satisfies:
$$L(Q, P_{K^\star}) = |\mathrm{Q}| + \rH(P^\star) + \mathrm{KL}(P_{K^\star}\|Q) \geq \rH(P^\star) + \mathrm{KL}(P_{K^\star}\|Q) \geq m + t.$$
Meanwhile, the keyed model $Q_{K^\star}$ satisfies: $L(Q_{K^\star}, P_{K^\star}) \leq 2m + c_1.$
If we set:
$$t = m + c_1 + \Delta,$$
we get a margin of $\Delta$:
\begin{align}
 L(Q, P_{K^\star}) \geq m + m + c_1 + \Delta > 2m + c_1 \geq   L(Q_{K^\star}, P_{K^\star}).  
\end{align}
This implies that there exists at least one model that achieves a lower description length than any $\mathrm{Q}$ with $|\mathrm{Q}|\leq s$ and the MDL minimizer must have $|\mathrm{Q}^\star| > s$.

\paragraph{Choosing parameters.} 
Set:
\begin{itemize}
    \item $s = \log m$
    \item $\Delta = \log m$
    \item $t = m + c_1 + \Delta = m + c_1 + \log m$
    \item $k = 4m + 4\Delta + 2c_1$
\end{itemize}
We now plug these values into \autoref{eq:union}. First, $2^{s+1} = \mathrm{poly}(m)$ and $\lim_{m \rightarrow \infty} 2^{s+1}\cdot 2 \varepsilon(m) = 0$. For the second term:
\begin{align*}
    &2^{s+1} \cdot 2 \cdot 2^{-(n-2(m+t))} \\
    =& 2^{\log m+1} \cdot 2 \cdot 2^{-(m+4m+4\Delta + 2c_1-2(m+m + c_1 + \log m))} \\
    =& 2^{\log m+2}  \cdot 2^{-(5m+4\log m + 2c_1-2(2m + c_1 + \log m))}\\
    =& 2^{\log m+2}  \cdot 2^{-(m+2\log m)} \\
    =& 2^{-m-\log m+2}.
\end{align*}
This term also approaches $0$ as $m$ increases. So for sufficiently large $m$ the RHS of \autoref{eq:union} is less than $1$ as desired.

\end{proof}

\subsection{Information Content is not Independent of Factorization}
\label{sec:inf_fac}

\begin{theorem}[OWP induces entropy asymmetry]
\label{thm:soi_entropy}
Let $f:\{0,1\}^n\to\{0,1\}^n$ be a polynomial-time computable one-way permutation secure
against non-uniform PPT inverters with negligible success probability. Let
$X= U_n$ and $Y=f(X)$. 
Let $\rH_{\mathrm{poly}}(\cdot)$ and $\rH_{\mathrm{poly}}(\cdot\mid\cdot)$ be defined as in
Definition~\ref{def:epiplexity}. Then for every constant $c>0$ there exists $N$ such that
for all $n\ge N$,
\[
\rH_{\mathrm{poly}}(X \mid Y) + \rH_{\mathrm{poly}}(Y)
>
\rH_{\mathrm{poly}}(Y\mid X) + \rH_{\mathrm{poly}}(X) + c\log n.
\]
\end{theorem}

\begin{proof}
We prove bounds on each term.

\paragraph{Unconditional terms $\rH_{\mathrm{poly}}(X)$ and $\rH_{\mathrm{poly}}(Y)$.}
Since $X= U_n$ and $f$ is a permutation, $Y=f(X)$ is also uniform on $\{0,1\}^n$.
By Lemma~\ref{lemma:maxdl} (time-bounded entropy of the uniform distribution),
there is a constant $c_0$ such that
\[
n \le \rH_{\mathrm{poly}}(X) \le n+c_0,\qquad n \le \rH_{\mathrm{poly}}(Y) \le n+c_0.
\]
In particular, $-c_0 \leq \rH_{\mathrm{poly}}(Y)-\rH_{\mathrm{poly}}(X)\leq c_0$, so $\rH_{\mathrm{poly}}(Y)-\rH_{\mathrm{poly}}(X) = O(1)$.

\paragraph{Forward conditional term $\rH_{\mathrm{poly}}(Y\mid X)$.}
There is a deterministic conditional sampler that on input $x$ outputs $f(x)$.
For this sampler, $P(Y\mid X)=1$, hence $\log(1/P(Y\mid X))=0$.
Since $\rH_{\mathrm{poly}}(Y\mid X)$ is the expected log-loss of the MDL-optimal conditional sampler,
we obtain
\[
\rH_{\mathrm{poly}}(Y\mid X)=O(1).
\]

\paragraph{Hard conditional term $\rH_{\mathrm{poly}}(X\mid Y)$.}
Let $P^\star:=P^\star_{X\mid Y}$ be the MDL-optimal conditional probabilistic model for $X\mid Y$
over the class of non-uniform PPT model, and define
\[
\phi(y)\;:=\; \Pr_{u\sim U_\infty}\left[\mathrm{Sample}_{P^\star_{X\mid y}}(u) = f^{-1}(y)\right].
\]
Because $Y=f(X)$ and $f$ is a permutation, we have $X=f^{-1}(Y)$, and thus
\[
P^\star(X\mid Y)=P^\star(f^{-1}(Y)\mid Y)=\phi(Y)\qquad\text{a.s.}
\]
Therefore
\[
\rH_{\mathrm{poly}}(X\mid Y)
=\E\Big[\log\frac{1}{P^\star(X\mid Y)}\Big]
=\E\Big[\log\frac{1}{\phi(Y)}\Big].
\]
By Jensen's inequality for the convex function $\log(1/t)$,
\[
\E\Big[\log\frac{1}{\phi(Y)}\Big]\;\ge\;\log\frac{1}{\E[\phi(Y)]}.
\]
Now consider the inverter $\mathcal I$ that on input $y$ runs the sampler $P^\star(X\mid Y)$
once and outputs the resulting $x$. Since $P^\star$ is a non-uniform PPT sampler, $\mathcal I$
is a non-uniform PPT inverter. Moreover, its inversion success probability is exactly
\[
\Pr[\mathcal I(Y)=f^{-1}(Y)] = \E[\phi(Y)].
\]
Equivalently (since $Y=f(X)$),
\[
\Pr_{X\sim U_n}\big[\mathcal I(f(X))=X\big] = \E[\phi(Y)].
\]
By one-wayness, this success probability is negligible. In particular, for every constant
$c>0$ there exists $N$ such that for all $n\ge N$,
\[
\E[\phi(Y)] \le n^{-c}.
\]
Plugging into the Jensen bound yields, for all $n\ge N$,
\[
\rH_{\mathrm{poly}}(X\mid Y)\;\ge\;\log\frac{1}{\E[\phi(Y)]}\;\ge\; c\log n.
\]

\paragraph{Combine.}
For $n\ge N$, we have
\begin{align}
\rH_{\mathrm{poly}}(X\mid Y)+\rH_{\mathrm{poly}}(Y)
&\ge c\log n + \rH_{\mathrm{poly}}(Y)\\
& \ge c\log n + \rH_{\mathrm{poly}}(X) - O(1) \\
&= \rH_{\mathrm{poly}}(Y\mid X)+\rH_{\mathrm{poly}}(X)+c\log n - O(1),
\end{align}
where we used $\rH_{\mathrm{poly}}(Y\mid X)=O(1)$ and $\rH_{\mathrm{poly}}(Y)-\rH_{\mathrm{poly}}(X)\ge -c_0$.
\end{proof}

\begin{corollary}
\label{cor:inf_fac2}
     Let $f$ be a one-way permutation and lef $X = \mathrm{Unif}(\{0, 1\}^n), Y=f(X)$. Define $\mathcal{P}$ as a family of probabilistic generative model that allows for multiple factorizations of the data, ie $P\in \mathcal P$ it can make predictions $P_{1\to 2}(X,Y)=P_1(X)P_2(Y;X)$ and $P_{2\to 1}(X,Y)=P_2(Y)P_1(X;Y)$ for the functions $P_1(\cdot),P_1(\cdot \ ;\cdot)$,$P_2(\cdot),P_2(\cdot \ ;\cdot)$ that are normalized probability distributions over the first variable. 

     Suppose that $P$ fits the forward direction of $f$ (and the input uniform distributions)
     \begin{align*}
         \E[-\log P_1(X)] &\le n+\varepsilon\\
         \E[-\log P_2(f(X)\mid X)] &\le \varepsilon
     \end{align*}

    then it must violate Bayes theorem $P_{1\to 2} = P_{2\to 1}$ by a margin growing with $n$.
     
     Specifically, for any value of $c$ there exists $N$ such that for all $n > N$, there exists at least one $x \in \{0,1\}^n$ such that
     
    \begin{equation}
         P_1(x)P_2(f(x);x)> n^c 2^{-2\varepsilon} P_2(f(x))P_1(x;f(x))
    \end{equation}
\end{corollary}
\begin{proof}
From \autoref{thm:soi_entropy} which applies also for each $P$, we have
\begin{equation*}
     \E\left[-\log P_2(X ; Y)\right] > c \log n .
\end{equation*}
       
The minimim value of $\E\left[-\log P_2(f(X))\right]$ is $n$ since $f$ is a bijection. Assembling these components,
\begin{equation}
   \E\left[\log \frac{P_1(X)P_2(f(X);X)}{P_2(f(X)) P_1(X ; f(X))}\right] >c \log n - 2\varepsilon.
\end{equation}
Since the inequality holds in expectation, it also must hold for at least one value of $X$. Exponentiating provides the final result.
\end{proof}

\subsection{Problems with time-bounded sophistication}
\label{app:time_bounded_sophistication}

Epiplexity can be seen as a time-bounded and distributional generalization of sophistication. A natural question is whether we can directly define a time-bounded version of sophistication for individual strings.
We show below that a naive time-bounded generalization degenerates: it makes the ``model'' part essentially constant for \emph{every} string.

\paragraph{Preliminaries.}
Fix a reference universal (prefix-free or plain) Turing machine $U$.
For a program $p$ and auxiliary input $d$, we write $U(p,d)$ for the output of running $p$ on input $d$.
The length of a binary string $p$ is denoted $|p|$.
A program $p$ is \emph{total} if $U(p,d)$ halts for every input $d$ (i.e., $p$ computes a total function).

We write $K(x)$ for Kolmogorov complexity (plain or prefix; the choice only changes values by $O(1)$).
For a time bound $t(\cdot)$, the time-bounded Kolmogorov complexity is
\[
K^{t}(x) \;:=\; \min\bigl\{\, |q| \;:\; U(q) \text{ outputs } x \text{ within } t(|x|) \text{ steps}\,\bigr\}.
\]
(Any standard time-constructible $t$ suffices for the discussion.)

We adopt the definition of sophistication from \citet{koppel1987structure} and \citet{antunes2006sophistication}, phrased for finite strings as in later expositions.
For a significance level $c \ge 0$, the sophistication of $x$ is
\begin{definition}[Sophistication at significance $c$]
\label{def:sophistication}
\[
\mathrm{soph}_c(x)
\;:=\;
\min_{p}\Bigl\{\, |p| \;:\; p \text{ is total and } \exists d \text{ such that } U(p,d)=x \text{ and } |p|+|d| \le K(x)+c \Bigr\}.
\]
\end{definition}
Intuitively, $(p,d)$ is a near-optimal two-part description of $x$.
The requirement that $p$ be \emph{total} is crucial: it prevents taking $p$ to be a tiny universal interpreter and pushing all information into $d$ (since a universal interpreter is not total).
One of the most intuitive attempts at ``time-bounded sophistication'' is to simply replace $K(x)$ by the time-bounded complexity $K^{t}(x)$ in Definition~\ref{def:sophistication}.

\begin{definition}[Naive time-bounded sophistication]
\label{def:naive_time_soph}
Fix a time bound $t(\cdot)$ and significance level $c\ge 0$. Define
\[
\mathrm{soph}^{t}_c(x)
\;:=\;
\min_{p}\Bigl\{\, |p| \;:\; p \text{ is total and } \exists d \text{ such that } U(p,d)=x \text{ and } |p|+|d| \le K^{t}(x)+c \Bigr\}.
\]
\end{definition}

The definition above \emph{collapses}, essentially because time bounds make it easy to ``totalize'' a universal interpreter by adding a timeout.

\begin{lemma}[Naive time-bounded sophistication is $O(1)$]
\label{prop:time_soph_degenerate}
For every time bound $t(\cdot)$ and every $c\ge 0$, there exists a constant $C_t$ (depending only on $t$ and the choice of $U$) such that for every string $x$,
\[
\mathrm{soph}^{t}_c(x) \le C_t.
\]
In particular, $\mathrm{soph}^{t}_c(x)$ does not meaningfully distinguish structured strings from random-looking strings.
\end{lemma}

\begin{proof}[sketch]
Fix $t$.
Let $p_{\mathrm{tl}}$ be a constant-size program that, on input $d$, simulates $U(d)$ for at most $t(|x|)$ steps (or more generally for the same time budget used in the definition of $K^{t}(x)$), and:
(i) if the simulation halts within the budget, output the same result; otherwise
(ii) output a fixed default string (say $0$).
By construction, $p_{\mathrm{tl}}$ is \emph{total} (it always halts, because it enforces a timeout).

Now let $d^\star$ be a shortest program witnessing $K^{t}(x)$, i.e., $|d^\star| = K^{t}(x)$ and $U(d^\star)$ outputs $x$ within the allowed time.
Then $U(p_{\mathrm{tl}}, d^\star)=x$.
Moreover,
\[
|p_{\mathrm{tl}}| + |d^\star|
=
|p_{\mathrm{tl}}| + K^{t}(x)
\le
K^{t}(x) + c
\quad\text{for all } c \ge |p_{\mathrm{tl}}|.
\]
Thus $p_{\mathrm{tl}}$ is feasible in Definition~\ref{def:naive_time_soph}, giving
$\mathrm{soph}^{t}_c(x) \le |p_{\mathrm{tl}}| = C_t$ for all $x$.
\end{proof}

In the original (unbounded-time) Definition~\ref{def:sophistication}, totality prevents a universal interpreter from being used as the ``model'' part, because such an interpreter cannot halt on inputs that encode non-halting computations.
However, once we commit to a time bound in the \emph{optimality criterion} (i.e., we compare against $K^{t}(x)$), the data part $d$ can be chosen to be a short program that is \emph{guaranteed to halt quickly}.
A constant-size \emph{clocked interpreter} $p_{\mathrm{tl}}$ is then total and suffices for every $x$, pushing all of the description length into $d$.
This is precisely the sense in which the naive time-bounded generalization becomes degenerate.

\section{Measuring Epiplexity}\label{app:measure}
\subsection{Further details on estimating epiplexity} \label{app:detail-procedure}
Here we provide further details on measuring epiplexity.
\paragraph{Evaluating code lengths and time bounds.}
As described in \Cref{sec:measuring}, evaluating the code length for the model boils down to tracking the training losses (prequential) or teacher-student KL (requential) at each step $i:$
\begin{align}
    |\mathrm{P}_{\mathrm{preq}}| &\,\approx \sum_{i=0}^{M-1} \qty(\log 1/P_i(Z_i) - \log 1/P_M(Z_i)), \\
    |\mathrm{P}_{\mathrm{req}}| &\,\approx \sum_{i=0}^{M-1} \mathrm{KL}(P^{\mathrm{t}}_i\|P^{\mathrm{s}}_i).
\end{align}

For prequential coding, we need to compute the loss of the final model summed over the entire training dataset, $\sum_{i=0}^{M-1} \log 1/P_M(Z_i)$, which is time-consuming if done exactly. Since all of our experiments are in the one-epoch training regime without data repeat and training data $Z_i$ are drawn i.i.d. (except for the ADO experiment \Cref{sec:ado}), we make the assumption that the generalization gap is small and estimate $\sum_{i=0}^{M-1} \log 1/P_M(Z_i)$ as $M \log 1/P_{M}(Z_{M}),$ where the latter is a rescaled loss for $P_M$ on unseen data $Z_M.$ The i.i.d. assumption breaks down for the ADO experiment \Cref{sec:ado}, where we instead compute $\sum_{i=0}^{M-1} \log 1/P_M(Z_i)$ exactly.

For requential coding, we need to evaluate the teacher-student KL, $\mathrm{KL}(P^{\mathrm{t}}\|P^{\mathrm{s}}),$ at each training step. The KL divergence over sequences decomposes as a sum over token positions and is estimated as:
\begin{align}
\mathrm{KL}(P^{\mathrm{t}}\|P^{\mathrm{s}}) &= \sum_{j=1}^{L} \E_{Z_{<j} \sim P^{\mathrm{t}}} \left[ \sum_{Z_j \in \mathcal{V}} P^{\mathrm{t}}(Z_j | Z_{<j}) \log \frac{P^{\mathrm{t}}(Z_j | Z_{<j})}{P^{\mathrm{s}}(Z_j | Z_{<j})} \right] \\
&\approx \sum_{j=1}^{L} \sum_{Z'_j \in \mathcal{V}} P^{\mathrm{t}}(Z'_j | Z_{<j}) \log \frac{P^{\mathrm{t}}(Z'_j | Z_{<j})}{P^{\mathrm{s}}(Z'_j | Z_{<j})},
\end{align}
where $Z \sim P^{\mathrm{t}}$ is a sample from the teacher, $L$ is the sequence length, and $\mathcal{V}$ is the vocabulary. We evaluate this estimator using the sample $Z$ generated by the teacher to train the student, along with their next-token-prediction logits $\{P^{\mathrm{t}}(Z_j|Z_{<j}), P^{\mathrm{s}}(Z_j|Z_{<j})\}_j$ recorded on the generated sequence

Finally, to estimate the expected entropy code length for the test data $\E[\log 1/P(X)]$ under the trained model $P,$ we use an appropriately scaled empirical entropy code length of a heldout test set $\hat{X}.$ Let $K$ and $\hat{K}$ denote the number of examples in each dataset. Then:
\begin{align}
 \E[\log 1/P(X)] &= \E\qty[\log \frac{1}{\prod_i P(X_i)}] \\
 &= \sum_i \E[\log 1/P(X_i)] \\
 &= K \E[\log 1/P(X_1)] \\
&\approx \frac{K}{\hat{K}}\sum_{i=1}^{\hat{K}} \log 1/P(\hat{X}_i)    
\end{align}
where we assumed the datasets $X$ and $\hat{X}$ consist of i.i.d. draws from the same distribution. This estimator is simply a scaled version of the standard empirical test loss, and it converges to the true expectation as $\hat{K}$ becomes large. To speedup evaluation, we typically choose $\hat{K} \ll K,$ but this choice does not affect our time-bound calculation: for both prequential and requential coding, the total decoding time of the two-part code for the test dataset $X$ is estimated as $6ND + 2N\D$ where $N$ is the number of parameters of the (student) model, $D$ is the number of (student) training tokens, and $\D$ is the number of tokens in the test dataset. When evaluating conditional epiplexity $\rS_T(Y|X),$ decoding time takes into account both the input ($X$) and label ($Y$) tokens, but code length only needs to be computed for the label tokens (tokens contributing to the training loss).

\paragraph{Finding Hyperparameters for Compute-Optimal Two-Part Code.}
To identify models that lead to compute-optimal two-part code, we need to optimize several key hyperparameters, including model size ($N$), training tokens ($D$), width-depth ratio, learning rate, etc. Through our early experiments, we found two interventions that reduce the model code length under requential coding: (1) distilling from an exponential moving average (EMA) of teacher checkpoints rather than instantaneous checkpoints, which reduces noise in the distillation signal, and (2) imposing a maximum KL threshold between teacher and student—when exceeded, the teacher is frozen while the student catches up, preventing divergence that would otherwise inflate the code length. The EMA time scale and the maximum KL threshold are additional hyperparameters for requential coding.

In each experiment, we first identify a good learning rate for a small model size and use the Maximum Update Parameterization \citep{yang2022tensor} and CompleteP \citep{dey2025don} to transfer the found learning rate to larger models. We also optimize the EMA time scale and maximum KL threshold for the small model when using requential coding. We then train models of various depths and widths to simultaneously sweep over model size and width-depth ratios, for a total number of training tokens chosen to be larger than the test dataset size $\D,$ motivated by the observation that the optimal training tokens typically grows with the model size but do not exceed $\D$ (see \Cref{app:epi-scaling}).  To avoid separately training a model for intermediate training token budgets, we record an EMA of the iterates (for requential coding, this is done for the student) under a constant learning rate schedule, rather than using a decaying learning rate schedule, following \citet{hagele2024scaling}. Each training run traces a curve in the $|\mathrm{P}| + \E[1 / \log P(X)]$ vs $T$ plane as more training tokens are seen. The Pareto frontier of all such curves yields the optimal hyperparameters ($N, D,$ width, depth, etc.) as a function of the compute budget.

\begin{figure*}[t]
\centering
\includegraphics[width=\linewidth]{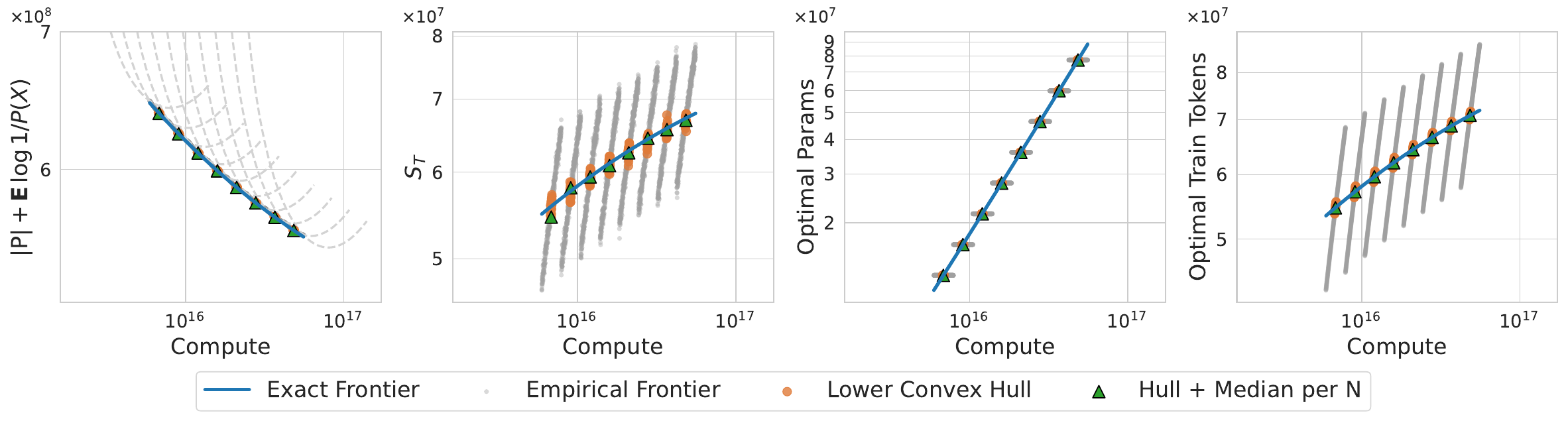}
\caption{
\small
\textbf{Estimating the Pareto frontier from a finite number of training runs.} While the exact Pareto frontier is smooth and the optimal model size and training tokens increase smoothly with compute, the empirical frontier is jagged and includes many spurious points due to selecting over only a finite number of hyperparameter combinations. Replacing the empirical Pareto frontier with the lower convex hull and retaining only the median point (ordered by compute) belong to a single training run with a fixed model size results in a much more accurate estimate of the true Pareto frontier. The example training curves are generated using the scaling laws in \citet{hoffmann2022training} and prequential coding. The exact frontier is found via root finding for \Cref{eq:scaling-foc}.
}
\label{fig:pareto-error}
\end{figure*}

\paragraph{Estimating the Pareto Frontier.} Due to computational constraints, we can only sweep over a limited set of hyperparameter combinations, which makes the empirical Pareto frontier noisy and jagged; we therefore use the lower convex hull of the resulting curves as a smoother approximation to the true Pareto frontier, a strategy often used in the compute-optimal scaling law literature \citep{henighan2020scaling,mcleish2025gemstones} to overcome similar issues. After applying this strategy, we still often observe that multiple checkpoints from a single training run appear on the Pareto frontier. This is an artifact of finite hyperparameter sweeps: we expect both the optimal training tokens $D$ and model size $N$ to vary smoothly with compute budget, precluding multiple values of $D$ at the same $N$ from lying on the true Pareto frontier. These spurious points cause noisy, oscillatory trends in the estimated epiplexity, as shown in \Cref{fig:pareto-error}. As a simple workaround, we retain only the median point (ordered by compute) per training run (which has a fixed model size) on the lower convex hull.

\paragraph{Sources of errors.} In addition to the artifacts produced by finite $(N,D)$ combinations, our estimated epiplexity may differ from the true value for a few reasons: 1) potential systematic errors introduced by using the lower convex hull and taking the median point, 2) using a fixed architecture (e.g., the transformer) and learning algorithm (e.g., requential training with Adam) rather than considering all possible programs, and 3) suboptimality of other hyperparameters, such as the learning rate, Adam $(\beta_1,\beta_2),$ etc. In most cases, we believe these sources of errors only contribute sub-leading corrections to the estimated epiplexity that do not impact the result qualitatively. For example, they are unlikely to alter the ordering between datasets if the estimated epiplexity gap is already significant or there is a clear trend along some axis of variation (e.g., number of hidden bits in the induction experiment in \Cref{sec:induction})

\subsection{Prequential Coding Approximates Requential Coding with a Static Teacher} \label{app:preq-as-approx}

In this section, we show that the prequential coding estimate in \Cref{eq:preq} can be viewed as an approximation to requential coding with a static teacher, providing an alternative justification for its use beyond the symmetry of information argument.

Consider requential coding with a fixed teacher across all time steps, i.e., $P^{\mathrm{t}}_i = P^{\mathrm{t}}$ for all $i \in \{0, \ldots, M-1\}$. The requential code length becomes
\begin{equation}
    |\mathrm{P}_{\mathrm{req}}| \approx \sum_{i=0}^{M-1} \mathrm{KL}(P^{\mathrm{t}} \| P^{\mathrm{s}}_i) = \sum_{i=0}^{M-1} \E_{P^{\mathrm{t}}}\left[\log \frac{1}{P^{\mathrm{s}}_i(X)} - \log \frac{1}{P^{\mathrm{t}}(X)}\right].
\end{equation}

Now suppose the static teacher closely matches the true data distribution, i.e., $P^{\mathrm{t}} \approx P_{X_1}$(we use $P_{X_1}$ in order to refer to the distribution of a single example, not the dataset). Under this assumption, we can make three simplifying approximations:
\begin{enumerate}
    \item The expectation under the teacher can be replaced by the expectation under the data distribution: $\E_{P^{\mathrm{t}}}[\cdot] \approx \E_{P_{X_1}}[\cdot]$.
    \item Training the student on synthetic samples from $P^{\mathrm{t}}$ yields similar dynamics to training on real data samples from $P_{X_1}$.
    \item If the student converges to the teacher, then $P^{\mathrm{s}}_M \approx P^{\mathrm{t}}$, allowing us to estimate the teacher's loss $\E_{P_{X_1}}[\log 1/P^{\mathrm{t}}(X)]$ by the final student's loss $\E_{P_{X_1}}[\log 1/P^{\mathrm{s}}_M(X)]$.
\end{enumerate}

Applying these approximations, the requential code length with a static teacher becomes
\begin{equation}
    |\mathrm{P}_{\mathrm{req}}| \,\approx \sum_{i=0}^{M-1} \E_{P_{X_1}}\left[\log \frac{1}{P^{\mathrm{s}}_i(X)} - \log \frac{1}{P^{\mathrm{s}}_M(X)}\right],
\end{equation}
which, when estimated empirically on real training data $Z_0, \ldots, Z_{M-1} \sim P_{X_1}$, recovers precisely the prequential estimate from \Cref{eq:preq}:
\begin{equation}
    |\mathrm{P}_{\mathrm{preq}}| \,\approx \sum_{i=0}^{M-1} \left(\log \frac{1}{P_i(Z_i)} - \log \frac{1}{P_M(Z_i)}\right).
\end{equation}
This connection also lends some justification to treating $6ND$ as the decoding time for the model in prequential coding, as it relates to a requential scheme that achieves this runtime.
Since a static teacher is generally suboptimal compared to the time-varying teachers used in full requential coding, which can remain close to the student throughout training while still guiding it toward the target distribution, we expect the prequential estimate to be an overestimate of the requential code length. This is consistent with the empirical observations in \Cref{fig:req_vs_preq}, where the prequential estimate is typically several times larger than the requential estimate.

\subsection{A Solvable Model Using Scaling Laws} \label{app:scaling-law-model}

In this section, we present a simplified analytical model from combining neural scaling laws with prequential coding to gain insight into how epiplexity and compute-optimal hyperparameters typically vary with compute and dataset size, along with their asymptotic behaviors.

We adopt a standard scaling law for the loss as a function of model size $N$ and training tokens $D$:
\begin{align}
    \L(N,D) = E + \qty(\frac{N_0}{N})^{\alpha} + \qty(\frac{D_0}{D})^{\beta},
\end{align}
where $E$ is the irreducible loss, $N_0$ and $D_0$ are scaling constants, and $0 < \alpha, \beta < 1$ are the scaling exponents. The total compute for training and evaluating on $\mathcal{D}$ test tokens is $T = 6ND + 2N\mathcal{D} = 2N(3D + \mathcal{D})$.

To simplify the analysis, we work in natural units: $n = N/N_0$, $d = D/D_0$, $\delta = \mathcal{D}/D_0$, and $t = T/(2N_0 D_0)$. The loss becomes $\L(n,d) = E + n^{-\alpha} + d^{-\beta}$, and the compute constraint simplifies to $t = n(3d + \delta)$.

\paragraph{Two-part code length.} The two-part code $\mathrm{P}_{\mathrm{tot}}$ consists of the model description and the data encoded using the model. The data code length on the test set is $\delta D_0 \cdot \L(n,d)$.

For the model description length, we use the prequential estimate from \Cref{eq:preq}, which corresponds to the area under the loss curve above the final loss\footnote{We start the sum and integral at $1$ to avoid the singularity at $0,$ which is an artifact of the scaling law as it typically only holds for large $D.$}:
\begin{align}
    |\mathrm{P}_{\mathrm{preq}}| &\,= \sum_{i=1}^{D} \qty[ \qty(\frac{i}{D_0})^{-\beta} - \qty(\frac{D}{D_0})^{-\beta} ] \nonumber\\
    &= \int_1^D \qty[ \qty(\frac{u}{D_0})^{-\beta} - \qty(\frac{D}{D_0})^{-\beta} ] du + O(1),
\end{align}
where the $O(1)$ term remains bounded as $D \to \infty$. Evaluating the integral and dropping $O(1)$ terms, we obtain the expression valid for large $D$:
\begin{align}
    |\mathrm{P}_{\mathrm{preq}}| \,= \frac{\beta}{1-\beta} D_0 \, d^{1-\beta}.
\end{align}

\paragraph{Optimality condition.} Dropping the constant term $\delta D_0 E$ from the two-part code length and dividing by $D_0$, we seek to minimize
\begin{align}
    f(n, d) = \frac{\beta}{1-\beta} d^{1-\beta} + \delta (n^{-\alpha} + d^{-\beta})
\end{align}
subject to $t = n(3d + \delta)$.

\paragraph{Solution.} Eliminating $n$ using the constraint $n = t/(3d + \delta)$, we obtain a one-dimensional optimization problem in $d$. Setting the derivative to zero and simplifying, the optimal $d^\star(t)$ satisfies
\begin{align} \label{eq:scaling-foc}
    \beta d^{-\beta-1} (\delta - d) = 3\alpha \delta \, t^{-\alpha} (3d + \delta)^{\alpha-1},
\end{align}
with the corresponding optimal model size given by
\begin{align}
    n^\star(t) = \frac{t}{3d^\star(t) + \delta}.
\end{align}

While \Cref{eq:scaling-foc} does not admit a simple closed-form solution in general, we can extract the asymptotic behavior in the large- and small-compute regimes.

\paragraph{Large-compute regime ($t \to \infty$).} As $t$ grows, the right-hand side of \Cref{eq:scaling-foc} scales as $t^{-\alpha} \to 0$. For the equation to remain balanced, we require $\delta - d \to 0$, i.e., $d^\star(t) \to \delta$. The leading-order scaling is therefore:
\begin{align}
    d^\star(t) = \delta - \Theta(t^{-\alpha}), \qquad n^\star(t) \sim \frac{t}{4\delta}.
\end{align}
In this regime, the optimal training set size saturates at the test set size $\delta$, while the model size grows linearly with compute. Correspondingly, the epiplexity saturates to
\begin{align}
    \rS_\infty(X) = \frac{\beta}{1-\beta} D_0 \, \delta^{1-\beta} = \frac{\beta}{1-\beta} D_0^{\beta} \mathcal{D}^{1-\beta}.
\end{align}
For the entropy, we have $(n^\star)^{-\alpha} \to 0$ while $(d^\star)^{-\beta} \to \delta^{-\beta}$, so
\begin{align}
    \rH_\infty(X) = \mathcal{D} \qty( E + \delta^{-\beta} ) = \mathcal{D} E + D_0^{\beta} \mathcal{D}^{1-\beta}.
\end{align}
The entropy approaches the irreducible entropy $\mathcal{D} E$ plus a residual term from finite training data that scales sublinearly with the test set size, meaning that the achieved per-token loss is $E + O(\D^{-\beta}).$

\paragraph{Small-compute regime ($d^\star \ll \delta$).} When compute is limited such that $d \ll \delta$, we approximate $\delta - d \approx \delta$ and $(3d + \delta)^{\alpha-1} \approx \delta^{\alpha-1}$. Substituting into \Cref{eq:scaling-foc} and solving for $d$ gives
\begin{align}
    d^\star(t) = \qty( \frac{\beta}{3\alpha} )^{\frac{1}{\beta+1}} t^{\frac{\alpha}{\beta+1}} \delta^{\frac{1-\alpha}{\beta+1}}.
\end{align}
Since $3d^\star \ll \delta$ in this regime, the optimal model size is
\begin{align}
    n^\star(t) \approx \frac{t}{\delta}.
\end{align}
Here, the model size is constrained by the need to evaluate on $\delta$ tokens, and the optimal training set size grows sublinearly with compute as $d^\star \propto t^{\alpha/(\beta+1)}$. The epiplexity in this regime scales as
\begin{align}
    \rS_T(X) = \frac{\beta}{1-\beta} D_0 \, (d^\star)^{1-\beta} \propto T^{\frac{\alpha(1-\beta)}{\beta+1}},
\end{align}
growing sublinearly with compute.

For the entropy, both the model and data contributions are significant. The model contribution scales as
\begin{align}
    \mathcal{D} (n^\star)^{-\alpha} = \mathcal{D} \qty(\frac{\delta}{t})^{\alpha} \propto T^{-\alpha},
\end{align}
while the data contribution scales as
\begin{align}
    \mathcal{D} (d^\star)^{-\beta} \propto T^{-\frac{\alpha\beta}{\beta+1}}.
\end{align}
Since $\alpha\beta/(\beta+1) < \alpha$, the data term decays more slowly and dominates for larger $t$ within this regime. The entropy above the irreducible level is thus
\begin{align}
    \rH_T(X) - \mathcal{D} E \propto T^{-\frac{\alpha\beta}{\beta+1}},
\end{align}
decaying as a power law with compute.

For typical scaling exponents (e.g., $\alpha \approx 0.34$ and $\beta \approx 0.28$ from \citet{hoffmann2022training}), the epiplexity grows as $\rS_T \propto T^{0.19}$ and the entropy decays as $\rH_T - \mathcal{D} E \propto T^{-0.07}$ in the small-compute regime.

\subsection{How Epiplexity and Time-Bounded Entropy Scale with Compute and Dataset Size} \label{app:epi-scaling}
In this section, we analyze how epiplexity and time-bounded entropy scale with compute budget and dataset size under natural assumptions about neural network training, without relying on specific functional forms for scaling laws. The goal is to provide some general intuitions for how these quantities are expected to vary as a function of the compute budget and dataset size.
\Cref{app:scaling-law-model} explicitly demonstrates using scaling laws and prequential coding that (1) epiplexity grows with both compute and dataset size, and (2) for a fixed $X$, epiplexity saturates to a finite value in the limit of infinite compute—specifically, to the amount of information acquired by an arbitrarily large model trained on a training set of the same size as the test set $X$, while time-bounded entropy decays to the loss achievable by an infinitely large model on this training set. Here, we show that similar or weaker statements hold more generally, requiring only a few natural assumptions about the effect of increasing model size $N$ and training data size $D$. These assumptions capture typically observed regularities in deep learning, such as the smoothly diminishing returns in scaling only model size while holding training set size fixed, but they may fail to capture rare exceptions like grokking and sudden improvement in performance above certain compute thresholds (as in \Cref{sec:emergent}).

Denote the code length for an $N$-parameter model trained on $D$ tokens as $|\mathrm{P}|(N,D),$ the per-token loss it achieves as $\L(N,D)$, the compute-optimal model size as $N^\star(T)$ and training data size as $D^\star(T),$ so that $\rS_T(X) = |\mathrm{P}|\qty(N^\star(T), D^\star(T))$ and $\rH_T(X) = \D\, \L\qty(N^\star(T), D^\star(T)).$ We establish the following results as we vary $T$ and $\D = |X|$, fixing the distribution of $X_i$ (only the dataset size changes):
\begin{itemize}
    \item \textbf{Monotonicity of $N^\star(T)$, $D^\star(T)$, $\rS_T(X)$, and $\rH_T(X)$} (\Cref{app:monotonicity-in-T}): Under natural assumptions on the effect of increasing $N$ and $D$, the compute-optimal model size $N^\star(T)$ and training data size $D^\star(T)$ are both increasing in the compute budget $T$. As a result, epiplexity typically grows with $T$ while time-bounded entropy typically decreases with $T$.

    \item \textbf{Monotonicity of $\rS_\infty(X)$ and $\rH_\infty(X)$ in $\D$} (\Cref{app:monotonicity-in-D}): In the infinite-compute limit, epiplexity $\rS_\infty(X)$ is nondecreasing in $\D = |X|$, while the per-token time-bounded entropy $h_\infty(X_\D) := \rH_\infty(X_\D) / \D$ is nonincreasing in $\D$.

    \item \textbf{$D^\star(T)$ generally approaches $\D$ in prequential coding} (\Cref{app:prequential-saturation}): For prequential coding, the compute-optimal training set size satisfies $D^\star(T) \to \D$ as $T \to \infty$, where $\D$ is the test set size, without assuming the scaling law form. Combined with monotonicity of $D^\star(T)$, this implies $D^\star(T) \uparrow \D$ from below.
\end{itemize}

\subsubsection{Monotonicity of $N^*(T)$, $D^*(T)$, $\rS_T(X)$, and $\rH_T(X)$}
\label{app:monotonicity-in-T}

The following theorem shows that the compute-optimal model size and training data size are both monotonically increasing in the compute budget under natural assumptions.

\begin{theorem}[Monotone growth of compute-optimal $N$ and $D$]
\label{thm:ND-monotone-logconvex}
Define the effective data $\widetilde{D} = 6D + 2\D$, so that the compute constraint becomes $T = N\widetilde{D}$. Let $J(N, \widetilde{D})$ denote the two-part code length as a function of model size $N$ and effective data $\widetilde{D}$, and assume $J$ is twice continuously differentiable. Consider the constrained MDL problem
\begin{align}
    \min_{N > 0,\, \widetilde{D} \ge 2\D}\; J(N, \widetilde{D})
    \qquad \text{s.t.} \qquad
    N\widetilde{D} = T.
\end{align}
Assume that for each $T$ in the regime of interest there is a unique interior optimizer $(N^\star(T), \widetilde{D}^\star(T))$ with $\widetilde{D}^\star(T) > 2\D$ and $N^\star(T)\widetilde{D}^\star(T) = T$.

Work in logarithmic coordinates $\mu := \log N$ and $\nu := \log \widetilde{D}$, and by slight abuse of notation write $J(\mu, \nu) = J(e^\mu, e^\nu)$. Assume that for all such $T$, the following conditions hold at the corresponding optimum $(\mu^\star(T), \nu^\star(T))$:
\begin{enumerate}
    \item \textbf{Complementarity (larger models are more sample-efficient):}
    \begin{align}
        \frac{\partial^2 J}{\partial \mu \partial \nu} \le 0.
    \end{align}

    \item \textbf{Diminishing returns in model size (in log coordinates):}
    \begin{align}
        \frac{\partial^2 J}{\partial \mu^2} > 0.
    \end{align}

    \item \textbf{Diminishing returns in effective data (in log coordinates):}
    \begin{align}
        \frac{\partial^2 J}{\partial \nu^2} > 0.
    \end{align}
\end{enumerate}

Then both compute-optimal choices are strictly increasing functions of $T$:
\begin{align}
    T_2 > T_1
    \quad \Longrightarrow \quad
    N^\star(T_2) > N^\star(T_1)
    \quad \text{and} \quad
    \widetilde{D}^\star(T_2) > \widetilde{D}^\star(T_1).
\end{align}
\end{theorem}

\begin{proof}
Work in logarithmic coordinates
\begin{align}
    \mu := \log N, \qquad \nu := \log \widetilde{D}, \qquad \tau := \log T.
\end{align}
The compute constraint $N\widetilde{D} = T$ becomes the affine constraint
\begin{align}
    \mu + \nu = \tau
    \qquad \Longleftrightarrow \qquad
    \nu = \tau - \mu.
    \label{eq:log-constraint-mu-nu}
\end{align}
By slight abuse of notation, write $J(\mu, \nu) := J(e^\mu, e^\nu)$ and denote its partial derivatives by $J_\mu, J_\nu, J_{\mu\mu}, J_{\nu\nu}, J_{\mu\nu}$, etc., all taken with respect to the log-coordinates $(\mu, \nu)$.

Define the \emph{restricted objective} along the compute frontier by
\begin{align}
    f(\mu, \tau) := J(\mu, \tau - \mu).
\end{align}
For each $\tau$ in the regime of interest, let $\mu^\star(\tau)$ denote the unique interior minimizer of $f(\cdot, \tau)$, and set $\nu^\star(\tau) := \tau - \mu^\star(\tau)$.

Holding $\tau$ fixed and differentiating $f$ with respect to $\mu$ gives
\begin{align}
    f_\mu(\mu, \tau)
    &= \frac{\partial}{\partial\mu} J(\mu, \tau - \mu) \nonumber\\
    &= J_\mu(\mu, \nu) + J_\nu(\mu, \nu) \frac{\partial}{\partial\mu}(\tau - \mu) \nonumber\\
    &= J_\mu(\mu, \nu) - J_\nu(\mu, \nu),
    \label{eq:fmu}
\end{align}
where $\nu = \tau - \mu$. The optimality condition for $\mu^\star(\tau)$ is therefore
\begin{align}
    f_\mu(\mu^\star(\tau), \tau) = 0
    \qquad \Longleftrightarrow \qquad
    J_\mu(\mu^\star(\tau), \nu^\star(\tau)) = J_\nu(\mu^\star(\tau), \nu^\star(\tau)).
    \label{eq:restricted-FOC}
\end{align}

Differentiating the identity $f_\mu(\mu^\star(\tau), \tau) = 0$ with respect to $\tau$ yields
\begin{align}
    0 = \frac{d}{d\tau} f_\mu(\mu^\star(\tau), \tau)
    = f_{\mu\mu}(\mu^\star(\tau), \tau) \, \frac{d\mu^\star}{d\tau}
      + f_{\mu\tau}(\mu^\star(\tau), \tau).
    \label{eq:IFT-start}
\end{align}
Assuming $f_{\mu\mu}(\mu^\star(\tau), \tau) \neq 0$ (verified below), we obtain
\begin{align}
    \frac{d\mu^\star}{d\tau}
    = -\frac{f_{\mu\tau}}{f_{\mu\mu}}
    \quad \text{evaluated at } (\mu, \tau) = (\mu^\star(\tau), \tau).
    \label{eq:IFT-compact}
\end{align}

We now express $f_{\mu\tau}$ and $f_{\mu\mu}$ in terms of second partial derivatives of $J$. From \eqref{eq:fmu} and the chain rule, using $\partial_\tau(\tau - \mu) = 1$,
\begin{align}
    f_{\mu\tau}(\mu, \tau)
    &= \frac{\partial}{\partial\tau} \qty( J_\mu(\mu, \nu) - J_\nu(\mu, \nu) ) \nonumber\\
    &= J_{\mu\nu}(\mu, \nu) \frac{\partial\nu}{\partial\tau}
       - J_{\nu\nu}(\mu, \nu) \frac{\partial\nu}{\partial\tau} \nonumber\\
    &= J_{\mu\nu}(\mu, \nu) - J_{\nu\nu}(\mu, \nu),
    \label{eq:fmutau}
\end{align}
with $\nu = \tau - \mu$. Similarly, differentiating \eqref{eq:fmu} with respect to $\mu$ while holding $\tau$ fixed, and using $\partial_\mu(\tau - \mu) = -1$ together with symmetry $J_{\nu\mu} = J_{\mu\nu}$, yields
\begin{align}
    f_{\mu\mu}(\mu, \tau)
    &= \frac{\partial}{\partial\mu} \qty( J_\mu(\mu, \nu) - J_\nu(\mu, \nu) ) \nonumber\\
    &= \qty( J_{\mu\mu}(\mu, \nu) + J_{\mu\nu}(\mu, \nu) \frac{\partial\nu}{\partial\mu} )
       - \qty( J_{\nu\mu}(\mu, \nu) + J_{\nu\nu}(\mu, \nu) \frac{\partial\nu}{\partial\mu} ) \nonumber\\
    &= (J_{\mu\mu} - J_{\mu\nu}) - (J_{\mu\nu} - J_{\nu\nu}) \nonumber\\
    &= J_{\mu\mu}(\mu, \nu) + J_{\nu\nu}(\mu, \nu) - 2J_{\mu\nu}(\mu, \nu).
    \label{eq:fmumu}
\end{align}
Substituting \eqref{eq:fmutau}--\eqref{eq:fmumu} into \eqref{eq:IFT-compact} gives
\begin{align}
    \frac{d\mu^\star}{d\tau}
    = -\frac{J_{\mu\nu} - J_{\nu\nu}}{J_{\mu\mu} + J_{\nu\nu} - 2J_{\mu\nu}}
    = \frac{J_{\nu\nu} - J_{\mu\nu}}{J_{\mu\mu} + J_{\nu\nu} - 2J_{\mu\nu}},
    \label{eq:dmu-dtau-clean}
\end{align}
with all second partial derivatives of $J$ evaluated at $(\mu, \nu) = (\mu^\star(\tau), \nu^\star(\tau))$.

By the assumptions $J_{\nu\nu} > 0$ and $J_{\mu\nu} \le 0$ at the optimum, the numerator in \eqref{eq:dmu-dtau-clean} satisfies $J_{\nu\nu} - J_{\mu\nu} > 0$. By the assumptions $J_{\mu\mu} > 0$, $J_{\nu\nu} > 0$, and $J_{\mu\nu} \le 0$, the denominator satisfies $J_{\mu\mu} + J_{\nu\nu} - 2J_{\mu\nu} > 0$. Hence
\begin{align}
    \frac{d\mu^\star}{d\tau} > 0.
\end{align}

Since $\nu^\star(\tau) = \tau - \mu^\star(\tau)$, we also have
\begin{align}
    \frac{d\nu^\star}{d\tau}
    = 1 - \frac{d\mu^\star}{d\tau}
    = \frac{J_{\mu\mu} - J_{\mu\nu}}{J_{\mu\mu} + J_{\nu\nu} - 2J_{\mu\nu}}
    > 0,
    \label{eq:dnu-dtau-clean}
\end{align}
where positivity follows from $J_{\mu\mu} > 0$ and $J_{\mu\nu} \le 0$ together with the same positive denominator.

Finally, $N^\star(T) = \exp(\mu^\star(\log T))$ and $\widetilde{D}^\star(T) = \exp(\nu^\star(\log T))$, so $d\mu^\star/d\tau > 0$ and $d\nu^\star/d\tau > 0$ imply that both $N^\star(T)$ and $\widetilde{D}^\star(T)$ are strictly increasing functions of $T$.
\end{proof}

\paragraph{Empirical plausibility of the assumptions.}
The three conditions in \Cref{thm:ND-monotone-logconvex} reflect well-documented empirical phenomena in deep learning. The complementarity condition $\partial^2 J / \partial \mu \partial \nu \le 0$ captures the observation that larger models are more sample-efficient: increasing model size leads to faster learning \citep{kaplan2020scaling,yang2022tensor}, which leads to a faster decrease in both the model description length and data code length (final loss), and thus $\partial J / \partial\nu$ should decrease with $\mu$. The diminishing returns conditions $\partial^2 J / \partial \mu^2 > 0$ and $\partial^2 J / \partial \nu^2 > 0$ simply state that there is diminishing return in successive doubling of the model size or training data size, holding the other quantity fixed.

\paragraph{Asymptotic growth of $\rS_T$ and monotone decay of $\rH_T$.}
The monotone growth of the compute-optimal $N^\star(T)$ and $D^\star(T)$ does not by itself imply that $\rS_T(X) := |\mathrm{P}|\qty(N^\star(T), D^\star(T))$ is monotone for all $T$. Intuitively, while we expect the model description length $|\mathrm{P}|(N, D)$ to grow with $D$, it need not increase with $N$: larger models can be more sample-efficient, which may reduce the effective complexity of the learned predictor under some coding schemes. However, one should still expect $\rS_T(X)$ to grow with $T$, at least asymptotically, if we assume (1) the compute-optimal model size diverges while the optimal training horizon converges, as in the scaling-law model of \Cref{app:scaling-law-model}, and (2) the existence of infinite-model-size limits of the training dynamics.

That is, assume that along the compute-optimal path,
\begin{align}
    N^\star(T) \to \infty
    \qquad \text{and} \qquad
    D^\star(T) \to D_\infty < \infty
    \qquad \text{as } T \to \infty.
\end{align}
Assume moreover that for bounded training horizons, the description length admits a well-defined infinite-model-size limit: there exists a function $|\mathrm{P}|_\infty(D)$ such that for each fixed $D$,
\begin{align}
    |\mathrm{P}|(N, D) \to |\mathrm{P}|_\infty(D)
    \qquad \text{as } N \to \infty.
\end{align}
This assumption is motivated by the existence of infinite-width and depth limits of neural networks under appropriate parameterizations \citep{yang2023tensor,dey2025don}, where scalar quantities such as loss and teacher--student KL divergence that determine $|\mathrm{P}|(N, D)$ admit stable large-model limits for bounded training durations. Under these conditions, any non-monotonic dependence of $|\mathrm{P}|$ on $N$ is a finite-model effect; once $N^\star(T)$ is large enough, $|\mathrm{P}|\qty(N^\star(T), D^\star(T))$ is well-approximated by the limiting curve $|\mathrm{P}|_\infty\qty(D^\star(T))$. Since $D^\star(T)$ is monotone increasing and convergent under our earlier assumptions, the large-$T$ behavior of $\rS_T(X)$ is therefore governed primarily by the behavior of $D^\star(T)$ alone, which we have shown is increasing with $T$, so one expects $\rS_T(X)$ to increase at large $T$ as $|\mathrm{P}|_\infty(D)$ should increase with $D.$

For the entropy term $\rH_T(X) := \D\, \L\qty(N^\star(T), D^\star(T))$, the conclusion is simpler and does not require taking $N \to \infty$. Assume only that the loss $\L(N, D)$ is nonincreasing in both $N$ and $D$ (more data and parameters cannot make the loss worse). Since $N^\star(T)$ and $D^\star(T)$ are increasing in $T$, we have
\begin{align}
    T_2 > T_1
    \quad \Longrightarrow \quad
    \L\qty(N^\star(T_2), D^\star(T_2))
    \le
    \L\qty(N^\star(T_1), D^\star(T_1)),
\end{align}
and therefore $\rH_T(X)$ is nonincreasing in $T$. In particular, whenever $\rH_T(X)$ has a finite large-compute limit $\rH_\infty(X)$, it approaches this limit from above.

\subsubsection{Monotonicity of $\rS_\infty(X)$ and $\rH_\infty(X)$ in $\D$}
\label{app:monotonicity-in-D}

We now show that epiplexity and time-bounded entropy (after appropriate normalization) in the infinite-compute limit are monotonic in the test set size $\D = |X|$, regardless of the coding scheme.

Fix a dataset $X_{\D}$ of length $\D$ tokens. For a two-part code of the form
\begin{align}
    J(N, D; \D) = |\mathrm{P}|(N, D) + \D\, \L(N, D),
    \label{eq:J-P-plus-DL}
\end{align}
let $(N^\star_T(\D), D^\star_T(\D))$ denote the compute-optimal choices. We write
\begin{align}
    \rS_T(X_{\D}) &:= |\mathrm{P}|\qty(N^\star_T(\D), D^\star_T(\D)), \\
    \rH_T(X_{\D}) &:= \D\, \L\qty(N^\star_T(\D), D^\star_T(\D)), \\
    h_T(\D) &:= \frac{H_T(X_{\D})}{\D} = \L\qty(N^\star_T(\D), D^\star_T(\D)).
\end{align}

In the infinite-compute limit $T \to \infty$, the compute constraint becomes irrelevant, so the limiting quantities coincide with the optimum of the unconstrained problem
\begin{align}
    (N_\infty^\star(\D), D_\infty^\star(\D)) = \arg\min_{N > 0,\, D \ge 0}\; |\mathrm{P}|(N, D) + \D\, \L(N, D).
    \label{eq:infty-unconstrained}
\end{align}
Thus
\begin{align}
    S_\infty(X_{\D}) = |\mathrm{P}|(N_\infty^\star, D_\infty^\star), \qquad h_\infty(X_\D) = \L(N_\infty^\star, D_\infty^\star).
\end{align}

We claim that $\rS_\infty(X_{\D})$ is nondecreasing in $\D$, and $h_\infty(X_\D)$ is nonincreasing in $\D$, assuming that for each $\D > 0$ the unconstrained problem \eqref{eq:infty-unconstrained} admits at least one minimizer.

To see this, fix $\D_2 > \D_1$ and let $(N_i, D_i)$ be minimizers of \eqref{eq:infty-unconstrained} at $\D = \D_i$. Write $P_i := |\mathrm{P}|(N_i, D_i)$ and $L_i := \L(N_i, D_i)$. Optimality of $(N_2, D_2)$ at $\D_2$ implies
\begin{align}
    P_2 + \D_2 L_2 \le P_1 + \D_2 L_1.
    \label{eq:ineq-D2}
\end{align}
Optimality of $(N_1, D_1)$ at $\D_1$ implies
\begin{align}
    P_1 + \D_1 L_1 \le P_2 + \D_1 L_2.
    \label{eq:ineq-D1}
\end{align}
Adding \eqref{eq:ineq-D2} and \eqref{eq:ineq-D1} gives
\begin{align}
    (P_2 + \D_2 L_2) + (P_1 + \D_1 L_1) &\le (P_1 + \D_2 L_1) + (P_2 + \D_1 L_2) \nonumber\\
    \D_2 L_2 + \D_1 L_1 &\le \D_2 L_1 + \D_1 L_2 \nonumber\\
    (\D_2 - \D_1)(L_2 - L_1) &\le 0,
\end{align}
hence $L_2 \le L_1$ since $\D_2 > \D_1$, i.e., the achieved loss $h_\infty(X_\D)$ is nonincreasing in $\D$. Substituting $L_2 \le L_1$ back into \eqref{eq:ineq-D1} yields $P_2 \ge P_1$, i.e., $\rS_\infty(X_{\D})$ is nondecreasing in $\D$.

\subsubsection{$D^\star(T)$ Generally Approaches $\D$ in Prequential Coding}
\label{app:prequential-saturation}

We now show that the compute-optimal training set size for prequential coding generically saturates at $D = \D$ as $T \to \infty$, without assuming specific scaling laws.

In continuous time, the prequential model description length is the area above the final loss:
\begin{align}
    |\mathrm{P}_{\mathrm{preq}}(N, D)|\, := \int_{0}^{D} \qty( \L(N, u) - \L(N, D) )\, du.
\end{align}
The corresponding two-part code length for a test set of size $\D$ is
\begin{align}
    J_{\mathrm{preq}}(N, D; \D)
    &= |\mathrm{P}_{\mathrm{preq}}(N, D)|\, + \D\, \L(N, D) \nonumber\\
    &= \int_{0}^{D} \L(N, u)\, du + (\D - D)\, \L(N, D).
    \label{eq:Jpreq-cont}
\end{align}
We express $N$ in terms of $D$ for fixed $T$ using the constraint $6ND + 2N\D = T$:
\begin{align}
    N_T(D) = \frac{T}{6D + 2\D}.
    \label{eq:NT-of-D}
\end{align}

\paragraph{Large-compute limit.}
Assume: (i) $\L(N, D)$ is nonincreasing in $N$ and admits a pointwise infinite-model-size limit $\L_\infty(D) := \lim_{N \to \infty} \L(N, D)$;\footnote{This limit provably exists under $\mu$P, but is a reasonable assumption in general as it simply asserts diminishing returns in scaling model size without increasing data.} (ii) $\L_\infty$ is continuously differentiable and strictly decreasing, i.e., $\L_\infty'(D) < 0$. Along the compute frontier \eqref{eq:NT-of-D}, for any fixed $D$ we have $N_T(D) \to \infty$ as $T \to \infty$, hence
\begin{align}
    J_{\mathrm{preq}}(N_T(D), D; \D) \to J_\infty(D) := \int_{0}^{D} \L_\infty(u)\, du + (\D - D)\, \L_\infty(D).
\end{align}
Differentiating gives
\begin{align}
    J_\infty'(D) = (\D - D)\, \L_\infty'(D).
\end{align}
Since $\L_\infty'(D) < 0$, we have $J_\infty'(D) < 0$ for $D < \D$ and $J_\infty'(D) > 0$ for $D > \D$. Thus $J_\infty$ is uniquely minimized at $D = \D$, implying
\begin{align}
    D^\star(T) \to \D \qquad \text{as } T \to \infty.
\end{align}

\paragraph{Approach from below and linear growth of $N^\star(T)$.}
By \Cref{thm:ND-monotone-logconvex}, under the complementarity and diminishing-returns assumptions, the compute-optimal training set size $D^\star(T)$ is strictly increasing in $T$. Combined with the convergence $D^\star(T) \to \D$, this yields $D^\star(T) \uparrow \D$, i.e., $D^\star(T)$ approaches $\D$ from below. Finally, since $N^\star(T) = N_T(D^\star(T))$,
\begin{align}
    N^\star(T) = \frac{T}{6D^\star(T) + 2\D} \sim \frac{T}{8\D},
\end{align}
so the compute-optimal model size grows linearly with $T$ in the large-compute regime.

\section{Experiment Details}\label{app:experiment_details}
Unless otherwise stated, we use the GPT-2 \citep{radford2019language} transformer architecture trained with Adam optimizer. In experiments where we vary the model size, we tune the base learning rate on a small model and transfer it to larger models using using $\mu$P \citep{yang2022tensor} and CompleteP \citep{dey2025don}. In \mup, the per-layer learning rate is base learning rate divided by the input dimension, so our reported base learning rate is larger than typical learning rates used for Adam. The hyperparameters presented below are shared between the teacher and the student for requential coding (width, depth, learning rate, EMA time scale, etc.). As described in \Cref{app:detail-procedure}, the EMA for the teacher is used only for producing the distillation target and does not alter the raw teacher training dynamics, while the EMA for the student model does alter its training dynamics and is used to replace a decaying learning rate schedule. 

\subsection{ECA}\label{app:eca}
In \Cref{fig:eca}, we train the transformer to predict $Y$ given $X$ where $X$ is the initial state with a state size of 64 cells and $Y$ is obtained by evolving $X$ for 48 steps. We apply a burnin period of 1000 steps for sampling the initial state $X$ to eliminate the less uninteresting transient dynamics from random initialization. That is $X$ is obtained by evolving the ECA on $Z$ for 1000 steps where $Z$ is a uniform random initial state. For each rule, we train models with width (embedding dimension) $\in \{16,32,64,128,256,512\}$ and depth (number of transformer blocks) $\in \{1,2,4,6,9\}$. We train both teacher and student using batches of 1536 sequences (each an $(X,Y)$ pair), a base learning rate of 0.03 with 100 warmup steps, and an EMA time scale of 50 steps (half-life divided by $\ln(2)$). We did not set a max teacher-student KL as the student smoothly trackes the teacher throughout training. The epiplexity and time-bounded entropy is estimated for a test set of size $\D=100$M tokens (counting $Y$ only).

\subsection{Easy induction}\label{app:easy-induct}
For this task, we use a sequence length of $n=512$ (as described in \Cref{sec:induction}). The model has 3 layers and a width of 128, and is trained with a learning rate of 0.03 and a batch size of 384 sequences for 3000 steps with 15 warmup steps and an EMA time scale of 50 steps. We found further increasing the model size led to negligible improvement in the loss, and \Cref{fig:induction_easy} shows that the model has nearly converged by the end of training to the theoretical minimum loss, so there is no need to further increase the training data. As a result, we expect the epiplexity $\rS_T(X)$ to stabilize as $T$ and $\D=|X|$ increases (in the relevant regime where $T$ is still much less than what is required for implementing the brute-force solution that enumerates all possible combinations of hidden entries in the transition matrix), and our estimated epiplexity approximates this stabilized value.

\subsection{Hard induction}\label{app:hard-induct}
We modify the ECA experiment in \Cref{app:eca} to remove the first $h \in \{0,1,\ldots,5\}$ bits in $X$ when fed to the model as input. We use a state size of 32, batch size of 1536 sequences, learning of 0.03, EMA time scale of 100 steps. We set the max KL threshold between the teacher and student as 0.03 (nats per token). To construct a forward function that is hard to invert, we use rule 30 iterated for 4 steps. We train models with 3 layers and width 256 for 20000. Further increasing model size or training data led to no improvement in the loss. As \Cref{fig:induction_hard} shows, the models converge by the end of training (the loss curves shown are for the student models, but the teacher models also converge) to the theoretical minimum values. Therefore, like the case for \Cref{app:easy-induct}, we expect the epiplexity $\rS_T(X)$ to stabilize as $T$ and $\D=|X|$ increases, at least in the relevant regime where $T$ is still much less than what is required for implementing the brute-force solution that enumerates all possible combinations of hidden bits, and our estimated epiplexity approximates this stabilized value.

\subsection{Chess}\label{app:chess}
We train models of varying sizes from 1M to 160M parameters with depth between $3$ and $24$ layers. The base learning rate is set to $2$ and the batch size is 256, with a sequence length of 512. We set the EMA time scale to 50 steps and max KL to 0.1 nats per token. We use character-level tokenization. The teacher models are trained for 5B tokens in total, and the student models are trained for slightly more due to hitting the max KL threshold during training. The test set size is set to $5$B tokens.

\paragraph{Pre-Training Data.} We use the Lichess dataset available on Hugging Face at \url{https://huggingface.co/datasets/Lichess/standard-chess-games} as pre-training data, formatted as either "<board>|<moves>" or "<moves>|<board>", where moves are in algebraic chess notation and board is the final board state in FEN notation. We use a slightly more concise version of the algebraic notation to further compress the move sequence. An example input where the board appears last is:
\begin{verbatim}
e4,e5;Nf3,Nc6;Bb5,a6;Ba4,Nf6;O-O,Be7;Re1,b5;Bb3,d6;c3,O-O;h3,Nb8;d4,Nbd7;
|r1bq1rk1/2pnbppp/p2p1n2/1p2p3/3PP3/1BP2N1P/PP3PP1/RNBQR1K1 w - - 0 10
\end{verbatim}

For downstream evaluation, we evaluate performance on the following two datasets after fine-tuning on $50$k examples for a 10M-parameter model with depth 24. We report accuracy under greedy decoding at zero temperature. 
\paragraph{Chess Puzzles.} We use puzzles from the Lichess puzzle database available at \url{https://huggingface.co/datasets/EleutherAI/lichess-puzzles}, filtering for puzzles with difficulty rating above 2000. The task is to predict the correct move sequence given the game context. Puzzles are formatted as move sequences where the model must predict the next optimal move, following \citep{burns2023weak}, with only the target moves included in the loss computation via masking. This tests the model's ability to recognize tactical patterns and calculate forced sequences.

\paragraph{Centipawn Evaluation.} We evaluate position understanding using the Lichess chess position evaluations dataset at \url{https://huggingface.co/datasets/Lichess/chess-position-evaluations}, where models classify positions into 9 evaluation buckets based on Stockfish centipawn (cp) scores: class 0 ($\leq -800$cp), class 1 ($-800$ to $-400$cp), class 2 ($-400$ to $-200$cp), class 3 ($-200$ to $-50$cp), class 4 ($-50$ to $+50$cp), class 5 ($+50$ to $+200$cp), class 6 ($+200$ to $+400$cp), class 7 ($+400$ to $+800$cp), and class 8 ($\geq +800$cp). Examples are formatted as "<board>|<class>" where the model predicts the evaluation class, with mate positions assigned to the extreme classes (0 or 8). Loss during fine-tuning is computed only for predicting the class.

\subsection{OpenWebText}\label{app:owt}
We use the OpenWebText dataset at \url{https://huggingface.co/datasets/Skylion007/openwebtext}, keeping only documents containing only 96 common alphanumeric symbols, and apply character-level tokenization. The setup is otherwise identical to the chess experiment (\Cref{app:chess}).

\subsection{CIFAR-5M}\label{app:c5m}
We use the CIFAR-5M dataset at \url{https://github.com/preetum/cifar5m}. We convert the $32 \times 32 \times 3$ images to greyscale and flatten to a 1D sequence of 1024 in raster-scan order. The vocabulary is the set of pixel intensities $\{0,\ldots,255\}$. The setup is otherwise identical to the chess experiment (\Cref{app:chess}).

\subsection{Prequential vs Requential Comparison} \label{app:preq-vs-req}
\paragraph{ECA.} The ECA experiment include rules $\{0, 32, 4, 15, 22, 30, 41, 54, 106, 110\},$ covering all 4 classes. We train models with width $\in \{16,32,64,128\}$ and depth $\in \{1,2,3\}$ up to 10000 steps. We use a base learning rate of 0.03 and batch size of 384. Other hyperparameters are identical to \Cref{app:eca}. We set $\D=250$M tokens. For each rule, we report the maximum epiplexity over the resulting compute range.

\paragraph{Induction.} Both the easy and hard induction results directly come from the experiments in \Cref{sec:induction}. As explained in \Cref{app:easy-induct} and \Cref{app:hard-induct}, the compute budget $T$ and test set size $\D$ need not be precisely specified for these two tasks as the epiplexity stabilizes as $T$ and $\D$ increase due to the convergent training dynamics.

\paragraph{Natural data.} We report the estimated epiplexity on each dataset at the maximum tested compute budget as described in \Cref{app:chess}, \Cref{app:owt}, and \Cref{app:c5m}.

\subsection{ECA Emergence} \label{app:eca-emergence}
We modify the setup in \Cref{app:eca} to include models that predict intermediate states and the final state rather than the final state directly. Let $X^{(0)}$ denote the initial ECA state, and $X^{(s)}$ denote it evolved for $s$ steps. For an $\ell$-loop model, we train the model to predict $(X^{(\Delta)}, X^{(2\Delta)}, \ldots, X^{(t)})$ instead of $X^{(t)}$ only, where $\Delta = t/\ell.$ Its marginal probability on the final state is lower-bounded by its joint probability on the ground truth trajectory: 
\begin{equation}
    P(X^{(t)}) = \sum_{X^{\prime(\Delta)},\ldots,X^{\prime(t-\Delta)}} P\!\big(X^{\prime(\Delta)},\ldots,X^{\prime(t-\Delta)},X^{(t)}\big)
\end{equation}
So we upper-bound its NLL as
\begin{align}
\log\frac{1}{P(X^{(t)})}
&\le
\log\frac{1}{P\!\big(X^{(\Delta)},\ldots,X^{(t)}\big)} \nonumber\\
&=
\sum_{k=1}^{\ell}
\log\frac{1}{P\!\big(X^{(k\Delta)} \mid X^{((k-1)\Delta)},\ldots,X^{(\Delta)}\big)},
\end{align}
We account for the intermediate tokens when computing the time bound and the code length (they contribute to the model code length as well as the data entropy code length).
In the experiment, we set the ECA steps to $t=64.$ We train models with width $\{16,32,64,128\},$ depth $\in \{1,2,4,8,16,32\},$ and number of loops $\ell \in \{1,2,4,8,16\}.$ We found $\ell \in \{2,4,8\}$ has no advantage over the non-looped model ($\ell=1$) in terms of the two-part code, only $\ell=16$ does. We therefore refer to $\ell=1$ as non-looped and $\ell=16$ as looped. The fact that a small $\ell > 1$ is not helpful is likely because the overhead of encoding and generating intermediate states exceeds the savings from only slightly simplifying each prediction step, as the per-step prediction horizon is still significant. We train all models with a base learning rate of 0.06, batch size of 147456 tokens, warmup of 100 steps, and EMA time scale of 50 steps. We did not set a max teacher-student KL. The test set size is set to $\D=100$M final state tokens.

\subsection{Scaling Laws}\label{app:scaling-law-exps}
We estimate epiplexity and time-bounded entropy using the expressions derived in \Cref{app:scaling-law-model} for prequential coding using existing scaling laws for $\L(N,D)$. We solve for the optimal training tokens $D^\star(T)$ as a function of compute using root finding for \Cref{eq:scaling-foc}. For language, we use the Chinchilla scaling laws from \citet{hoffmann2022training}, which were fit to total parameter counts. For all other modalities (images and video), we use the scaling laws from \citet{henighan2020scaling}, which follow the methodology of \citet{kaplan2020scaling} and report non-embedding parameter counts. We correct these to use total parameters following \citet{pearce2024reconciling}, as described below.

\paragraph{Correcting for embedding parameters.}
The scaling laws in \citet{kaplan2020scaling} and \citet{henighan2020scaling} are reported in terms of non-embedding parameters $N_{\backslash E}$ and non-embedding compute $C_{\backslash E}$, excluding embedding and unembedding parameters. As shown by \citet{pearce2024reconciling}, this choice---combined with smaller model scales---accounts for much of the discrepancy between the Kaplan and Chinchilla scaling laws. Following their approach, we relate total parameters $N$ to non-embedding parameters via
\begin{align}
    N = N_{\backslash E} + \omega N_{\backslash E}^{1/3}, \qquad \omega = (V + L_{\mathrm{ctx}}) \qty(\frac{A}{12})^{1/3},
\end{align}
where $V$ is the vocabulary size, $L_{\mathrm{ctx}}$ is the context length, and $A$ is the aspect ratio ($\mathrm{width} / \mathrm{depth})$. We use $A=5$ as \citet{henighan2020scaling} showed the optimal aspect ratio is around this value for non-language datasets. We generate points $(C_{\backslash E}, N_{\backslash E}, L)$ from the original scaling laws, convert to $(C, N, \L)$ using this relation (with total compute as $C = C_{\backslash E} \cdot N / N_{\backslash E}$), and refit the power-law exponents and the irreducible loss.

\paragraph{Parameterization conversion.}
The scaling laws in \citet{henighan2020scaling} are reported in compute-centric form, expressing the optimal loss $L^\star(C) = (C/C_0)^{-\gamma} + E$ and optimal model size $N^\star(C) = (C/\hat{C})^{\delta}$ as functions of compute budget $C$. We convert these to the $(N, D)$ parameterization used in this work:
\begin{align}
    \L(N, D) = \qty(\frac{N}{N_0})^{-\alpha} + \qty(\frac{D}{D_0})^{-\beta} + E,
\end{align}
where the exponents transform as $\alpha = \gamma / \delta$ and $\beta = \gamma / (1 - \delta)$, and the token scale is given by $D_0 = \frac{\hat{C}}{6} N_0^{\alpha/\beta} (\beta/\alpha)^{-1/\beta}$.

\paragraph{Corrected parameters.}
\Cref{tab:scaling-params} presents the corrected scaling law parameters used in our final calculations.

\begin{table}[H]
\centering
\caption{Final scaling law parameters used. Image and video domains from \citet{henighan2020scaling} are corrected for embedding parameters using aspect ratio $A = 5$ following \citep{pearce2024reconciling}; Chinchilla (language) from \citet{hoffmann2022training} was originally fit to total parameter counts and requires no correction. $D_0$ is measured in tokens and $E$ is measured in nats.}
\label{tab:scaling-params}
\begin{tabular}{lcccccc}
\toprule
Domain & $\alpha$ & $\beta$ & $N_0$ & $D_0$ & $E$ \\
\midrule
Image 8$\times$8 & 0.331 & 0.566 & $8.0 \times 10^{1}$ & $2.66 \times 10^{6}$ & 3.14 \\
Image 16$\times$16 & 0.307 & 0.820 & $2.8 \times 10^{2}$ & $8.94 \times 10^{7}$ & 2.68 \\
Image 32$\times$32 & 0.258 & 0.399 & $6.3 \times 10^{1}$ & $1.95 \times 10^{6}$ & 2.30 \\
Image VQ 16$\times$16 & 0.322 & 0.441 & $2.7 \times 10^{4}$ & $4.44 \times 10^{7}$ & 4.23 \\
Image VQ 32$\times$32 & 0.287 & 0.560 & $1.9 \times 10^{4}$ & $1.63 \times 10^{8}$ & 3.32 \\
Video VQ 16$^3$ & 0.428 & 0.718 & $3.7 \times 10^{4}$ & $1.79 \times 10^{8}$ & 1.15 \\
Language (Chinchilla) & 0.339 & 0.285 & $4.91 \times 10^{7}$ & $1.49 \times 10^{9}$ & 1.69 \\
\bottomrule
\end{tabular}
\end{table}

\section{RASP-L for Elementary Cellular Automata}
\label{app:rule30_rasp}
Below we provide RASP-L code \citep{zhou2023algorithms} demonstrating how the evolution rule of an ECA can be implemented, providing evidence that the solution can be expressed within an autoregressive transformer model.

\begin{lstlisting}[style=rasplstyle, caption={RASPL implementation of a cellular automaton evolution step}, label={lst:ca-evolve}]
from np_rasp import *

def int2bits(x, bits=8): # returns LSB first
    """ Helper function to generate fixed bitstring representing a number.
    Not RASP-L, can be assumed constant."""
    bits_str = bin(x)[2:].zfill(bits)
    return np.array(list(map(int,bits_str[::-1])),dtype=np.uint8)

sep = -1
sep2 = -2
def evolve_ca(x, rule):
    """ Function to autoregressively output produce the output of one step of the ECA rule. Problem encoded as x= --input state--,sep,sep2,--output state--.
    Rule: int (specifying the ECA)"""
    lookup = int2bits(rule, 8)
    in_input = 1 - has_seen(x, full(x, sep))
    in_input2 = 1 - has_seen(x, full(x, sep2))
    width = cumsum(in_input)  # only valid after sep
    idx = indices(x)
    circ_x = where(in_input, x, index_select(x, idx - width))
    prev = shift_right(x, 1)
    cprev = where(in_input2, prev, index_select(prev, idx - width))
    prev2 = shift_right(x, 2)
    nbhd = (prev2 << 2) + (cprev << 1) + circ_x
    shifted_nextstate = lookup[nbhd]
    to_select_idx = idx - width
    to_select_idx = where(to_select_idx < 3, idx, to_select_idx)
    outstate = index_select(shifted_nextstate, to_select_idx)
    return outstate
\end{lstlisting}

\section{Cellular Automata and Game of Life}\label{app:conway}
\textbf{Elementary cellular automata}
Elementary cellular automata (ECA) \citep{wolfram2003new} are one-dimensional cellular automata defined on a finite or infinite line of cells, each in one of two states: 0 or 1. The system evolves in discrete time steps according to local rules: a cell's next state depends only on its current state and those of its two immediate neighbors, yielding $2^3=8$ possible neighborhood configurations. Since each configuration can map to either 0 or 1, there are $2^8=256$ possible rules, conventionally numbered 0–255 using Wolfram's notation, where the rule number's binary representation specifies the output for each neighborhood. Despite their simplicity, ECAs exhibit diverse behaviors ranging from trivial (e.g., Rule 0) to complex and chaotic (e.g., Rule 30), with Rule 54 proven to be Turing-complete. These systems serve as minimal models for studying emergence, computation, and the relationship between local rules and global behavior.

\textbf{Conways Game of Life}
Conway's Game of Life \citep{gardner1970mathematical} is a cellular automaton defined on an infinite two-dimensional grid of cells, each in one of two states: alive (1) or dead (0). The system evolves in discrete time steps according to deterministic local rules: a cell's next state depends only on its current state and those of its eight neighbors. Specifically, a live cell survives if it has exactly 2 or 3 live neighbors (otherwise it dies), while a dead cell becomes alive if it has exactly 3 live neighbors (otherwise it remains dead). Despite the simplicity of these rules, the Game of Life exhibits remarkably complex emergent behavior, including stable structures (blocks), periodic oscillators (blinkers), mobile patterns (gliders), and structures that generate infinite streams of gliders (glider guns). The system also happens to be Turing-complete, with a specific initial configuration specifying the program, it is capable of universal computation.

\section{Emergence}\label{app:emergence}

\begin{figure}[t]
\centering
\begin{minipage}[t]{0.4\linewidth}
    \centering
    \includegraphics[width=\linewidth]{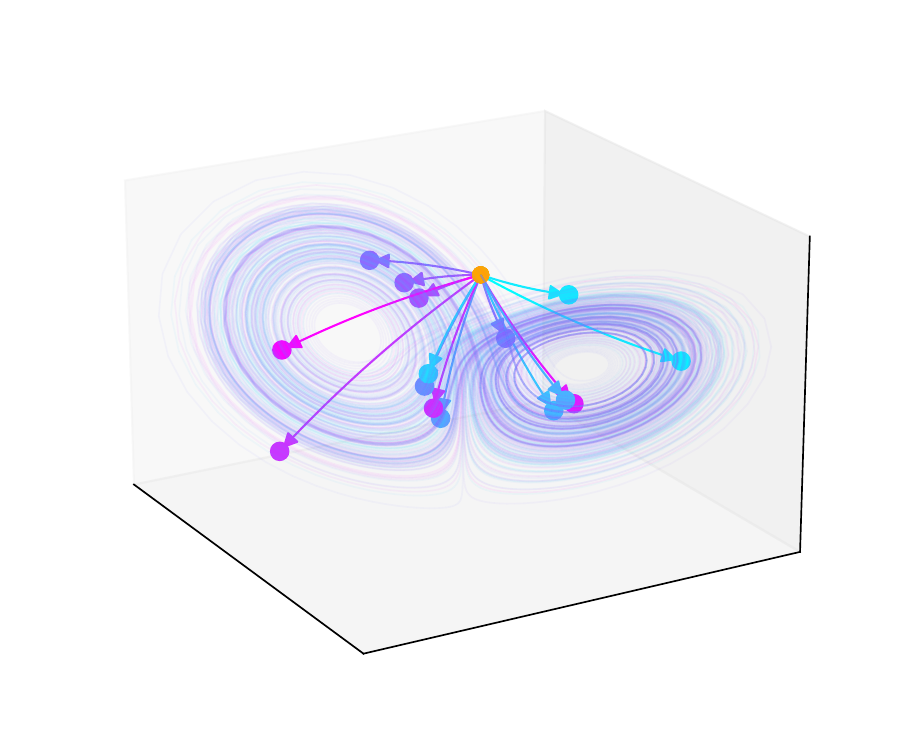}
    \subcaption{}
    \label{fig:lorenz_traj}
\end{minipage}%
\hfill
\begin{minipage}[t]{0.48\linewidth}
    \centering
    \includegraphics[width=\linewidth]{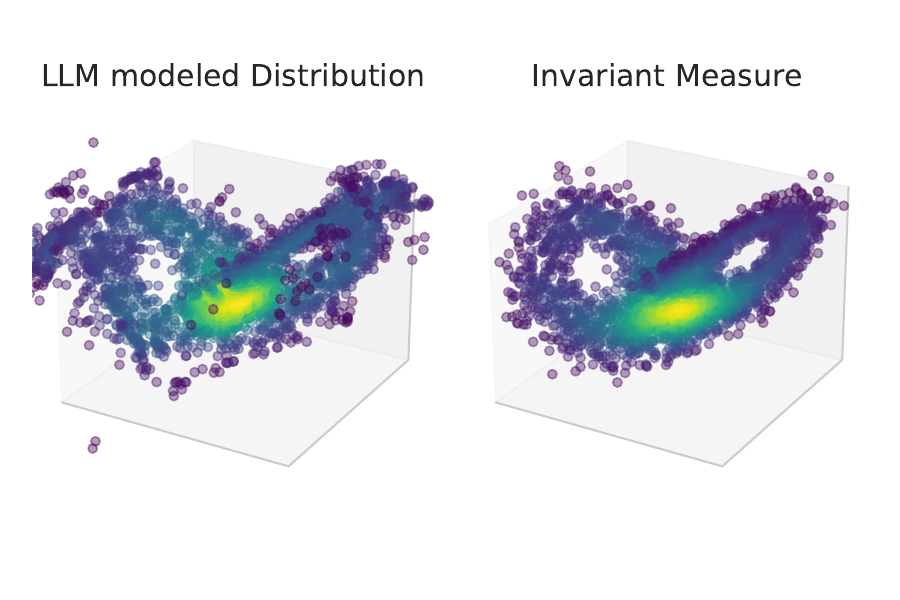}
    \subcaption{}
    \label{fig:lorenz}
\end{minipage}%
\caption{
\small
\textbf{LLMs can learn the invariant measure of chaotic systems despite unpredictable trajectories.}
(\textbf{a}) Chaotic systems like the Lorenz equations display sensitive dependence on initial conditions. Tiny perturbations to the initial conditions (orange) diverge exponentially, making long-term predictions impossible when simulating with limited computation and precision on a computer.
(\textbf{b}) $3000$ sampled points from the distribution modeled by the LLM (left) and from the invariant measure of the Lorenz system (right). Color denotes kernel density estimation of each density.
}
\label{fig:lorenz-combined}
\end{figure}

\textbf{Lorenz System and Chaotic Dynamics}
For the Lorenz system, a canonical example of a chaotic ODE, we can observe a different kind of emergence (Type-0 in \citet{carroll2024emergence}). There exists a canonical invariant measure in dynamical systems (under some regularity conditions) known as the SRB measure\citep{metzger2000sinai}. States evolved for a long time in the Lorenz system will converge this measure. As the Lorenz system is chaotic, tiny perturbations are exponentially amplified through time at a rate related to the largest Lyapunov exponent $\lambda_1 \approx 0.9$. There is a precise sense in which entropy is created in this system at a rate of $\lambda_1 \log_2(e)$ bits per second, formalized through Pesin's theorem \citep{pesin1977characteristic}, despite the fact that it is a purely deterministic process. Intuitively one can see this picture when simulating the system using fixed precision numbers, and seeing $\log_2(e)$ bits of that description replaced with unpredictable random content after every Lyapunov time $1/\lambda_1$. On the one hand randomness is produced, but it is not uniformly random. Rather, there is a stationary measure in the shape of a butterfly, and an observer who has lost track of all previous bits due to chaos can still learn the shape of the butterfly. Moreover, the shape of the stationary measure is not immediately obvious from the ODE, it is emergent and cannot easily be understood without intensive numerical simulation of the system (hence why most of chaos theory was developed after computers).

To demonstrate this interplay, we train a language model to predict the first $B=10$ bits of the future state $\Phi_t(X)$ from an initial state sampled uniformly from the box $X\sim U[-20,20]^3 + 20[0,0,1]$ quantized to $B$ bits, in comparison to directly modeling $\Phi_t(X)$. For both we set the time $t$ to be $30$ Lyapunov times into the future, $t=30/\lambda_1$. The resulting model has a nearly identical loss and estimated epiplexity in the two settings. Despite being unable to distinguish the initial conditions, the LLM learns the invariant (SRB) measure to a reasonable approximation as shown in \autoref{fig:lorenz}.
With very limited compute the stationary measure is not predictable apriori from the dynamics, but with more compute it is merely a consequence. The epiplexity of the attractor for limited compute may be larger than a description of the dynamics $\rS_T(\Phi_t(X)) > \rS_T(\Phi,t)$. 

\textbf{Chess: AlphaZero and Minimax}
A qualitatively different kind of example can be had by considering the models produced by AlphaZero \citep{silver2018general} and the theoretically optimal minimax solution for these two player zero sum perfect information games \citep{vonneumann1928theorie,shannon1950chess}. The minimax strategy can be implemented by a short program, and with sufficient compute (exponential in the size of the board \citep{fraenkel1981computing}) the optimal strategy can be found. On the other hand the CNN policy and value network produced by AlphaZero contain $10$s of millions of parameters. Given that the rules of chess can be encoded in just a few hundred bytes, and the algorithm used to train the model can be simply described and also implemented by a short program, one may wonder \emph{where does this information come from?} With the other examples of emergent phenomena in mind, we can make sense of this information being produced by the computational process of the AlphaZero system. In contrast, with unbounded compute, the best strategy contains little information.

To summarize, due to the existence of emergent phenomena, even systems that have simple generating processes or simple descriptions can lead to large amounts of structural information to be learned by computationally constrained observers.

\section{Induction is Not Specific to Autoregressive Factorization}\label{app:induction-not-specific}
One might get the impression that key constraint that leads to this induction phenomenon is the autoregressive factorization, as it is intuitive to see how such a model needs to perform induction in-context to achieve minimum loss. However, we argue this phenomenon takes place with other classes of generative models trained as long as they are trained with Maximum Likelihood Estimation (MLE) or its approximations. 

In MLE, a generative model allowing explicit likelihood evaluation is trained to maximize the likelihood of the data. Computing the likelihood can be significantly more computationally challenging than sampling from the distribution $P.$ This distinction is particularly clear in the examples we gave where the ground-truth $P$ is a mixture distribution represented by a latent variable model with the CA initial state or Markov chain transition matrix acting as the latent variable $Z$. Given access to $P_{X|Z}$ (equivalent to some easy to implement forward function $F$), sampling is easy as long as $P_Z$ is a simple, but computing $P_X(x)$ for some input $x$ requires evaluating an intractable integral $P_X(x) = \int P_{X|Z}(x|z) P_Z(z) \,dz$ due to the high-dimensionality of $Z.$ As such, a model given a limited compute-budget is forced to learn a cheaper but more sophisticated algorithm for computing $P_X(x),$ often involving approximating the inverse $P_{Z|X}$ either explicitly as done in expectation–maximization-type algorithms and Variational Autoencoders \citep{kingma2013auto}, or implicitly as we illustrated for the autoregressive transformer.

\section{Minimum Description Legnth}
\label{app:mdl}
Intuitively, $L(H)$ can be interpreted as the structural information, and $-\log P(x \mid H)$ can be understood as the remaining random information that cannot be predicted by the best model in $\mathcal H$. 
A main problem with the crude two-part code is that it does not prescribe how one should design the code for $H$ (i.e.,  a procedure for describing $H$ within $\mathcal{H}$).
The description of a particular $H$ can be short under one code but very large under another, which could require additional knowledge to resolve.
To circumvent this issue, one can use a more refined one-part code that describes the data with the entire model class $\mathcal{H}$ rather than any single model $H$. One of the most important one-part codes is the normalized maximum likelihood code.
\begin{definition}[Normalized maximum likelhood code~\citep{grunwald2007minimum}]
    The NML distribution $P^{\mathrm{NML}}_{\mathcal{H}}: \{0,1\}^{n\times d} \rightarrow [0, 1]$ of a probablistic model class $\mathcal{H}$ is:
    $$P^{\mathrm{NML}}_{\mathcal{H}}(x) = \frac{P(x \mid \widehat{H}(x))}{\sum_{y \in \{0,1\}^{n\times d}} P(y \mid \widehat{H}(y))},$$
    where $\widehat{H}(x) = \argmax_{H \in \mathcal{H}} P(x \mid H)$ is the maximum likelihood estimator for $x$.
\end{definition}

Crucially, notice that the NML code only depends on $\mathcal{H}$ rather than any particular $H \in \mathcal{H}$, so we do not have to design a particular code for $H$.
Unfortunately, the NML code requires integrating over the maximum likelihood estimator for all possible data, which is intractable for most practical models such as deep neural networks.
We can instead use a more tractable variant of one-part code based on sequential prediction called prequential coding.
\begin{definition}[Prequential code~\citep{grunwald2007minimum}]
    The prequential distribution $P^{\mathrm{PREQ}}_{\mathcal{H}}: \{0,1\}^{n\times d} \rightarrow [0, 1]$ of a probabilistic model class $\mathcal{H}$ is:
    $$P^{\mathrm{PREQ}}_{\mathcal{H}}(x) = \prod_{k=1}^n P(x_k \mid \widehat{H}(x_{1:k})),$$
    where $\widehat{H}(x_{1:k}) = \argmax_{H \in \mathcal{H}} P(x_{1:k} \mid H)$ is the MLE for the first $k$ elements of $x$. 
\end{definition}
This definition above uses the MLE for updating $\widehat{H}$ but there are in fact no constraints on how the update is performed.
We may use any update method of our choice to produce the next model in the sequence, so long as it only depends on the previous data.
This means that we can naturally adapt it for deep learning, where we use stochastic gradient descent to update the model sequentially.

A code cannot be optimal simultaneously for all possible data $x$ unless it has knowledge of the particular $x$. Therefore, it is useful to characterize how close a given code is to the optimal model, which can be formalized via the notion of \emph{regret}.
\begin{definition}[Regret~\citep{grunwald2007minimum}]
    The regret of a code $Q$ relative to $\mathcal{H}$ for $x$ is the additional number of bits needed to encode $x$ using $Q$ compared to the best model in hindsight,
    $$\mathsf{Reg}(Q, \mathcal{H}, x) = -\log Q(x) - \min_{H \in \mathcal{H}} \{-\log P(x \mid H)\}.$$
\end{definition}
Under this notion of penalty, the NML is optimal in the sense that it achieves the minimax regret.
The regret provides a way to compare different codes.
Consider the two-part regret of the crude two-part code $P^{\mathrm{2P}}(\cdot)$ with minimizer $H^\star$ and associated predictive distribution $P(\cdot \mid H^\star)$, $$\mathsf{Reg}(P^{\mathrm{2P}}, \mathcal{H}, x) = L(H^\star) + \log \frac{1}{P(x \mid H^\star)} - \log \frac{1}{P(x \mid \widehat{H})}.$$
This means that for a two-part code, the regret is an upper bound on the description length of the model.
For sufficiently large $n$, the last two terms become close to each other and $\mathsf{Reg}(P^{\mathrm{2P}}, \mathcal{H}, x) \approx L(H^\star)$.
In the case of NML, the regret is the minimax regret that $\mathsf{Reg}(P^{\mathrm{NML}}_{\mathcal{H}}, \mathcal{H}, x) = \log \sum_{y\in \{0,1\}^n} P(y \mid \widehat{H}(y))$. 
This quantity is independent of $x$, which is also called \emph{parametric complexity} of $\mathcal{H}$, because it measures how expressive the entire \emph{model class} is by counting the total amount of possible data sequences the model class can model well.

\end{document}